\documentclass{article}
\usepackage{amsmath}
\usepackage{times}
\usepackage{graphicx}
\usepackage{color}
\usepackage{multirow}
\usepackage{fullpage}

\usepackage{hyperref}
\usepackage{epstopdf}
\usepackage{amssymb}
\usepackage{amsthm}
\usepackage{algorithm}
\usepackage{algorithmic}
\usepackage{subfigure} 
\newtheorem{theorem}{Theorem}
\newtheorem{prop}{Proposition}
\newtheorem{lemma}{Lemma}

\renewcommand{\thefootnote}{\normalsize \arabic{footnote}} 	

\newcommand{\Tab}[1]{Table~\ref{#1}}

\newcommand{\Eqn}[1]{Eq.~(\ref{#1})}

\newcommand{\Lem}[1]{Lemma~\ref{#1}}
\newcommand{\Thm}[1]{Theorem~\ref{#1}}
\newcommand{\Prop}[1]{Proposition~\ref{#1}}

\renewcommand{\>}{{\,\rightarrow\,}}
\renewcommand{\hat}{\widehat}
\renewcommand{\tilde}{\widetilde}

\newcommand{\argmax}{\textup{\textrm{argmax}}}
\newcommand{\argmin}{\textup{\textrm{argmin}}}
\newcommand{\arginf}{\textup{\textrm{arginf}}}
\newcommand{\R}{{\mathbb R}}
\newcommand{\Z}{{\mathbb Z}}
\newcommand{\N}{{\mathbb N}}
\renewcommand{\P}{{\mathbf P}}
\newcommand{\E}{{\mathbf E}}

\newcommand{\X}{{X}}

\newcommand{\T}{{\mathcal T}}
\newcommand{\D}{{\mathcal D}}

\newcommand{\C}{{\mathcal C}}
\newcommand{\F}{{\mathcal F}}

\newcommand{\1}{{\mathbf 1}}

\renewcommand{\a}{{\mathbf a}}

\newcommand{\bT}{{\mathbb T}}

\newcommand{\TPR}{\textup{\textrm{TPR}}}
\newcommand{\FPR}{\textup{\textrm{FPR}}}

\newcommand{\AUC}{\textup{\textrm{AUC}}}
\newcommand{\pAUC}{\textup{\textrm{pAUC}}}
\newcommand{\SVM}{\textup{\textrm{SVM}}}
\newcommand{\tight}{\textup{\textrm{tight}}}
\newcommand{\struct}{\textup{\textrm{struct}}}

\newcommand{\sign}{\textup{\textrm{sign}}}

\newcommand{\hinge}{\textup{\textrm{hinge}}}

\newcommand{\newtilde}{\raise.17ex\hbox{$\scriptstyle\mathtt{\sim}$}}
\newcommand{\dc}{\textup{\textrm{dc}}}

\newcommand{\term}{\textrm{term}}

\newcommand{\wt}{\widehat{t}}

\newcommand{\werr}{{\widehat{R}}}
\newcommand{\terr}{{\widetilde{R}}}
\newcommand{\berr}{{\bar{R}}}

\allowdisplaybreaks

\begin{document}

\title{Support Vector Algorithms for Optimizing the\\ Partial Area Under the ROC Curve}
\author{
Harikrishna Narasimhan\footnote{This work was done when the author was a PhD student at the Indian Institute of Science, Bangalore.}\\
{John A. Paulson  School of Engineering and Applied Sciences}, Harvard University\\
{Cambridge, MA  02138, USA}\\
{Email: \href{mailto:hnarasimhan@seas.harvard.edu}{hnarasimhan@seas.harvard.edu}}\\
\and
Shivani Agarwal\\
Radcliffe Institute for Advanced Study, Harvard University\\
Cambridge, MA 02138, USA\\
{Department of Computer Science and Automation, Indian Institute of Science}\\
{Bangalore  560012, INDIA}\\
{Email: \href{mailto:shivani@csa.iisc.ernet.in}{shivani@csa.iisc.ernet.in}}
}
\date{}
\maketitle


\thispagestyle{empty}

{\let\thefootnote\relax\footnotetext{~\\[-10pt]\indent Parts of this paper were presented at the 30th International Conference on Machine Learning, 2013 \cite{NarasimhanAg13a} and at the 19th ACM SIGKDD Conference on Knowledge Discovery and Data Mining, 2013 \cite{NarasimhanAg13b}.}}
\begin{abstract}
The area under the ROC curve (AUC) is a widely used performance measure in machine learning.  Increasingly, however, in several applications, ranging from ranking to biometric screening to medicine, performance is measured not in terms of the full area under the ROC curve, but in terms of the \emph{partial} area under the ROC curve between two false positive rates. In this paper, we develop support vector algorithms for directly optimizing the partial AUC between any two false positive rates. Our methods are based on minimizing a suitable proxy or surrogate objective for the partial AUC error. In the case of the full AUC, one can readily construct and optimize convex surrogates by expressing the performance measure as a summation of  pairwise terms. The partial AUC, on the other hand, does not admit such a simple decomposable structure, making it more challenging to design and optimize (tight) convex surrogates for this measure.

Our approach builds on the structural SVM framework of Joachims (2005) to design convex surrogates for partial AUC, and solves the resulting optimization problem using a cutting plane solver. Unlike the full AUC, where the combinatorial optimization needed in each iteration of the cutting plane solver can be decomposed and solved efficiently, the corresponding problem for the partial AUC is harder to decompose. One of our main contributions is a polynomial time algorithm for solving the combinatorial optimization problem associated with partial AUC. We also develop an approach for optimizing a tighter non-convex hinge loss based surrogate for the partial AUC using 
difference-of-convex programming. Our experiments on a variety of real-world and benchmark tasks confirm the efficacy of the proposed methods. 
\end{abstract}

\section{Introduction}
\label{sec:intro}
The receiver operating characteristic (ROC) curve plays an important role as an evaluation tool in machine learning and data science. In particular, the area under the ROC curve (AUC) is widely used to summarize the performance of a scoring function in binary classification problems, and is often a performance measure of interest in bipartite ranking \cite{CortesMo04, Agarwal+05}. Increasingly, however, in several applications, the performance measure of interest is not the full area under the ROC curve, but instead, the \emph{partial} area under the ROC curve between two specified false positive rates (FPRs) (see Figure~1). For example, in ranking applications where accuracy at the top is critical, one is often interested in the left-most part of the ROC curve \cite{pnorm,infpush,sparseinfpush}; this corresponds to maximizing partial AUC in a false positive range of the form $[0,\beta]$. In biometric screening, where false positives are intolerable, one is again interested in
maximizing the partial AUC in a false positive range $[0,\beta]$ for some suitably small $\beta$. 
In the KDD Cup 2008 challenge on breast cancer detection, performance was measured in terms of the partial AUC in a specific false positive range $[\alpha,\beta]$ deemed clinically relevant  \cite{kddcup}.\footnote{More specifically, the KDD Cup 2008 challenge used the partial area under the \emph{free-response} operating characteristic curve, where a scaled version of usual FPR is used.} 


In this paper, we develop support vector machine (SVM) based algorithms for directly optimizing the partial AUC between any two false positive rates $\alpha$ and $\beta$. Our methods are based on minimizing a suitable proxy or surrogate objective for the partial AUC error. In the case of the full AUC, where the evaluation measure can be expressed as a summation of pairwise indicator terms, one can readily construct and optimize surrogates by exploiting this structure. The partial AUC, on the other hand, does not admit such a decomposable structure, as the set of negative instances associated with the specified false positive range can be different for different scoring models; as a result, it becomes more challenging to design and optimize convex surrogates for this measure.

\begin{figure}[t]
\centering
\vspace{-10pt}
\includegraphics[scale=0.3]{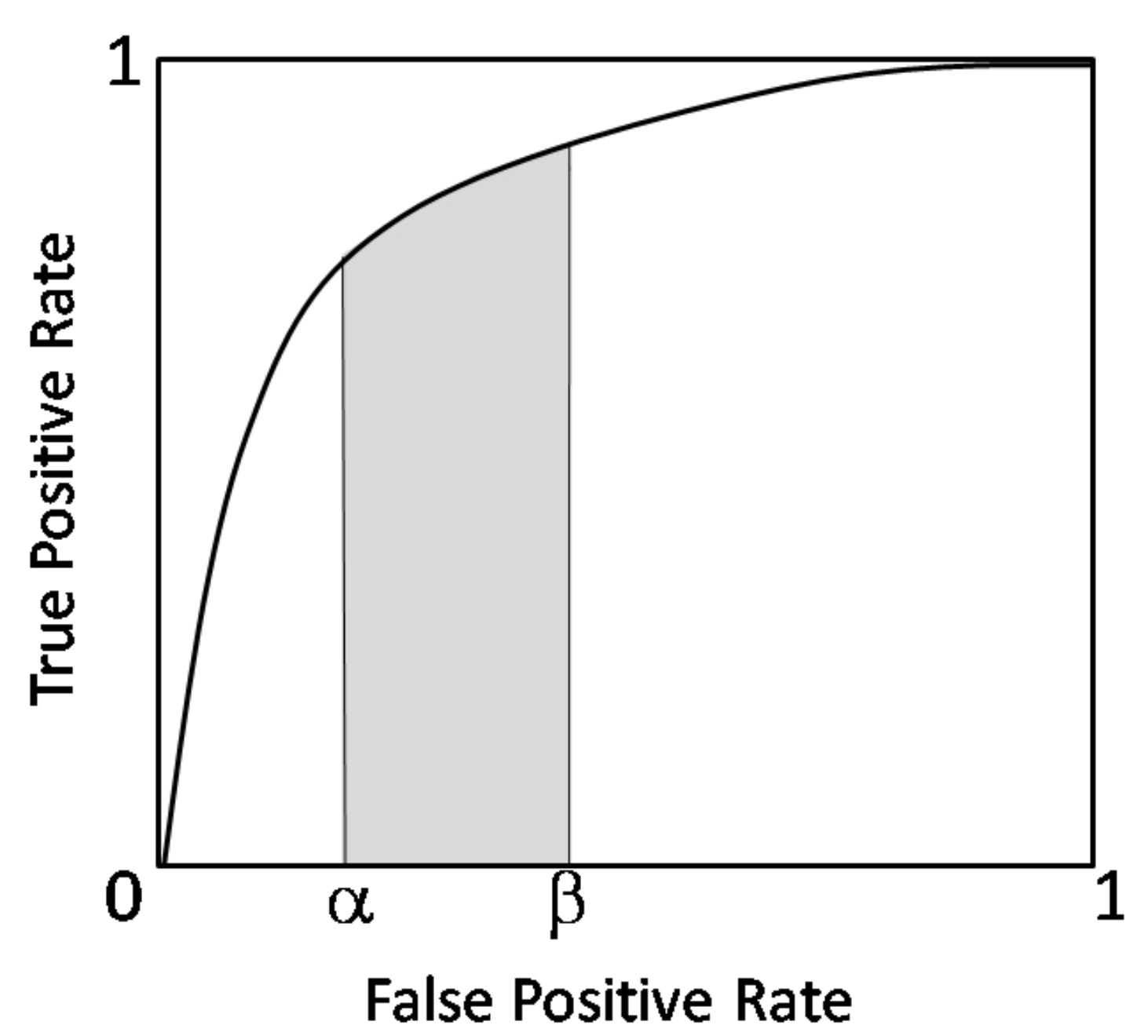}
\caption{Partial AUC in false positive range $[\alpha, \beta]$.}
\label{fig:pauc}
\end{figure}

For instance, a popular approach for constructing convex surrogates for the full AUC is to replace the indicator terms in its definition with a suitable pairwise convex loss such as the pairwise hinge loss; in fact, there are several efficient methods available to solve the resulting optimization problem \cite{ranksvm1,ranksvm2,svmperf}. This is not the case with the more complex partial AUC measure; here, a surrogate constructed by replacing the indicators with the pairwise hinge loss is non-convex in general. Even in the special case of FPR intervals of the form $[0, \beta]$, where the hinge loss based surrogate turns out to be convex, solving the resulting optimization problem is not straightforward.

In our approach, we construct and optimize convex surrogates on the partial AUC by building on the structural SVM formulation of Joachims (2005) developed for general complex performance measures  \cite{svmperf}. It is known that for the full AUC, this formulation recovers the corresponding hinge surrogate \cite{lintime}. On the other hand, a direct application of this framework to the partial AUC results in a loose approximation to the performance measure (in a sense that we will elaborate in later sections). Instead, we first rewrite the evaluation measure as a maximum of a certain term over subsets of negative instances, and leverage the structural SVM setup to construct a convex approximation to the inner term. This yields a tighter surrogate, which for the special case of partial AUC in the $[0,\beta]$ range, is equivalent to the hinge surrogate obtained by replacing the indicators with the pairwise hinge loss; for general FPR intervals $[\alpha, \beta]$, the surrogate obtained can be seen as a convex relaxation to the (non-convex) hinge surrogate.
%
%
%

We make use of the cutting plane method to optimize the proposed structural SVM surrogates. Each iteration of this solver requires a combinatorial search over subsets of instances and over binary matrices (representing relative orderings of positive and negative training instances) to find the currently most violated constraint. In the case of the full AUC (where the optimization is only over binary matrices), this problem decomposes neatly into one where each matrix entry can be chosen independently \cite{svmperf}.  Unfortunately, for the partial AUC, a straightforward decomposition is not possible, again because the negative instances involved in the relevant false positive range can be different for different orderings of instances. 
One of our main contributions in this paper is a polynomial time algorithm for solving the corresponding combinatorial optimization within the cutting plane method for the partial AUC, by breaking down the problem into smaller tractable ones. When the specified false positive range is of the form $[0,\beta]$, we show that after fixing the optimal subset of negatives to the top ranked negatives, one can still optimize the individual entries of the ordering matrix separately. For the general case $[\alpha,\beta]$, we require to further formulate an equivalent optimization problem over a restricted search space, where each \emph{row} of the matrix can be optimized separately -- and efficiently. 


While the use of convex surrogates in the above approach allows for efficient optimization and guarantees convergence to the global surrogate optimum, it turns out that for the partial AUC in a general FPR interval $[\alpha,\beta]$, the previous non-convex hinge surrogate (obtained by replacing the indicators with the pairwise hinge loss) is a tighter approximation to the original evaluation measure. Hence, as a next step, we also develop a method for  directly optimizing this non-convex surrogate using a popular non-convex optimization technique based on difference-of-convex (DC) programming; here we exploit the fact that the partial AUC in  $[\alpha,\beta]$ can be written as a difference of (scaled) partial AUC values in $[0,\beta]$ and $[0,\alpha]$. 

We evaluate the proposed methods on a variety of real-world applications where partial AUC is an evaluation measure of interest and on benchmark data sets. We find that in most cases, the proposed methods yield better partial AUC performance compared to an approach for optimizing the full AUC, thus confirming the focus of our methods on a select false positive range of the ROC curve. Our methods are also competitive with existing algorithms for optimizing partial AUC. For partial AUC in $[\alpha, \beta]$, we find that in some settings, the proposed DC programming method for optimizing the non-convex hinge surrogate (despite having the risk of getting stuck at a locally optimal solution) performs better than the structural SVM method, though overall there is no clear winner. 


\subsection{Related Work}
There has been much work on developing algorithms to optimize the full AUC, mostly in the context of ranking \cite{ranksvm1, ranksvm2, rankboost, ranknet, svmperf}. There has also been interest in the ranking literature in optimizing measures focusing on the left end of the ROC curve, corresponding to maximizing accuracy at the top of the list \cite{pnorm}; in particular, the Infinite Push ranking algorithm \cite{infpush, sparseinfpush} can be viewed as maximizing the partial AUC in the range $\big[0,\frac{1}{n}\big]$, where $n$ is the number of negative training examples.  

While the AUC is widely used in practice, increasingly, the partial AUC is being preferred as an evaluation measure in several applications in bioinformatics and medical diagnosis \cite{pauc_two,ppiQi06,kddcup,Hsu+14}, and more recently even in domains like computer vision \cite{Paisitkriangkrai+13,Paisitkriangkrai+14}, personalized medicine \cite{Majumder+15}, and demand forecasting \cite{SchneiderGorr15}. The problem of optimizing the partial AUC in false positive ranges of the form $[0,\beta]$ has received some attention primarily in the bioinformatics and biometrics literature \cite{pauc_two, paucreg, pauc_marker_selection, paucmax, pauc_lin_max}; however in most cases, the algorithms developed are heuristic in nature. The asymmetric SVM algorithm of \cite{asvm} also aims to maximize the partial AUC in a range $[0,\beta]$ by using a variant of one-class SVM; but the optimization objective used does not directly approximate the partial AUC in the specified range, but instead seeks to indirectly promote good partial AUC performance through a fine-grained parameter tuning procedure. There has also been some work on optimizing the partial AUC in general false positive ranges of the form $[\alpha,\beta]$ 
including the boosting-based algorithms pAUCBoost \cite{paucboost} and p\textit{U}-AUCBoost \cite{puaucboost}.

Support vector algorithms have been extensively used in practice for various supervised learning tasks, with both standard and complex performance measures \cite{CortesVapnik95,CrammerSinger02,ChuKeerthi04,ranksvm2,structsvm}. The proposed methods are most closely related to the structural SVM framework of Joachims for optimizing the full AUC \cite{svmperf}. To our knowledge, ours is the first work to develop principled support vector methods that can directly optimize the partial AUC in an arbitrary false positive range $[\alpha, \beta]$. 

\subsection{Paper organization.} 
We begin with the problem setting in Section \ref{sec:prelims}, along with background material on the previous structural SVM framework for full AUC maximization. 
In Section \ref{sec:surrogates}, we consider two initial surrogates for the partial AUC, one based on the pairwise hinge loss and the other based on a na\"{i}ve application of the structural SVM formulation, and point out drawbacks in each case. 
We then present a tight convex surrogate for the special case of FPR range $[0, \beta]$ in Section \ref{sec:svm-beta} and for the general case of $[\alpha, \beta]$ intervals in Section \ref{sec:svm-general}, along with cutting plane solvers for solving the resulting optimization problem. 
Subsequently in Section \ref{sec:svm-dc}, we also describe a DC programming approach for directly optimizing the non-convex hinge surrogate for partial AUC in $[\alpha, \beta]$. 
We provide a generalization bound for the partial AUC in Section \ref{sec:genbound}, and present our experimental results on real-world and benchmark tasks in Section \ref{sec:expts}.
~\\[-0.5cm]

\section{Preliminaries and Background}
\label{sec:prelims}
\subsection{Problem Setting}
Let $X$ be an instance space and $\D_+$ and $\D_-$ be probability distributions over positive and negative instances in $X$. We are given a training sample $S=(S_+,S_-)$ containing $m$ positive instances $S_+ = (x^+_1,\ldots,x^+_m) \in X^m$ drawn iid according to $\D_+$ and $n$ negative instances $S_- = (x^-_1,\ldots,x^-_n) \in X^n$ drawn iid according to  $\D_-$. Our goal is to learn from $S$ a scoring function $f:X\>\R$ that assigns higher scores to positive instances compared to negative instances, and in particular yields good performance in terms of the partial AUC between some specified false positive rates $\alpha$ and $\beta$, where $0 \leq \alpha < \beta \leq 1$. 
 In a ranking application, this scoring function can then be deployed to rank new instances accurately, while in a classification setting, the scoring function along with a suitable threshold serves as a binary classifier. 

\textbf{Partial AUC}. 
Define for a scoring function $f:X\>\R$ and threshold $t\in\R$, the true positive rate (TPR) of the binary classifier $\sign(f(x)-t)$ as the probability that it correctly classifies a random positive instance from $\D_+$ as positive:\footnote{Here $\sign(z) = 1$ if $z > 0$ and $-1$ otherwise.}
\[
\TPR_f(t) = 
	\textbf{P}_{x^+ \sim \D_+}[f(x^+) > t]
\]
and the false positive rate (FPR) of the classifier as the probability that it misclassifies a random negative instance from $\D_-$ as positive:
\[
\FPR_f(t) = \P_{x^- \sim \D_-}[f(x^-) > t] 
	.
\]
The ROC curve for the scoring function $f$ is then defined as the plot of $\TPR_f(t)$ against $\FPR_f(t)$ for different values of $t$. The area under this curve can be computed as
\begin{eqnarray*}
\AUC_f
	& = & 
	\int_{0}^{1} \TPR_f\big(\FPR_f^{-1}(u)\big)\,du 
	\,,
\end{eqnarray*}
where 
\(
\FPR_f^{-1}(u) = 
	\inf \big\{ t \in \R ~|~ \FPR_f(t) \leq u \big\}
	\,.
\)
Assuming there are no ties, it can be shown \cite{CortesMo04} that the AUC can be written as 
\[
\AUC_f  =  
	\P_{(x^+,x^-)\sim \D_+\times\D_-} [ f(x^+) > f(x^-) ]
	\,.
\]
Our interest here is in the area under the curve between FPRs $\alpha$ and $\beta$. The (normalized) partial AUC of $f$ in the range $[\alpha,\beta]$ is defined as
\begin{eqnarray*}
\pAUC_f(\alpha, \beta) 
	& = & 
	\frac{1}{\beta - \alpha} \int_{\alpha}^{\beta} \TPR_f\big(\FPR_f^{-1}(u)\big)\,du 
	\,.
\end{eqnarray*}
%


\textbf{Empirical Partial AUC.} 
Given a sample $S=(S_+,S_-)$ as above, one can plot an empirical ROC curve corresponding to a scoring function $f:X\>\R$; assuming there are no ties, this is obtained by using 
\[
\widehat{\TPR}_f(t) = 
	\frac{1}{m} \sum_{i=1}^m 
		\1\big( f(x^+_i) > t \big) 
~~\text{ and }~~
\widehat{\FPR}_f(t) = 
	\frac{1}{n} \sum_{j=1}^n 
		\1\big( f(x^-_j) > t \big) 
\]
instead of $\TPR_f(t)$ and $\FPR_f(t)$ respectively.
The area under this empirical curve is given by
\begin{equation}
\widehat{\AUC}_f
	=
	\frac{1}{mn} \sum_{i = 1}^m \sum_{j = 1}^n 
		\1 \big( f(x_i^+) > f(x_j^-) \big)
	\,.
	\label{eqn:auc}
\end{equation}

\noindent Denoting $j_\alpha= {\lfloor n\alpha \rfloor}$ and $j_\beta = {\lceil n\beta \rceil}$,  
the (normalized) empirical partial AUC of $f$ in the FPR range $[\alpha,\beta]$ can then be written as \cite{paucreg}:
\begin{eqnarray}
\widehat{\pAUC}_f(\alpha, \beta) \,=\,
	\frac{1}{m(j_\beta - j_\alpha)} \sum_{i = 1}^m  
		\sum_{j = j_\alpha + 1}^{j_\beta} 
		\1\big(  f(x_i^+) > f(x^-_{(j)}) \big) 
		.
\label{eqn:pauc}
\end{eqnarray}
where $x^-_{(j)}$ denotes the negative instance in $S_-$ ranked in $j$-th position (among negatives, in descending order of scores) by $f$.\footnote{The empirical partial AUC in \cite{NarasimhanAg13a} includes two additional terms to avoid the use of ceil and floor approximations in computing $j_\alpha$ and $j_\beta$. For ease of exposition, we work with a simpler definition here.} 

\begin{figure}[t]
\centering
\small
\begin{center}
\begin{tabular}{|l|cccccc|cccccc|ccc}
\hline
\hline
& & $x_1^+$ & $x_2^+$ & $x_3^+$ & $x_4^+$ &  & & $x_1^-$ & $x_2^-$ & $x_3^-$ & $x_4^-$ & $x_5^-$
\\
\hline
$f_1$  & & 9.1 & 6.8 & 6.1 & 5.7  & & & 8.5 & 8.1 & 4.2 & 3.6 & 2.3\\
$f_2$  & & 9.9 & 8.7 & 3.3 & 2.1  & & & 7.6 & 5.3 & 4.9 & 4.4 & 0.8\\
\hline
\hline
\end{tabular}
\end{center}
\vspace{-6pt}
%
\includegraphics[scale=0.45]{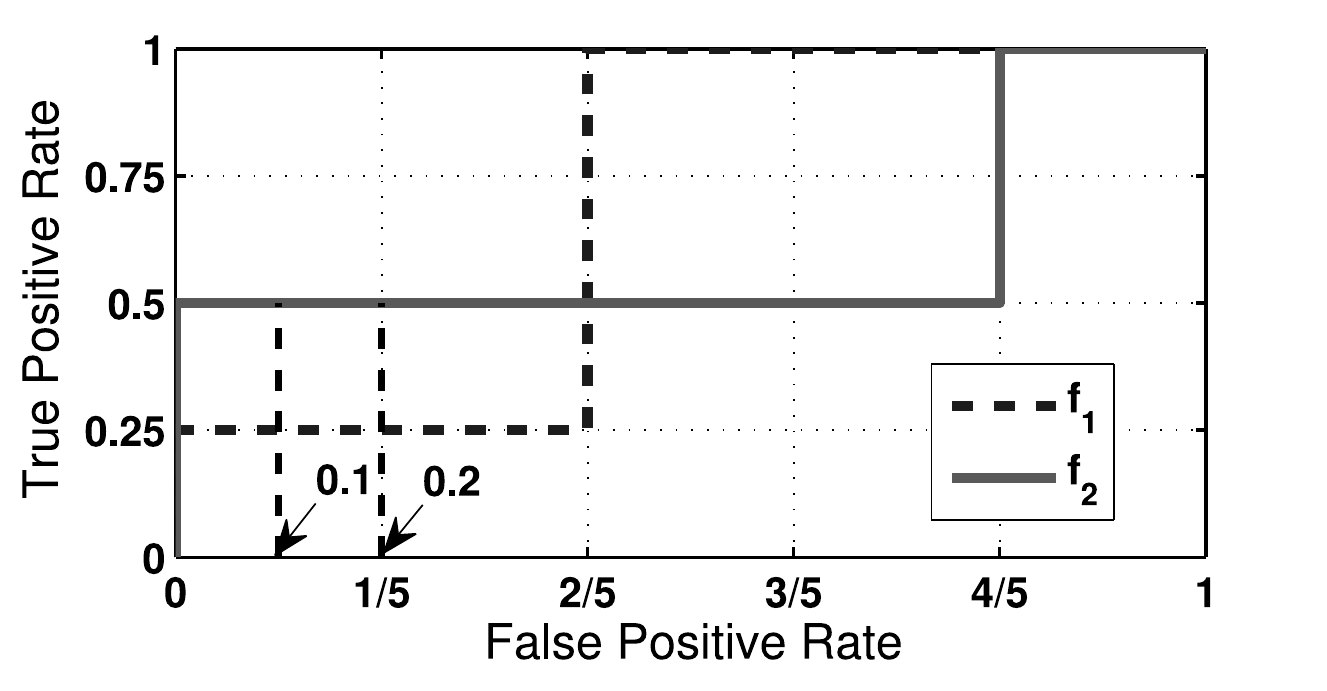}
\caption{ROC curves for scoring functions $f_1$ and $f_2$ described in the above table on a sample containing 4 positive instances and 5 negative instances.}
\label{fig:roc-f1-f2}
\end{figure}

\textbf{Partial AUC vs.\ Full AUC.} It is important to note that for the AUC to take its maximum value of 1, a scoring function needs to rank the positive instances above \textit{all} the negative instances. On the other hand, for the partial AUC in a specified interval $[\alpha, \beta]$ to take a value of 1, it is sufficient that a scoring function ranks the positives above a subset of the negative instances (specifically, above those in positions $j_\alpha+1$ to $j_{\beta}$ in the ranking of negatives). Another key difference between the two evaluation measures is that the full AUC can be expressed as an expectation or sum of indicator terms over pairs of positive-negative instances (see \Eqn{eqn:auc}), whereas the partial AUC does not have such a simple additive structure. This is clearly evident in the definition in \Eqn{eqn:pauc}, where the set of negatives corresponding to FPR range $[\alpha, \beta]$ that appear in the inner summation is not fixed and can be different for different scoring functions $f$. 

We would also like to stress that a scoring function with a high AUC value need not be optimal in terms of partial AUC in a particular FPR range. This is illustrated in Figure \ref{fig:roc-f1-f2}, which shows scores assigned by two scoring functions $f_1$ and $f_2$ on a hypothetical sample of 4 positive and 6 negative instances, and the corresponding ROC curves. As can be seen, while $f_1$ gives a higher AUC value, $f_2$ has higher partial AUC in the FPR range $[0.1, 0.2]$. This motivates the need to design algorithms that are tailored to directly optimize partial AUC. 

\subsection{Background on Structural SVM Framework for Full AUC}
\label{subsec:struct-auc}
As a first step towards developing a method for optimizing the partial AUC, we shall provide some background on the popular structural SVM framework for maximizing the full AUC \cite{svmperf}. Unless otherwise specified, we shall assume that $X\subseteq \R^d$ for some $d\in\Z_+$ and shall consider linear scoring functions of the form $f(x) = w^\top x$ for some $w \in \R^d$; the methods described will easily extend to non-linear functions~/~non-Euclidean instance spaces using kernels \cite{structsvm_kernel}.

\textbf{Hinge loss based surrogate.} Given a training sample $S$, our goal here is to find a scoring function that yields maximum empirical AUC on  $S$, or equivalently minimizes one minus the empirical AUC, given as follows:
\begin{equation}
\widehat{R}_{\AUC}(w;\, S)
	\,=\,
		\frac{1}{mn} \sum_{i = 1}^m \sum_{j = 1}^n 
		\1 \big( w^\top x_i^+ \leq w^\top x_j^- \big)
	.
\label{eqn:emp-auc-risk}
\end{equation}
Owing to its discrete nature, minimizing this objective is a hard problem in general. One instead works with a convex proxy or surrogate objective for the above risk that is easier to optimize. A common approach is to replace the above indicator term with a pair-wise loss such as  the pair-wise hinge loss, which for any scoring function $f$ and instances $x^+$ and $x^-$ is defined as $(1 - (w^\top x^+ - w^\top x^-))_+$, where $(z)_+ = \max\{0, z\}$; this is clearly convex in $w$ and an upper bound on $\1 \big( w^\top x^+ \leq w^\top x^-  \big)$. 
The following is then the well-known pairwise hinge surrogate for the AUC risk:
\begin{equation}
\widehat{R}^\hinge_{\AUC}(w; \,S)
	\,=\,
		\frac{1}{mn} \sum_{i = 1}^m \sum_{j = 1}^n 
		(1 - (w^\top x_i^+ - w^\top x_j^-))_+.
\label{eqn:emp-auc-risk-hinge}
\end{equation}
This surrogate is convex in $w$, upper bounds the AUC risk, and is minimized by a scoring function that ranks positives instances above the negative instances with a sufficient margin of separation. 
It is also evident that one can minimize this objective over all model vectors $w$ using a standard convex optimization solver. In fact, there  are several specialized methods available to solve (a regularized form of) this optimization problem \cite{ranksvm1,ranksvm2,svmperf}. For example, one popular approach is to solve the corresponding dual optimization problem using a coordinate descent type method \cite{ranksvm2}. 

On the other hand, the partial AUC has a more complex structure as the subset of negatives relevant to the given FPR range can be different for different scoring models; as a result, a surrogate obtained by replacing the indicators with the pairwise hinge loss turns out to be non-convex in general. The approach that we take for  the partial AUC will instead make use of the structural SVM framework developed by Joachims (2005) for designing surrogate minimizing methods for general complex performance measures \cite{svmperf}. For the full AUC, it has been shown that this formulation recovers the corresponding hinge surrogate in \Eqn{eqn:emp-auc-risk-hinge} \cite{lintime}. We give the details for AUC below, and in subsequent sections build on this formulation to construct and optimize convex surrogates for the partial AUC.  

\textbf{Structural SVM formulation.} 
For any ordering of the training instances,
we shall represent (errors in) the relative ordering of the $m$ positive instances in $S_+$ and $n$ negative instances in $S_-$ via a matrix $\pi = [\pi_{ij}] \in\{0, 1\}^{m \times n}$ as follows:
\[
\pi_{ij} = \left\{ \begin{array}{ll}
	1 & ~~\mbox{if $x^+_i$ is ranked \emph{below} $x^ -_j$} \\[2pt]
	0 & ~~\mbox{if $x^+_i$ is ranked above $x^ -_j$.} \\
\end{array} \right.
\]
Not all $2^{mn}$ matrices in $\{0,1\}^{m\times n}$ represent a valid relative ordering (due to transitivity requirements). We let $\Pi_{m,n}$ denote the set of all matrices in $\{0,1\}^{m\times n} $ that do correspond to valid orderings. Clearly, the correct relative ordering $\pi^*$ has $\pi^*_{ij} = 0 ~\forall i,j$. For any $\pi\in \Pi_{m,n}$, we can define the AUC loss of $\pi$ with respect to $\pi^*$ as %
\begin{eqnarray}
\Delta_{\AUC}(\pi^*,\pi) =  \frac{1}{mn}\sum_{i = 1}^m \sum_{j=1}^n \pi_{i,j}.
\label{eqn:auc-loss}
\end{eqnarray}
It can be verified that for any $\pi$ that is consistent with scoring function $w^\top x$, $\Delta_{\AUC}(\pi^*,\pi)$ evaluates to the AUC risk $\widehat{R}_\AUC(w; \, S)$ in \Eqn{eqn:emp-auc-risk}.

Further, we shall also define a joint feature map between the input training sample and an output ordering matrix $\phi:(X^m\times X^n)\times\Pi_{m,n} \> \R^d$ as 
%
\begin{equation}
\phi(S,\pi) = \frac{1}{mn} \sum_{i=1}^m \sum_{j=1}^n (1-\pi_{ij}) (x^+_i - x^-_j)
	\,.
\vspace{-4pt}
\label{eqn:psi}
\end{equation}
The above expression evaluates to a (normalized) sum of feature vector differences over all pairs of positive-negative instances in $S$ in which the positive instance is ordered by $\pi$ above the negative instance.
This choice of $\phi(S,\pi)$ ensures that for any fixed $w\in\R^d$, maximizing $w^\top \phi(S,\pi)$ over $\pi\in \Pi_{m, n}$ yields an ordering matrix consistent with the scoring function $w^\top x$, and thus for which the loss term evaluates to $\widehat{R}_\AUC(w; \, S)$. The problem of optimizing the AUC now reduces to finding a $w\in\R^d$ for which the maximizer over $\pi\in \Pi_{m, n}$ of $w^\top \phi(S,\pi)$ has minimum AUC loss. This is approximated by the following structural SVM based relaxation of the AUC loss: 
\begin{equation}
\widehat{R}^\struct_{\AUC}(w; \, S) \,=\, 
\underset{\pi \,\in\, \Pi_{m,n}}{\max}\big\{\Delta_{\AUC}(\pi^*,\pi) \,-\, (w^\top\phi(S,\pi^*) - w^\top \phi(S,\pi))\big\}.
\label{eqn:auc-struct}
\end{equation}
Clearly, this surrogate is convex in $w$ as it is a maximum of linear functions in $w$. Moreover, this is also an upper bound on the empirical AUC risk $\widehat{R}_{\AUC}(w; \, S)$: let $\bar{\pi}$ be the maximizer of $w^\top \phi(S,\pi)$ over $\Pi_{m, n}$; then from the above definition $\widehat{R}^\struct_{\AUC}(w; \, S) \,\geq\, {\Delta_{\AUC}(\pi^*,\bar{\pi}) - (w^\top\phi(S,\pi^*) - w^\top \phi(S,\bar{\pi}))} \,\geq\,\Delta_{\AUC}(\pi^*,\bar{\pi}) = \widehat{R}_{\AUC}(w; \, S)$. 

Interestingly, this surrogate can be shown to be equivalent to the hinge-loss based surrogate in \Eqn{eqn:emp-auc-risk-hinge}.
\begin{theorem}[Joachims, 2006 \cite{lintime}] 
For any $w \in \R^d$ and training sample $S \in X^m \times X^n$, $\widehat{R}^\struct_{\AUC}(w; \, S) 
\,=\, \widehat{R}^\hinge_{\AUC}(w; \, S).
$
\label{thm:auc-struct}
\end{theorem}
\noindent For completeness, we provide a proof for the theorem in Appendix \ref{app:auc-struct-proof} (see Supplementary Material).

Thus the problem of minimizing a pairwise hinge surrogate for the AUC can be cast as one of optimizing the above structural SVM surrogate (along with a suitable regularizer), which results in the following SVM style convex (quadratic) program:
\begin{equation}
	\min_{w,\, \xi \geq 0} \, \frac{1}{2}||w||^2 + C\xi
\nonumber
\end{equation}
\vspace{-0.75cm}
\[
	\text{s.t.}~~\forall \pi\in\Pi_{m,n}:~~
	w^\top \big( \phi(S,\pi^*) - \phi(S,\pi) \big) ~ \geq ~
	\Delta_{\AUC}(\pi^*,\pi) - \xi 
	\,, 
\]
%
%
where $C>0$ is a regularization parameter. 

\textbf{Cutting plane method.} While the above optimization problem contains an exponential number of constraints (one for each $\pi \in \Pi_{mn}$), it can be solved efficiently using the cutting plane method \cite{structsvm}. Each iteration of this solver requires a combinatorial optimization over matrices in $\Pi_{mn}$. By exploiting the simple structure of the AUC loss, this combinatorial problem can be decomposed into simpler ones, where each entry of the matrix can be optimized independently \cite{svmperf}. The cutting plane method is guaranteed to yield an $\epsilon$-accurate solution in $O(1/\epsilon)$ iterations \cite{lintime}; in the case of the AUC, each iteration requires $O((m+n)\log(m+n))$ computational time. We elaborate on this solver in Section \ref{sec:svm-beta} when we develop a structural SVM approach for the partial AUC.

\section{Candidate Surrogates for Partial AUC}
\label{sec:surrogates}
As noted earlier, our goal in this paper is to design efficient methods for optimizing the partial AUC in a specified false positive range. In particular, given a training sample $S=(S_+,S_-)$, we wish to find a scoring function $f(x) = w^\top x$ that maximizes partial AUC in $[\alpha,\beta]$, or equivalently minimizes the following risk:
\begin{eqnarray}
\widehat{R}_{\pAUC(\alpha, \beta)}(w; \, S) \,=\,
		\frac{1}{m j_\beta} \sum_{i = 1}^m\sum_{j = j_\alpha+1}^{j_\beta} 
		\1 \big( w^\top x_i^+ \leq w^\top x_{(j)_w}^- \big)
		.
\label{eqn:emp-pauc-risk-prev}
\end{eqnarray}
As before, optimizing this quantity directly is computationally hard in general. Hence, we work with a continuous surrogate objective that acts as a proxy for the above risk. As first-cut attempts at devising surrogates for the partial AUC, we replicate the two approaches used above for constructing surrogates for the full AUC, namely those based on the hinge loss and  the structural SVM framework respectively. As we shall see, the surrogates obtained in both cases have certain drawbacks, requiring us to use a somewhat different approach.


\subsection{Hinge Loss Based Surrogate}
\label{sec:surrogate-hinge}
We begin by considering a hinge style surrogate for the partial AUC obtained by replacing the indicator functions in the partial AUC risk with the pairwise hinge loss. 
\begin{equation}
\widehat{R}^\hinge_{\pAUC(\alpha,\beta)}(w;\, S)
	\,=\,
		\frac{1}{m(j_\beta-j_\alpha)} \sum_{i = 1}^m \sum_{j = j_\alpha+1}^{j_\beta} 
		\big( 1 \,-\, (w^\top x_i^+ \,-\, w^\top x_{(j)_w}^-)\big)_+
	.
\label{eqn:emp-pauc-hinge-general}
\end{equation}
In the case of the full AUC (i.e. when the FPR interval is $[0,1]$), the hinge surrogate is convex in $w$ and can hence be optimized efficiently. However, the corresponding surrogate given above for the partial AUC turns out to be non-convex in general. This is because the surrogate is defined on only a  subset of negative instances relevant to the given FPR range, and this subset can be different for different scoring functions. 
\begin{theorem}
\label{thm:hinge-nonconvex}
Let $0 < \alpha < \beta$ with $j_\alpha > 0$. Then there exists a training sample $S \in X^m\times X^n$ for which the surrogate $\widehat{R}^\hinge_{\pAUC(\alpha,\beta)}(w;\, S)$ is non-convex in $w$.
\end{theorem}
\begin{proof}[Proof (Sketch)] Consider the FPR range $\displaystyle \bigg[\frac{n-1}{n}, 1\bigg]$ where $j_\alpha = n-1 > 0$ and $j_\beta = n$. Here the hinge surrogate reduces to a `min' of convex functions in $w$ and it is easy to see that there are samples $S$ where the surrogate is non-convex in $w$. Note the same also holds for a general $k^{\text{th}}$ order statistic of a set of convex functions in $w$ when $k > 1$. In similar lines, one can construct a sample $S$ where the hinge surrogate is non-convex for more general FPR intervals. The details are provided in  Appendix \ref{app:hinge-non-convex} (see Supplementary Material).
\end{proof}
Fortunately, for the case where $\alpha =0$ and $j_\alpha = \lfloor n\alpha \rfloor = 0$, i.e. for FPR intervals of the form $[0,\beta]$, the hinge loss based surrogate turns out to be convex. Here the surrogate is given by
\begin{equation}
\widehat{R}^\hinge_{\pAUC(0,\beta)}(w;\, S)
	\,=\,
		\frac{1}{mj_\beta} \sum_{i = 1}^m \sum_{j = 1}^{j_\beta} 
		\big( 1 \,-\, (w^\top x_i^+ \,-\, w^\top x_{(j)_w}^-)\big)_+
	,
\label{eqn:emp-pauc-hinge-beta}
\end{equation}
and we have:
\begin{theorem}
\label{thm:hinge-convex}
Let $\beta > 0$. For any training sample $S \in X^m\times X^n$, the surrogate $\widehat{R}^\hinge_{\pAUC(0,\beta)}(w;\, S)$ is convex in $w$.
\end{theorem}
\begin{proof} Fix $S \in X^m\times X^n$. Let $w_1, w_2 \in \R^d$. Let $\lambda \in (0,1)$, and $\widetilde{w} = \lambda w_1 + (1-\lambda) w_2$. We wish to then show that
\[
\widehat{R}^\hinge_{\pAUC(0,\beta)}(\widetilde{w};\, S)
\,\leq\, \lambda \widehat{R}^\hinge_{\pAUC(0,\beta)}(w_1;\, S) \,+\, (1-\lambda) \widehat{R}^\hinge_{\pAUC(0,\beta)}(w_2;\, S).
\]
Define for any negative $x^-$, $r(w;\, S_+, x^-) = \frac{1}{m}\sum_{i=1}^m \1 \big( w^\top x_i^+ \leq w^\top x^- \big)$. Notice that $r(w; \,S_+, x^-)$ is convex in $w$, and moreover is monotonically increasing in the score $w^\top x^-$ assigned to $x^-$. Expanding the left hand side from the desired inequality,
\begin{eqnarray*}
{\widehat{R}^\hinge_{\pAUC(0,\beta)}(\widetilde{w};\, S)}
 &=&
\frac{1}{j_\beta}\sum_{j = 1}^{j_\beta} 
		 r(\widetilde{w}; S^+, x_{(j)_{\widetilde{w}}}^-)
		\\
&\leq&
\lambda 
\frac{1}{j_\beta}\sum_{j = 1}^{j_\beta} 
		 r(w_1; S^+, x_{(j)_{\widetilde{w}}}^-)
		\,\,+\,\,
(1-\lambda)
\frac{1}{j_\beta}\sum_{j = 1}^{j_\beta} 
		 r(w_2; S^+, x_{(j)_{\widetilde{w}}}^-)
		 \\
&\leq&
\lambda 
\frac{1}{j_\beta}\sum_{j = 1}^{j_\beta} 
		 r(w_1; S^+, x_{(j)_{w_1}}^-)
		\,\,+\,\,
(1-\lambda)
\frac{1}{j_\beta}\sum_{j = 1}^{j_\beta} 
		 r(w_2; S^+, x_{(j)_{w_2}}^-)
\\
&=&
 \lambda \widehat{R}^\hinge_{\pAUC(0,\beta)}(w_1;\, S) \,+\, (1-\lambda) \widehat{R}^\hinge_{\pAUC(0,\beta)}(w_2;\, S).
\end{eqnarray*}
The second step follows by convexity of $r$ (notice that the negative instances here are still ranked by $\widetilde{w}$). The third step uses the fact that $r(w_1; S^+, x^-)$ is monotonic in $w_1^\top x^-$ and $r(w_2; S^+, x^-)$ is monotonic in $w_2^\top x^-$; hence the top ranked negative $j_\beta$ instances according to $w_1$ will yield higher summation value than those ranked according to $\widetilde{w}$, and similarly for $w_2$. This completes the proof.
\end{proof}

Despite the hinge surrogate for FPR intervals of the form $[0,\beta]$ being convex, it is not immediate how the resulting optimization problem can be solved efficiently. For instance, a common approach for optimizing the full AUC surrogate is to derive and solve the corresponding dual optimization problem. Since the surrogate for the partial AUC is defined on a subset of negative instances that can be different for different scoring functions, even deriving the dual problem for the hinge partial AUC surrogate turns out to be non-trivial.

\subsection{Na\"{i}ve Structural SVM Surrogate}
As an alternative to the hinge surrogate, we next consider constructing a convex surrogate for the partial AUC by a direct application of the structural SVM formulation described in the previous section for the AUC. 
%
Specifically, we consider the surrogate obtained by replacing the loss term in the structural SVM surrogate for the AUC in \Eqn{eqn:auc-struct} with an appropriate loss for the partial AUC in $[\alpha, \beta]$:
\begin{eqnarray*}
\Delta_{\pAUC}(\pi^*,\pi) =  \frac{1}{m(j_\beta-j_\alpha)}\sum_{i = 1}^m \sum_{j=j_\alpha+1}^{j_\beta} \pi_{i,(j)_\pi},
\end{eqnarray*}
where $(j)_\pi$ denotes the index of the $j$-th ranked negative instance by any fixed ordering of instances consistent with $\pi$; the resulting surrogate is given by:
\begin{equation}
\widehat{R}^\text{struct}_{\pAUC(\alpha,\beta)}(w; S)\,=\,\underset{\pi \,\in\, \Pi_{m,n}}{\max}\big\{\Delta_{\pAUC}(\pi^*,\pi) \,-\, (w^\top\phi(S,\pi^*) - w^\top \phi(S,\pi))\big\}.
\label{eqn:Naive}
\end{equation}
As with the AUC, this surrogate serves as a convex upper bound for the partial AUC risk in \Eqn{eqn:emp-pauc-risk-prev}. At first glance, this surrogate does appear as a good proxy for the partial AUC risk. However, on closer look, one can show that this surrogate does have drawbacks, as explained below.

Recall that with the AUC, the structural SVM surrogate is equivalent to the corresponding hinge surrogate (see \Thm{thm:auc-struct}). However, even for the special case of partial AUC in FPR intervals of the form $[0,\beta]$ (where the hinge surrogate is convex), the above structural SVM surrogate turns out to be looser convex upper bound on the partial AUC risk than the hinge surrogate in \Eqn{eqn:emp-pauc-hinge-general}. In particular, one can show that in its simplified form, the above structural SVM surrogate for the partial AUC contains redundant terms that penalize misrankings of the scoring function with respect to negative instances outside the relevant FPR range, and in particular, in positions $j_\beta + 1, \ldots, n$ of the ranked list. These additional terms appear because the joint feature map $\phi$ in the surrogate is defined on all negative instances and not just the ones relevant to the given FPR range (see \Eqn{eqn:psi}). Clearly, these terms disrupt the emphasis of the surrogate on the specified FPR interval. The details can be found in the earlier conference versions of this paper \cite{NarasimhanAg13a,NarasimhanAg13b} and are left out here to keep the exposition simple. 

Thus a na\"{i}ve application of the structural SVM formulation yields a loose surrogate for the partial AUC. Of course, one could look at tightening the surrogate by restricting the joint feature map to only a subset of negative instances; however, it is not immediate how this can be done, as the subset of negatives relevant to the given FPR interval can be different for different scoring models $w$, while the definition of the joint feature map in the structural SVM framework needs to be independent of $w$. 

The approach that we take constructs a tighter surrogate for the partial AUC by making use of the structural SVM framework in a manner that suitably exploits the structure of the partial AUC performance measure. In particular, we first rewrite the partial AUC risk as a maximum of  a certain term over subsets of negatives, and compute a convex approximation to the inner term using the structural SVM setup; in the rewritten formulation, the joint feature maps need to be defined on only a subset of the negative instances. The resulting surrogate is convex and is equivalent to the corresponding hinge surrogate for $[0,\beta]$ intervals in \Eqn{eqn:emp-pauc-hinge-beta}. For general FPR intervals $[\alpha,\beta]$, the proposed surrogate can be seen as a tighter convex relaxation to the partial AUC risk compared to the na\"{i}ve structural SVM surrogate. 

\begin{figure}
\begin{table}[H]
\centering
\begin{tabular}{|cc|}
\hline
(a)
&
$
\widehat{R}_{\AUC}(w; S)
\,\leq\,
\widehat{R}^\hinge_{\AUC}(w; S)
\,=\,
\widehat{R}^\struct_{\AUC}(w; S)
$
\\[8pt]
(b)&
$
\widehat{R}_{\pAUC(0,\beta)}(w; S)
\,\leq\,
\widehat{R}^\hinge_{\pAUC(0,\beta)}(w; S)
\,=\,
{\color{blue}
\widehat{R}^\tight_{\pAUC(0,\beta)}(w; S)
}
\,\leq\,
\widehat{R}^\text{struct}_{\pAUC(0,\beta)}(w; S)
$
\\[8pt]
(c)&
$
\widehat{R}_{\pAUC(\alpha,\beta)}(w; S)
\,\leq\,
{\color{red}\widehat{R}^\hinge_{\pAUC(\alpha,\beta)}(w; S)}
\,\leq\,
{\color{blue}
\widehat{R}^\tight_{\pAUC(\alpha,\beta)}(w; S)
}
\,\leq\,
\widehat{R}^\text{struct}_{\pAUC(\alpha,\beta)}(w; S)
$
\\\hline
\end{tabular}
\end{table}
\caption{Relationship between surrogates for (a) AUC and for partial AUC in (b) $[0,\beta]$ and (c) $[\alpha,\beta]$. Here $w \in \R^d$ is a fixed model vector and $S$ is the given training sample. Those colored in blue are the tight convex structural SVM surrogates that we optimize using the cutting plane method (see Sections \ref{sec:svm-beta} and \ref{sec:svm-general}); the one in red is the non-convex hinge surrogate that we optimize using a DC programming method (see Section \ref{sec:svm-dc}).}
\label{fig:surrogates}
\vspace{5pt}
\end{figure}

A summary of the surrogates discussed here is given in Figure \ref{fig:surrogates}. Among the surrogates considered, the hinge surrogates serve as the tightest upper bound on the partial AUC risk, but are not necessarily convex; on the other hand, the na\"{i}ve structural SVM surrogates are convex, but yield a looser upper bound. The structural SVM based surrogates proposed in this paper (highlighted in blue) are convex and also serve as tighter upper bounds compared to the na\"{i}ve surrogates; moreover in the case of $[0,\beta]$ intervals, the proposed surrogate is equivalent to the corresponding hinge surrogate.

We also provide a cutting plane method to optimize the prescribed surrogates. Unlike the full AUC, here the combinatorial optimization required in each iteration of the solver does not decompose easily into simpler problems. One of our main contributions is a polynomial time algorithm for solving this combinatorial optimization for the partial AUC. The details are provided for the $[0,\beta]$ case in Section \ref{sec:svm-beta} and for the $[\alpha, \beta]$ case in Section \ref{sec:svm-general}. In addition to methods that optimize convex structural SVM surrogates on the partial AUC, we also develop a method for directly optimizing the non-convex hinge surrogate for general FPR intervals using difference-of-convex programming; this approach is explained in Section \ref{sec:svm-dc}.
\section{Structural SVM Approach for Partial AUC in $[0, \beta]$}
\label{sec:svm-beta}
 We start by developing a method for optimizing the partial AUC in  FPR intervals $[0, \beta]$:
\begin{eqnarray}
\widehat{R}_{\pAUC(0, \beta)}(w; \, S) \,=\,
		\frac{1}{m j_\beta} \sum_{i = 1}^m\sum_{j = 1}^{j_\beta} 
		\1 \big( w^\top x_i^+ \leq w^\top x_{(j)_w}^- \big)
		.
\label{eqn:emp-pauc-risk-beta}
\end{eqnarray} 
  We saw in the previous section that the hinge loss based surrogate is convex in this case, but it was not immediate how this objective can be optimized efficiently. We also saw that a na\"{i}ve application of the structural SVM framework results in a surrogate that is a looser convex approximation to the partial AUC risk than the hinge surrogate. 
  
  Our approach makes use of the structural SVM formulation in a manner that allows us to construct a tighter convex surrogate for the partial AUC, which in this case is equivalent to the corresponding hinge surrogate. The key idea here  is that the partial AUC risk in $[0,\beta]$ can be written as maximum over subsets of negatives of the full AUC risk evaluated on all positives and the given subset of negatives. The structural SVM formulation described earlier in Section \ref{sec:prelims} for the full AUC can then be leveraged to design a convex surrogate, and to optimize it efficiently  using a cutting plane solver. 

\subsection{Tight Structural SVM Surrogate for pAUC in $[0, \beta]$} 
For any subset of negatives $Z \subseteq S_-$, let $\widehat{R}_{\AUC}(w; \, S_+, \, Z)$ denote the full AUC risk of scoring function $w^\top x$ evaluated on a sample containing all the positives $S_+$ and the subset of negatives $Z$. Then the partial AUC risk of $w^\top x$ is simply the value of this quantity on the top ranked $j_\beta$ negatives; this can be shown to be equivalent to the maximum value of $\widehat{R}_{\AUC}(w; \, S_+, \, Z)$ over all subsets of negatives $Z$ of size $j_\beta$.
\begin{theorem} 
\label{thm:pauc-special-max}
For any $w \in \R^d$ and training sample $S = (S_+, S_-) \in X^m \times X^n$,
\begin{eqnarray}
\vspace{-10pt}
\widehat{R}_{\pAUC(0, \beta)}(w; \, S)
	&=&
	\max_{Z \subseteq S_-, \, |Z| = j_\beta}\,
		\frac{1}{mj_\beta} \sum_{i = 1}^m \sum_{x^- \in Z} 
		\1 ( w^\top x_i^+ \leq w^\top x^- )\nonumber\\
	&=&
	\max_{Z \subseteq S_-, \, |Z| = j_\beta}\,
			\widehat{R}_{\AUC}(w; \, S_+, \, Z)
	.\label{eqn:pauc-inner-auc}
\end{eqnarray}
\end{theorem}
\begin{proof}
Define for any negative $x^-$, $\displaystyle r(w;\, S_+, x^-) = \frac{1}{m}\sum_{i=1}^m \1 \big( w^\top x_i^+ \leq w^\top x^- \big)$. Notice that $r(w; \,S_+, x^-)$ is monotonically increasing in the score $w^\top x^-$ assigned to $x^-$. Thus $\widehat{R}_{\AUC}(w; \, S_+, \, Z) = \frac{1}{j_\beta}\sum_{x^- \in Z}\, r(w; \,S_+, x^-)$ takes the highest value when $Z$ contains the top ranked $j_\beta$ negatives, and by definition (see \Eqn{eqn:emp-pauc-risk-beta}), this maximum value is equal to the partial AUC risk of the given scoring function in the FPR range $[0, \beta]$.
\end{proof}
%

%

Having expressed the partial AUC risk in $[0,\beta]$ in terms of the full AUC risk on a subset of instances, we can devise a convex surrogate for the evaluation measure by constructing  a convex approximation to the full AUC term using the structural SVM formulation explained earlier in Section \ref{subsec:struct-auc}. 

In particular, let us define \textit{truncated} ordering matrices $\pi\in \{0,1\}^{m\times j_\beta}$ for positive instances $S_+$ and any subset $Z = \{z_1, \ldots, z_{j_\beta}\} \subseteq S_-$ of negative instances as 
\[
\pi_{ij} = \left\{ \begin{array}{ll}
	1 & ~~\mbox{if $x^+_i$ is ranked {below} $z_j$} \\[2pt]
	0 & ~~\mbox{if $x^+_i$ is ranked above $z_j$.} \\
\end{array} \right.
\] 
The set of all valid orderings is denoted as $\Pi_{m,j_\beta}$, and the correct ordering is given by $\pi^* = \mathbf{0}^{m\times j_\beta}$. Also redefine the joint feature map for $m$ positive and $j_\beta$ negatives: $\phi: X^{m} \times X^{j_\beta} \> \R^d$. As seen earlier, the following is then a convex upper bound on the inner AUC term in \Eqn{eqn:pauc-inner-auc}:
\[
\underset{\pi \,\in\, \Pi_{m,j_\beta}}{\max}\big\{\Delta_{\AUC}(\pi^*,\pi) \,-\, w^\top\big(\phi((S_+,Z),\pi^*) - \phi((S_+,Z),\pi\big)\big\}.
\]
Replacing the AUC term in \Eqn{eqn:pauc-inner-auc} with the above expression gives us a surrogate that upper bounds the partial AUC risk in $[0,\beta]$:

\begin{eqnarray}
\widehat{R}^\tight_{\pAUC(0, \beta)}(w; \, S)
	& 
	=
	\displaystyle
	\max_{Z \subseteq S_-, \, |Z| = j_\beta}\,
	\underset{\pi \,\in\, \Pi_{m,j_\beta}}{\max}\big\{\Delta_{\AUC}(\pi^*,\pi) \,-\, w^\top\big(\phi((S_+,Z),\pi^*) - \phi((S_+,Z),\pi\big)\big\},
	\label{eqn:pauc-risk-struct-beta-2}
\end{eqnarray}
where the $\pi_{ij}$'s index over all positive instances, and over negative instances in the corresponding subset $Z$ in the outer argmax. 

Clearly, the prescribed surrogate objective is convex in $w$ as it is a maximum of convex functions in $w$. 
In fact, this surrogate is equivalent to the corresponding hinge surrogate for partial AUC in $[0,\beta]$ in \Eqn{eqn:emp-pauc-hinge-beta}. More specifically, we know from Theorem \ref{thm:auc-struct} that the structural SVM expression used above to approximate the inner full AUC term is same as the hinge surrogate for the AUC:
%
\begin{eqnarray}
{
\widehat{R}^\tight_{\pAUC(0, \beta)}(w; \, S)
}	& = &
	\max_{Z \subseteq S_-, \, |Z| = j_\beta}\,
	\frac{1}{mj_\beta} \sum_{i = 1}^m \sum_{x^- \in Z} 
		\big(1 - ( w^\top x_i^+ - w^\top x^- )\big)_+.
	\label{eqn:pauc-hinge}
\end{eqnarray}
At first glance, this appears different from the hinge surrogate for $[0,\beta]$ range in \Eqn{eqn:emp-pauc-hinge-beta}. However, as seen next, the above maximum in attained at the top $j_\beta$ negatives according to $w$, which clearly implies that the two surrogates are equivalent.
\begin{prop}
\label{prop:pauc-struct-max-special}
Let $\bar{Z} = \{\bar{z}_1, \ldots, \bar{z}_{j_\beta}\}$ be the set of negative instances ranked in the top $j_\beta$ positions (among all negative instances in $S_-$, in descending order of scores) by $w^\top x$. Then the maximum value of the objective in \Eqn{eqn:pauc-hinge} (or equivalently in \Eqn{eqn:pauc-risk-struct-beta-2}) is attained at $\bar{Z}$.
\end{prop}
\begin{proof}
Define for any $x^-$, $\displaystyle r(w;\, S_+, x^-) = \frac{1}{m}\sum_{i=1}^m \big(1 - ( w^\top x_i^+ - w^\top x^- )\big)_+$. The proof then follows the same argument used in proving \Thm{thm:pauc-special-max} and uses the fact that $r$ is monotonically increasing in the scores on negative instances.
\end{proof}

\noindent The following result then follows directly from the above proposition.
\begin{theorem}
For any $w \in \R^d$ and training sample $S \in X^m \times X^n$, $\widehat{R}^\tight_{\pAUC(0,\beta)}(w; \, S) 
\,=\, \widehat{R}^\hinge_{\pAUC(0,\beta)}(w; \, S).
$
\label{thm:pauc-struct}
\end{theorem}

 Also notice that unlike the na\"{i}ve structural SVM surrogate in \Eqn{eqn:Naive}, the joint feature map $\phi$ in the proposed surrogate in \Eqn{eqn:pauc-risk-struct-beta-2} is not defined on all negatives, but only on a subset of negatives; consequently, this surrogate does not contain additional redundant terms, and is thus tighter than the na\"{i}ve surrogate \cite{NarasimhanAg13b} (see Figure \ref{fig:surrogates}).  We shall next develop a cutting plane method for optimizing this surrogate.

\begin{figure}[t]
\begin{algorithm}[H]
\caption{Cutting Plane Method for $\text{SVM}_{\pAUC}$ in $[\alpha, \beta]$}
\label{algo:cutting-plane}
\begin{algorithmic}[1]
\small{
\STATE \textbf{Inputs:} $S = (S_+, S_-), ~\alpha, ~\beta, ~C, ~\epsilon$ 
\STATE \textbf{Initialize:} 
\STATE 
If $\alpha = 0$:
\STATE \hspace{0.5cm} 
	$H_w(Z, \pi) \,\equiv\, \Delta_{\AUC}(\pi^*,\pi) - 
	w^\top \big( \phi((S_+, Z), \pi^*) - \phi((S_+, Z),\pi) \big)$
\STATE 
else:
\STATE \hspace{0.5cm}
$H_w(Z, \pi) \,\equiv\, \Delta^\text{tr}_{\pAUC}(\pi^*,\pi) - 
	w^\top \big( \phi((S_+, Z), \pi^*) - \phi((S_+, Z),\pi) \big)$
	~~(see \Eqn{eqn:pauc-loss})
\STATE 
$\mathcal{C} = \emptyset$
\STATE \textbf{Repeat} 
\STATE \hspace{0.5cm} 
	$(w, \xi) = \underset{w, \,\xi \geq 0}{\operatorname{argmin}} \, \frac{1}{2}||w||^2 + C\xi$ 
	$\hspace{0.25cm}  ~~\text{s.t.}~~ \hspace{0.25cm}  \forall (Z, \pi) \in \C:~~ \xi \geq H_w(Z, \pi)$
\STATE \hspace{0.5cm}
	$(\bar{Z}, \bar{\pi}) = 
	\underset{\substack{Z \subseteq S_-, \, |Z| = j_\beta\\\pi  \,\in\,\Pi_{m,j_\beta}}}
	{\operatorname{argmax}} ~H_w(Z, \pi)$
\hspace{0.5cm}(compute the most violated constraint)
\STATE \hspace{0.5cm} $\C = \C \cup \{(\bar{Z},\bar{\pi})\}$
\STATE \textbf{Until} ~$H_w(\bar{Z}, \bar{\pi}) \,\leq\, \xi + \epsilon$
\STATE \textbf{Output:} $w$ 
}
\end{algorithmic}
\end{algorithm}	
\end{figure}

\subsection{Cutting Plane Method for Optimizing $\widehat{R}^\tight_{\pAUC(0,\beta)}$}
We would like to minimize the proposed surrogate in \Eqn{eqn:pauc-risk-struct-beta-2} with an additional regularization term on $w$. This yields the (convex) quadratic program given below.
\begin{equation*}
	\min_{w,\, \xi \geq 0} \, \frac{1}{2}||w||^2 + C\xi
\label{opt:structpauc}
\end{equation*}
\vspace{-1cm}
\begin{eqnarray*}
	\lefteqn{\text{s.t.}~~\forall Z \subseteq S_-, \, |Z| = j_\beta, ~~ \forall \pi\in\Pi_{m,j_\beta}:~~}\\
	&\hspace{1cm}&
	w^\top \big( \phi((S_+, Z), \pi^*) - \phi((S_+, Z),\pi) \big) ~ \geq ~
	\Delta_{\AUC}(\pi^*,\pi) - \xi 
	\,.
\end{eqnarray*}

Notice that the optimization problem has an exponential number of constraints, one for each subset of negative instances of size $j_\beta$ and matrix $\pi \in\Pi_{m,j_\beta}$. As with the full AUC, we use the cutting plane method to solve this problem. The idea behind this method is that for any $\epsilon>0$, a small subset of the constraints is sufficient to find an $\epsilon$-approximate solution to the problem \cite{lintime}. In particular, the method starts with an empty constraint set $\C=\emptyset$, and on each iteration, adds the most violated constraint to $\C$, and solves a tighter relaxation of the optimization problem in the subsequent iteration; this continues until no constraint is violated by more than $\epsilon$  (see Algorithm \ref{algo:cutting-plane}).

It is known that for any fixed regularization parameter $C>0$ and tolerance  $\epsilon > 0$, 
the cutting plane method converges in $O(C/\epsilon + \log(1/C))$ iterations, and will yield a surrogate value within $\epsilon$ of the minimum value \cite{lintime}. Since in each iteration, the quadratic program needed to be solved grows only by a single constraint, the primary bottleneck in the algorithm is the combinatorial optimization (over subsets of negatives and ordering matrices) required to find the most violated constraint (line 10). 

\textbf{Finding most-violated constraint.} The specific  combinatorial optimization problem that we wish to solve can be stated as:
\begin{equation*}
\underset{\substack{Z \subseteq S_-, \, |Z| = j_\beta\\\pi  \,\in\,\Pi_{m,j_\beta}}}
	{\operatorname{argmax}}\big\{
		\Delta_{\AUC}(\pi^*,\pi) \,-\, w^\top\big(\phi((S_+,Z),\pi^*) - \phi((S_+,Z),\pi\big)
	\big\}.
	\label{eqn:struct-special}
\end{equation*}
In the case of AUC, where $j_\beta = n$, the  above argmax is only over ordering matrices in $\Pi_{m,n}$, and can be easily computed by exploiting the additive form of the AUC loss and in particular, by neatly decomposing the problem into one where each $\pi_{ij}$ can be chosen independently \cite{svmperf}. In the case of the partial AUC in $[0,\beta]$, the decomposition is not as straightforward as the argmax is also over subsets of negatives.

\begin{figure}[t]
\begin{algorithm}[H]
\caption{Find Most-Violated Constraint for pAUC in $[0, \beta]$}
\label{algo:mvc}
\begin{algorithmic}[1]
\small{
\STATE \textbf{Inputs:} $S = (S_+, S_-)$, $\beta$, $w$
\STATE Set $\bar{Z} = \{\bar{z}_1, \ldots, \bar{z}_{j_\beta}\}$ as the set of instances ranked in the top $j_\beta$ positions among $S_-$ (in descending order of scores) by $w^\top x$ 
\STATE
$\bar{\pi}_{ij}  =  \1( w^\top x^+_i \,-\, w^\top \bar{z}_{j} \leq 1 )
~~~
\forall i \in \{1, \ldots, m\}, ~j \in \{1, \ldots, j_\beta\}$
%
\STATE \textbf{Output:} $\bar{Z}, \bar{\pi}$
}
\end{algorithmic}
\end{algorithm}	
\end{figure}

\textbf{Reduction to simpler problems.}  We however know from \Prop{prop:pauc-struct-max-special} that the above argmax is attained at the top $j_\beta$ negatives  $\bar{Z} = \{\bar{z}_1, \ldots, \bar{z}_{j_\beta}\}$ according to $w$, and all that remains is to compute the optimal ordering matrix in $\Pi_{m,j_\beta}$ keeping $\bar{Z}$ fixed; the optimization problem can then be decomposed easily.
 In particular, having fixed the subset $\bar{Z}$, the combinatorial optimization problem becomes equivalent to:
\begin{equation}
	\underset{\pi  \,\in\,\Pi_{m,j_\beta}}
	{\operatorname{argmax}}\,\frac{1}{mj_\beta}\sum_{i=1}^m\sum_{j=1}^{j_\beta} \pi_{ij}\big(1 \,-\, w^\top( x^+_i -  \bar{z}_j)\big).
	\tag{OP1}	
	\label{eqn:mvc-special}
\end{equation}
  Now consider solving a relaxed form of \ref{eqn:mvc-special} over all matrices in $\{0,1\}^{m\times j_\beta}$. The objective now decomposes into a sum of terms involving individual elements $\pi_{i,j} \in \{0,1\}$ and can be maximized by optimizing each term separately; the optimal matrix is then given by $\bar{\pi}_{ij} = \1(w^\top x^+_i \,-\,  w^\top\bar{z}_j \,\leq\, 1)$. It can be seen that this optimal matrix $\bar{\pi}$ is in fact a valid ordering matrix in $\Pi_{m,j_\beta}$, as it corresponds to ordering of instances where the positives are scored according to $w^\top x$ and the negatives are scored according to $w^\top x + 1$. Hence $\bar{\pi}$ is also a solution to the original unrelaxed form of \ref{eqn:mvc-special} for fixed $\bar{Z}$, and thus $(\bar{Z}, \bar{\pi})$ gives us the desired most-violated constraint.

\textbf{Time complexity.} A straightforward implementation to compute the above solution (see Algorithm \ref{algo:mvc}) would take computational time $O(mj_\beta + n\log(n))$ (assuming score evaluations on instances can be done in unit time). Using a more compact representation of the orderings \cite{svmperf}, however, this can be further reduced to $O((m+j_\beta)\log(m+j_\beta) + n\log(n))$. The details can be found in Appendix \ref{app:mvc-efficient}. Thus computing the most-violated constraint for the partial AUC in a small interval $[0,\beta]$ is faster than that for the full AUC \cite{svmperf}; this is because the number of negative instances relevant to the given FPR range over which the most-violated constraint is computed is smaller for the partial AUC. On the other hand, it turns out that in practice, the number of iterations required by the cutting plane method to converge (and in turn the number of calls to the given combinatorial optimization) is often higher for partial AUC compared to AUC; we will elaborate this when we discuss our experimental results in Section \ref{sec:expts}.


We have presented an efficient method for optimizing the structural SVM surrogate for the partial AUC in the $[0,\beta]$ range,  which we saw was equivalent to the hinge surrogate. We next proceed to algorithms for optimizing partial AUC in a general FPR interval $[\alpha, \beta]$. 


\section{Structural SVM Approach for Partial AUC in $[\alpha, \beta]$}
\label{sec:svm-general}
Recall that the partial AUC in a general FPR interval $[\alpha, \beta]$ is given by:
\begin{eqnarray}
\widehat{R}_{\pAUC(\alpha, \beta)}(w; \, S) \,=\,
		\frac{1}{m (j_\beta-j_\alpha)} \sum_{i = 1}^m\sum_{j = j_\alpha+1}^{j_\beta} 
		\1 \big( w^\top x_i^+ \leq w^\top x_{(j)_w}^- \big)
		.
\label{eqn:emp-pauc-risk}
\end{eqnarray}
We have already seen in Section \ref{sec:surrogates} that in this case, the simple hinge surrogate (obtained by replacing the indicator terms in the above risk by the pairwise hinge loss) is not necessarily convex. We have also seen that a na\"{i}ve application of the structural SVM formulation to the above risk yields a surrogate with redundant additional terms involving negative instances outside the specified FPR range. As with the $[0,\beta]$ case, we now apply the structural SVM framework in a manner that yields a tighter convex surrogate for the partial AUC risk; of course, in this case, the resulting convex  surrogate is not equivalent to the non-convex hinge surrogate, but as we explain later, can be seen as a convex relaxation to the hinge surrogate. 

Again, the main idea here is to rewrite the partial AUC risk as a maximum of a certain term over subsets of negatives, and use the structural SVM formulation to compute a convex approximation to the inner term. We provide an efficient cutting plane method for solving the resulting optimization problem; here the combinatorial optimization step for finding the most-violated constraint in the cutting plane solver does not admit a decomposition involving individual matrix entries. We show that by using a suitable reformulation of the problem over a restricted search space, the optimization can be still reduced into simpler ones, but now involving individual rows of the ordering matrix. 
 In Section \ref{sec:svm-dc}, we shall also look at an approach for directly optimizing the non-convex hinge  surrogate for general FPR ranges. 


\subsection{Tight Structural SVM Surrogate for pAUC in $[\alpha, \beta]$}
We begin by describing the  construction of the tight structural SVM surrogate. 
Just as the partial AUC risk in $[0,\beta]$ could be written as a maximum over subsets of negative instances of the full AUC risk evaluated on this subset (see \Thm{thm:pauc-special-max}), the partial AUC risk in $[\alpha,\beta]$ can also be written as a maximum of a certain term over subsets of negative instances of size $j_\beta$.
\begin{theorem} 
\label{thm:pauc-general-max}
For any $w \in \R^d$ and training sample $S = (S_+, S_-) \in X^m \times X^n$,
\begin{eqnarray*}
\widehat{R}_{\pAUC(\alpha, \beta)}(w; \, S)
	&=&
	\max_{Z \subseteq S_-, \, |Z| = j_\beta}\,
			\widetilde{R}(w; \, S_+, \, Z)
	,
\end{eqnarray*}
where for any subset of negative instances $Z = \{z_1, \ldots, z_{j_\beta}\}$ that (w.l.o.g.) satisfy $w^\top z_1 \geq \ldots \geq w^\top z_{j_\beta}$, 
$$ \widetilde{R}(w; \, S_+, \, Z) = \frac{1}{m(j_\beta-j_\alpha)}\sum_{i=1}^m\sum_{j=j_\alpha+1}^{j_\beta} \1 \big( w^\top x_i^+ \leq w^\top z_{j} \big).$$
\end{theorem}
\begin{proof}
As in \Thm{thm:pauc-special-max}, define for any $x^-$, $
\displaystyle
r(w;\, S_+, x^-) = \frac{1}{m}\sum_{i=1}^m \1 \big( w^\top x_i^+ \leq w^\top x^- \big)$. 
Thus $\widetilde{R}(w; \, S_+, \, Z)$ 
evaluates to the (scaled) average value of this quantity on the bottom ranked $j_\beta - j_\alpha$ negatives within $Z$ by $w$. Moreover, $r(w; \,S_+, x^-)$ is monotonically increasing in the score $w^\top x^-$ assigned to $x^-$. As a result, $\widetilde{R}(w; \, S_+, \, Z)$ takes the highest value when $Z$ contains negatives in the top $j_\beta$ positions in the ranking of all negatives in $S_-$ by $w$. By the definition in \Eqn{eqn:emp-pauc-risk}, this maximum value is equal to the partial AUC risk of the scoring function in $[\alpha, \beta]$. 
\end{proof}

Note when $\alpha = 0$, the term $\widetilde{R}(w; \, S_+, \, Z)$ is the full AUC risk on the sample $(S_+, Z)$, recovering our previous result in Theorem \ref{thm:pauc-special-max}. In this case, we directly made use of the structural SVM formulation for the full AUC to construct a convex approximation for this term. However, when $\alpha > 0$,  $\widetilde{R}(w; \, S_+, \, Z)$ is more complex and can be essentially seen as (a scaled version of) partial AUC risk in the FPR range $\big[{j_\alpha}/{j_\beta}, 1\big]$ defined on a subset of instances $(S_+, Z)$; we will hence have to rework the structural SVM formulation for $\widetilde{R}$, as described next.

\textbf{Convex upper bound on $\widetilde{R}$.} 
In particular, we describe how the structural SVM framework can be used to construct a convex upper bound on the inner term $\widetilde{R}$, and thus obtain a convex surrogate for the partial AUC risk in $[\alpha, \beta]$. Restricting ourselves to \textit{truncated} ordering matrices $\pi \in \Pi_{m,j_\beta}$ defined for $m$ positives and a subset of $j_\beta$ negatives, let us again use $(j)_\pi$ to denote the index of the $j$-th ranked negative instance by any fixed ordering of instances consistent with $\pi$ (note that all such orderings yield the same value of $\widetilde{R}$). 
We further define the loss term for the truncated ordering matrices:
\begin{eqnarray}
\Delta^\text{tr}_\pAUC(\pi^*,\pi) =  \frac{1}{m(j_\beta-j_\alpha)}\sum_{i = 1}^m \sum_{j=j_\alpha+1}^{j_\beta} \pi_{i,(j)_\pi}.
\label{eqn:pauc-loss}
\end{eqnarray}
Given that this loss is defined on a subset of instances, as noted above, it can be seen as the partial AUC loss in a scaled interval $\big[{j_\alpha}/{j_\beta}, 1\big]$. The following is then a convex upper bound on $\widetilde{R}$:
\begin{eqnarray*}
	\underset{\pi \,\in\, \Pi_{m,j_\beta}}{\max}\big\{\Delta^\text{tr}_\pAUC(\pi^*,\pi) \,-\, w^\top\big(\phi((S_+,Z),\pi^*) - \phi((S_+,Z),\pi\big)\big\},
\end{eqnarray*}
and replacing $\widetilde{R}$ in the rewritten partial AUC risk in \Thm{thm:pauc-general-max} with the above expression gives us the following upper bounding surrogate:
\begin{eqnarray}
\widehat{R}^\tight_{\pAUC(\alpha, \beta)}(w; \, S) 
	\,\,=\,\,
	\max_{Z \subseteq S_-, \, |Z| = j_\beta}\,
	\underset{\pi \,\in\, \Pi_{m,j_\beta}}{\max}\big\{	
		\Delta^\text{tr}_\pAUC(\pi^*,\pi) \,-\, w^\top\big(\phi((S_+,Z),\pi^*) - \phi((S_+,Z),\pi\big)\big\}
	.
\label{eqn:pauc-struct-max-general}
\end{eqnarray}
The surrogate is a maximum of convex functions in $w$, and is hence convex in $w$. Even here, it turns out the above maximum is attained by the top $j_\beta$ negatives according to $w$.

\begin{prop}
\label{prop:pauc-struct-max-general}
Let $\bar{Z} = \{\bar{z}_1, \ldots, \bar{z}_{j_\beta}\}$ be the set of instances in the top $j_\beta$ positions in the ranking of negative instances (in descending order of scores) by $w^\top x$. Then the maximum value of the objective in \Eqn{eqn:pauc-struct-max-general} is attained at $\bar{Z}$.
\end{prop}
\begin{proof}
For any subset of negative instances $Z = \{{z}_1, \ldots, {z}_{j_\beta}\} \subseteq S_-$, we assume w.l.o.g. that $w^\top z_1 \geq \ldots \geq w^\top z_{j_\beta}$ (this ensures that the identity of each $z_j$ is unique). Expanding the objective in \Eqn{eqn:pauc-struct-max-general} then gives us: 
\begin{eqnarray*}
{\max_{Z = \{{z}_1, \ldots, {z}_{j_\beta}\} \subseteq S_-}\,
	\max_{\pi \,\in\, \Pi_{m,j_\beta}}
	\frac{1}{m(j_\beta-j_\alpha)}\sum_{i=1}^m\bigg[
\sum_{j=j_\alpha+1}^{j_{\beta}} \pi_{i(j)_\pi} \,-\, 
\sum_{j=1}^{j_{\beta}} \pi_{ij}\, w^\top x_i^+ \,+\, \sum_{j=1}^{j_{\beta}} \pi_{ij}\, w^\top z_j\bigg]}.
\end{eqnarray*}
Interchanging the two max over finite sets, we equivalently have:
\begin{eqnarray*}
\frac{1}{m(j_\beta-j_\alpha)}  
	\max_{\pi \,\in\, \Pi_{m,j_\beta}}
\bigg[
	\sum_{i=1}^m\bigg[
\sum_{j=j_\alpha+1}^{j_{\beta}} \pi_{i(j)_\pi} -
\sum_{j=1}^{j_{\beta}} \pi_{ij}\, w^\top x_i^+ 
\bigg]
\,+\, 
\max_{Z = \{{z}_1, \ldots, {z}_{j_\beta}\}\subseteq S_-}
\sum_{j=1}^{j_{\beta}} q_j\, w^\top z_j
\bigg],
\end{eqnarray*}
where $q_j = \sum_{i=1}^m \pi_{ij}$ is a positive integer between 0 and $m$. Clearly, the only term in the above objective that depends on instances in $Z$ is the third term. For any fixed $\pi$ (or equivalent $q$), this term is maximized when the subset $Z$ contains the negatives with the highest scores by $w$ and in particular, the top $j_\beta$ ranked negatives by $w$.
\end{proof}

Unlike the partial AUC in $[0, \beta]$, the above structural SVM surrogate is not equivalent to the non-convex hinge surrogate in \Eqn{eqn:emp-pauc-hinge-general}  for $[\alpha,\beta]$ intervals and is a looser upper bound on the partial AUC risk. On the other hand, compared to the na\"{i}ve structural SVM surrogate in \Eqn{eqn:Naive} for the $[\alpha,\beta]$ range, the joint feature map here is only defined on a subset of negatives, and as a result, the proposed surrogate is tighter and lays more emphasis on good performance in the given range \cite{NarasimhanAg13b} (see Figure \ref{fig:surrogates}). This will become clear from the characterization provided below. 

\subsection{Characterization for $\widehat{R}^\tight_{\pAUC(\alpha,\beta)}$}
Before proceeding to develop a method for optimizing the proposed structural SVM surrogate for $[\alpha, \beta]$ intervals, we analyze how the surrogate is related to the original partial AUC risk in \Eqn{eqn:emp-pauc-risk}, and to the other surrogates discussed in Section \ref{sec:surrogates}. 
These relationships were obvious for the $[0,\beta]$ case, as the prescribed structural SVM surrogate there was exactly equivalent to the associated hinge surrogate; for the $[\alpha,\beta]$, it is not immediate from the surrogate whether it closely mimics the partial AUC risk. We know so far that the proposed structural SVM surrogate for $[\alpha,\beta]$ intervals upper bounds the partial AUC risk; below we give a more detailed characterization:\footnote{We note that in the characterization result provided in the conference version of this paper \cite{NarasimhanAg13b} (Theorem 2), there are no terms $\eta_{[0,\alpha]}$ and $\eta^+_{[0,\alpha]}$; we have corrected this error here.}
%

\begin{theorem}
\label{thm:surrogate-characterization}
Let $0< \alpha < \beta \leq 1$. Then for any sample $S \in X^m\times X^n$ and $w\in\R^d$:
\begin{eqnarray*}
\widehat{R}^{\hinge,+}_{\pAUC(\alpha,\beta)}(w; S) \,+\, \eta^+_{[0,\alpha]}(w)
~~\leq~~
\widehat{R}^\tight_{\pAUC(\alpha,\beta)}(w; S) ~~ \leq ~~
	\widehat{R}^\hinge_{\pAUC(\alpha,\beta)}(w; S) \,+\, \eta_{[0,\alpha]}(w),
\end{eqnarray*}
where $\widehat{R}^\hinge_{\pAUC(\alpha,\beta)}$ is the hinge surrogate in \Eqn{eqn:emp-pauc-hinge-general} and $\widehat{R}^{\hinge,+}_{\pAUC(\alpha,\beta)}$ is a version of this surrogate defined on a subset of positive instances:
\begin{eqnarray*}
		\widehat{R}^{\hinge,+}_{\pAUC(\alpha,\beta)}(w;S) = 
\frac{1}{m(j_\beta-j_\alpha)} \sum_{i:  w^\top x^+_i < w^\top x^-_{(j_\alpha)_w}}\sum_{j=1}^{j_\alpha} 
			\big(1-w^\top(x_i^+ - x^-_{(j)_w})\big)_+,
\end{eqnarray*}
while $\eta_{[0,\alpha]}$ and $\eta^+_{[0,\alpha]}$ are positive terms which penalize misrankings against negatives in the FPR range $[0,\alpha]$ with a margin of zero:
$$
	\eta_{[0,\alpha]}(w) = 
\frac{1}{m(j_\beta-j_\alpha)} \sum_{i=1}^m\sum_{j=1}^{j_\alpha} 
			\big(-w^\top(x_i^+ - x^-_{(j)_w})\big)_+;
$$
$$
		\eta^+_{[0,\alpha]}(w) = 
\frac{1}{m(j_\beta-j_\alpha)} \sum_{i:  w^\top x^+_i < w^\top x^-_{(j_\alpha)_w}}\sum_{j=1}^{j_\alpha} 
			\big(-w^\top(x_i^+ - x^-_{(j)_w})\big)_+.
$$
Moreover, if $|w^\top x^+_i - w^\top x^-_j| \geq 1\,\, \forall\, i \in \{1,\ldots,m\}, \,j \in \{1,\ldots,n\}$, then the lower and upper bounds match and we have:
\[
\widehat{R}^\tight_{\pAUC(\alpha,\beta)}(w; S) = \widehat{R}^\hinge_{\pAUC(\alpha,\beta)}(w; S) \,+\, \eta_{[0,\alpha]}(w).
\]
\end{theorem}
The proof is provided in Appendix \ref{app:surrogate-characterization-proof}. 
Note that in both the lower and upper bounds, certain positive-negative pairs are penalized with a larger margin than others; in particular, those involving the negative  instances in positions $j_{\alpha} + 1$ to $j_{\beta}$ are penalized with a margin of 1, while the rest are penalized with zero margin. This confirms the surrogate's focus on a select portion of the ROC curve in the range $[\alpha,\beta]$. 

Further, if $w \in \R^d$ is such that the difference in scores assigned to any pair of positive-negative training instances is either greater than 1 or lesser than $-1$ (which is indeed the case when $w$ has a sufficiently large norm and the training instances are all unique), then the characterization is more precise. Here the structural SVM surrogate is exactly equal to the sum of two terms; the first is the non-convex hinge surrogate; the second is a positive term $\eta_{[0,\alpha]}$ that penalizes misrankings w.r.t.\ negatives in positions $1$ to $j_{\alpha}$, and which can be seen as enforcing convexity in the surrogate. 

While the proposed structural SVM surrogate is not equivalent to the hinge surrogate, it can clearly be interpreted as a convex approximation to the hinge surrogate for the $[\alpha,\beta]$ range. Also, a similar characterization for the na\"{i}ve structural SVM surrogate in \Eqn{eqn:Naive} contains additional terms involving negative instances ranked in positions $j_\beta+1, \ldots, n$ outside the specified FPR range \cite{NarasimhanAg13b}; the proposed surrogate does not contain these terms and is therefore a tighter upper bound on the partial AUC risk  (also see Figure \ref{fig:surrogates}).


\subsection{Cutting Plane Method for Optimizing $\widehat{R}^\tight_{\pAUC(\alpha,\beta)}$}
Having constructed a tight structural SVM surrogate for $[\alpha,\beta]$ intervals, we now provide an efficient method to optimize it. In particular, we would like to minimize a regularized form of \Eqn{eqn:pauc-struct-max-general}, which results in the following quadratic program:
\begin{equation*}
	\min_{w,\, \xi \geq 0} \, \frac{1}{2}||w||^2 + C\xi
\end{equation*}
\vspace{-1cm}
\begin{eqnarray*}
	\lefteqn{\text{s.t.}~~\forall Z \subseteq S_-, \, |Z| = j_\beta, ~~ \forall \pi\in\Pi_{m,j_\beta}:~~}\\
	&\hspace{1cm}&
	w^\top \big( \phi((S_+, Z), \pi^*) - \phi((S_+, Z),\pi) \big) ~ \geq ~
	\Delta^\text{tr}_\pAUC(\pi^*,\pi) - \xi 
	.
\end{eqnarray*}
Since the optimization problem has an exponential number of constraints, we once again employ the cutting plane method for solving it (see Algorithm \ref{algo:cutting-plane}). Recall that the crucial step in the cutting plane solver is to  efficiently compute the most-violated constraint in each iteration. Below, we provide an algorithm for performing this combinatorial optimization within the cutting plane method in polynomial time.

The specific combinatorial optimization problem we wish to solve has the form:
\begin{eqnarray*}
	\underset{\substack{Z \subseteq S_-, \, |Z| = j_\beta\\\pi  \,\in\,\Pi_{m,j_\beta}}}{\argmax}
		\big\{\Delta^\text{tr}_\pAUC(\pi^*,\pi) \,-\, w^\top\big(\phi((S_+,Z),\pi^*) - \phi((S_+,Z),\pi\big)\big\}.
\end{eqnarray*}
In the case of the full AUC or the partial AUC in $[0, \beta]$, the corresponding combinatorial optimization problem decomposes into simpler problems involving individual $\pi_{ij}$'s. For the general partial AUC, solving this problem is however trickier as the set of negatives involved in the summation in $\Delta^\text{tr}_\pAUC$ is different for different ordering matrices.
In this case, we will no longer be able to optimize each $\pi_{ij}$ independently, as the resulting matrix need not correspond to a valid ordering. Nevertheless, we will be able to formulate an equivalent problem over a restricted search space of ordering matrices, where each \emph{row} of the matrix can be optimized separately and efficiently. 

To this end, we first observe from \Prop{prop:pauc-struct-max-general} that it suffices to maximize the above optimization objective over ordering matrices in $\Pi_{m,j_\beta}$, fixing $Z$ to the subset of top $j_\beta$ ranked negative instances $\bar{Z} = \{\bar{z}_1, \ldots, \bar{z}_{j_\beta}\}$ according to $w$ (where w.l.o.g.\ we assume that $\bar{z}_j$ denotes the $j$-th instance in the ranking of negatives):
\begin{equation}
		\underset{\pi  \,\in\,\Pi_{m,j_\beta}}{\argmax}\,
		\big\{\Delta^\text{tr}_\pAUC(\pi^*,\pi) \,-\, w^\top\big(\phi((S_+,\bar{Z}),\pi^*) - \phi((S_+,\bar{Z}), \pi)\big)\big\}.
		\tag{\textup{OP2}}
		\label{opt:most-violated-general}		
\end{equation}

\textbf{Restricted search space of ordering matrices.} Notice that among the negative instances in $\bar{Z}$, it is only the bottom ranked $j_\beta - j_\alpha$ negatives that appear in $\Delta^\text{tr}_\pAUC$. As noted above, this subset of negative instances is different for different ordering matrices, and hence computing the argmax requires a further reformulation. In particular, we shall next show that the above argmax can be equivalently computed over a restricted set of ordering matrices, given for any $w \in \R^d$ as:
\[
\Pi^w_{m,j_\beta} =  \big\{ \pi \in \Pi_{m,j_\beta} ~\big|~ \forall i, j_1<j_2: \pi_{i,(j_1)_w} \geq \pi_{i,(j_2)_w} \big\}
	\,,
\]
where as before $(j)_w$ denotes the index of the $j$-th ranked negative instance in $S_-$ or equivalently in $\bar{Z}$, when the instances are sorted (in descending order) by $w^\top x$. This is the set of all ordering matrices $\pi$ in which any two negative instances that are separated by a positive instance are sorted according to $w$. We then have:
\begin{theorem}
\vspace{-5pt}
\label{thm:Pi-w}
The solution $\bar{\pi}$ to \ref{opt:most-violated-general} lies in $\Pi^w_{m,j_\beta}$.\footnote{We note that a similar observation was made for the optimizer associated with the mean average precision (MAP) objective in \cite{svmmap}.}
\end{theorem}
\begin{proof}
Suppose $\bar{\pi} \notin \Pi^w_{m,j_\beta}$. Then $\exists i,j_1<j_2$ such that $\bar{\pi}_{i,(j_1)_w} < \bar{\pi}_{i,(j_2)_w}$, i.e.\ such that $\bar{\pi}_{i,(j_1)_w} = 0$ and $\bar{\pi}_{i,(j_2)_w} = 1$. This means that $\bar{\pi}$ ranks $x^+_i$ above $x^-_{(j_1)_w}$ but below $x^-_{(j_2)_w}$. Now let us construct from $\bar{\pi}$ an ordering $\bar{\pi}'$ in which the instances $x^-_{(j_1)_w}$ and $x^-_{(j_2)_w}$ are swapped, i.e.\ for all $i'$ with $\bar{\pi}_{i',(j_1)_w} = 0$ and $\bar{\pi}_{i',(j_2)_w} = 1$, we set $\bar{\pi}'_{i',(j_1)_w} = 1$ and $\bar{\pi}'_{i',(j_2)_w} = 0$. Then it can be seen that while the loss term in the objective in \ref{opt:most-violated-general} is the same for $\bar{\pi}'$ as for $\bar{\pi}$, the second term increases, yielding a higher objective value. This contradicts the fact that $\bar{\pi}$ is a maximizer to \ref{opt:most-violated-general}. 
\end{proof}
It is further easy to see that for any $\pi \in \Pi^w_{m,j_\beta}$, $\pi_{i,(j)_\pi} = \pi_{i,(j)_w}$, as there always exists an ordering consistent with $\pi$ in which the negatives are all sorted according to $w$ (follows from the definition of $\Pi^w_{m,j_\beta}$). As a result, \ref{opt:most-violated-general} can be equivalently framed as:
\begin{eqnarray}
\lefteqn{
	\underset{\pi  \,\in\,\Pi^w_{m,j_\beta}}{\argmax}\,
		\big\{\Delta^\text{tr}_\pAUC(\pi^*,\pi) \,-\, w^\top\big(\phi((S_+,\bar{Z}),\pi^*) - \phi((S_+,\bar{Z}), \pi)\big)\big\}
	}\nonumber\\
	& 
	=
	&
	\underset{\pi  \,\in\,\Pi^w_{m,j_\beta}}{\argmax}\,
		\sum_{i=1}^m\bigg[\sum_{j=j_\alpha+1}^{j_\beta}\pi_{i(j)_w} \,-\, \sum_{j=1}^{j_\beta}\pi_{ij}\,w^\top (x^+_i- \bar{z}_j)\bigg]
		\nonumber\\
	&=&
	\underset{\pi  \,\in\,\Pi^w_{m,j_\beta}}{\argmax}\,
		\sum_{i=1}^m\bigg[
			-\sum_{j=1}^{j_\alpha}\pi_{i(j)_w} w^\top (x^+_i- \bar{z}_{j}) \,+\, \sum_{j=j_\alpha+1}^{j_\beta}\pi_{i(j)_w} \big(1-w^\top (x^+_i- \bar{z}_{j})\big)
		\bigg]\nonumber\\
	&=&		
		\underset{\pi  \,\in\,\Pi^w_{m,j_\beta}}{\argmax}\,
		\sum_{i=1}^m H^i_w(\pi_i)~~~~\text{(say)},
		\label{eqn:mvc-Q}
\end{eqnarray}
where $\pi_i\in\{0,1\}^{j_\beta}$ denotes the $i$-th row of the ordering matrix $\pi$.

\textbf{Reduction to simpler problems.} With this reformulation, it turns out that each {row} $\pi_i$ can be considered separately, and moreover, that the optimization over each $\pi_i$ can be done efficiently. In particular, note that for each $i$, the $i$-th row of the optimal ordering matrix $\bar{\pi}$ to the above problem essentially corresponds to an interleaving of the lone positive instance $x^+_i$ with the list of negative instances sorted according to $w^\top \bar{z}_j$; thus each $\bar{\pi}_i$ is of the form 
\begin{eqnarray}
\bar{\pi}_{i,(j)_w} ~ = ~
	\begin{cases}
		1 & \mbox{if $j \in \{1,\ldots,r_i\}$} \\
		0 & \mbox{if $j \in \{r_i+1,\ldots,j_\beta\}$}
	\end{cases} 
\label{eqn:r_i}
\end{eqnarray}
for some $r_i\in\{0,1,\ldots,j_\beta\}$. In other words, the optimization over $\pi_i \in\{0,1\}^{j_\beta}$ reduces to an optimization over $r_i\in\{0,1,\ldots,j_\beta\}$, or equivalently to an optimization over $\pi_i \in Q^w_i$, where 
$
Q^w_i = \big\{ \pi_i \in \{0,1\}^{j_\beta} ~\big|~ \forall j_1 < j_2: \pi_{i,(j_1)_w} \geq \pi_{i,(j_2)_w} \big\} 
$
with $|Q^w_i| = j_\beta+1$.
Clearly, we have $\Pi^w_{m,j_\beta} = Q^w_1 \times \ldots \times Q^w_m$,
and therefore we can rewrite \Eqn{eqn:mvc-Q} as 
\begin{equation*}
	\underset{\pi\in Q^w_1 \times \ldots \times Q^w_m}{\operatorname{argmax}} \, \sum_{i=1}^m H^i_w(\pi_i) 
	\,.
\hfill{\tag{OP3}}
\label{opt:mvc-final}
\end{equation*}
Since the objective given above decomposes into a sum of terms involving the individual rows $\pi_i$, \ref{opt:most-violated-general} can be solved by maximizing ${H}^i_w(\pi_i)$ over each row $\pi_i\in R^w_i$ separately. 

\begin{figure}[t]
\begin{algorithm}[H]
\caption{Find Most-Violated Constraint for pAUC in $[\alpha, \beta]$}
\label{algo:mvc-general}
\begin{algorithmic}[1]
\small{
\STATE \textbf{Inputs:} $S = (S_+, S_-)$, $\alpha$, $\beta$, $w$
\STATE Set $\bar{Z} = \{\bar{z}_1, \ldots, \bar{z}_{j_\beta}\}$ as the set of top $j_\beta$ ranked instances among $S_-$ (in descending order of scores) by $w^\top x$. Further, let $\bar{z}_j$ denote the negative instance in position $j$ of the ranking.\\
\STATE 
\textbf{For} $i = 1, \dots, m$  \textbf{do}
\STATE \hspace{0.5cm}
Optimize over $r_i \in \{0,\ldots,j_\alpha\}$:
\\
\hspace{0.5cm} 
\(
{\pi}_{i,j}  =  
	\begin{cases}
		\1( w^\top x^+_i \,-\,  w^\top \bar{z}_j \,\leq\, 0 ) \,,	&	j \in \{1,\ldots, j_\alpha\} \\
		0						\,,	& 	j \in \{j_\alpha+1, \ldots, j_\beta\} \\[2pt]
	\end{cases}
\)
\STATE \hspace{0.5cm}
Optimize over $r_i \in \{j_\alpha+1,\ldots,n\}$:
\\
\hspace{0.5cm}
\(
{\pi}'_{i,j} = 
	\begin{cases}
		1	\,,&						j \in \{1, \ldots, j_\alpha\}	    	\\
		\1( w^\top x^+_i \,-\,  w^\top \bar{z}_j \,\leq\, 1 )  \,,&	j \in \{j_\alpha + 1, \ldots, j_\beta\}
		\\[4pt]
	\end{cases}
\)
\STATE \hspace{0.5cm} \textbf{If}~ $H^i_w(\pi_i) > H^i_w(\pi'_i)$ ~\textbf{Then}~ $\bar{\pi}_i = \pi_i$ ~\textbf{Else}~ $\bar{\pi}_i = \pi'_i$
~~(see \ref{opt:mvc-final})
\STATE
\textbf{End For}
\STATE \textbf{Output:} $(\bar{Z}, \bar{\pi})$
}
\end{algorithmic}
\end{algorithm}	
\end{figure}

\textbf{Time complexity.} 
In a straightforward implementation of this optimization, for each $i\in\{1,\ldots,m\}$, one would evaluate the term $H^i_w$ for each of the $j_\beta+1$ values of $r_i$ (corresponding to the $j_\beta+1$ choices of $\pi_i\in Q^w_i$; see \Eqn{eqn:r_i}) and select the optimal among these; each such evaluation takes $O(j_\beta)$ time, yielding an overall time complexity of $O(mj_\beta^2)$. It turns out, however, that one can partition the $j_\beta+1$ values of $r_i$ into two groups, $\{0,\ldots,j_\alpha\}$, and $\{j_\alpha+1,\ldots,j_\beta\}$, such that the optimization over $r_i$ in each of these groups (after the negative instances have been sorted according to $w$) can be implemented in $O(j_\beta)$ time.
 A description is given in Algorithm~\ref{algo:mvc-general}, where the overall time complexity is $O(mj_\beta + n\log(n))$. 
 
 Again, using a more compact representation of the orderings, it is possible to further reduce the computational complexity of Algorithm~\ref{algo:mvc-general} to $O((m+j_\beta)\log(m+j_\beta) + n\log(n))$ (see Appendix \ref{app:mvc-efficient} for details). Note that the time complexity for finding the most-violated constraint for partial AUC (with $\beta < 1$) has a better dependence on the number of training instances compared to that for the usual AUC \cite{svmperf}. However we find in our experiments in Section \ref{sec:expts} that the overall running time of the cutting plane method is often higher for partial AUC compared to full AUC, because the number of calls made to the inner combinatorial optimization routine turns out to be higher in practice for partial AUC.

We thus have an efficient method for optimizing a convex surrogate on the partial AUC risk in $[\alpha,\beta]$. 
%
While the surrogate optimized by our method is a tighter approximation to the partial AUC risk than the na\"{i}ve structural SVM surrogate considered initially in Section \ref{sec:surrogates}, we known from the characterization result in Theorem \ref{thm:surrogate-characterization} that the surrogate does contain terms 
involving negative instances in positions $1, \ldots, j_\alpha$ outside the specified FPR range. On the other hand, the hinge surrogate in \Eqn{eqn:emp-pauc-hinge-general} for $[\alpha,\beta]$ intervals, while being non-convex, serves as a tighter approximation to the partial AUC risk (see Figure \ref{fig:surrogates}). 
This motivates us to next develop a method for directly optimizing the non-convex hinge surrogate for $[\alpha, \beta]$ intervals; we will make use of a popular non-convex optimization technique based on DC programming for this purpose.
\section{DC Programming Approach for Partial AUC in $[\alpha, \beta]$}
\label{sec:svm-dc}
As noted above, for general FPR intervals $[\alpha, \beta]$, the structural SVM based surrogate optimized in the previous section is often a looser relaxation to the partial AUC risk compared to the non-convex hinge loss based surrogate considered earlier in \Eqn{eqn:emp-pauc-hinge-general}. We now develop an approach for directly optimizing the hinge surrogate. Here we resort to a popular difference-of-convex (DC) programming technique, where we shall exploit the fact that the partial AUC in $[\alpha, \beta]$ is essentially a difference between (scaled) partial AUC values in $[0, \beta]$ and $[0, \alpha]$; the structural SVM algorithm developed in Section \ref{sec:svm-beta} for false positive ranges of the form $[0, \beta]$ will be used as a subroutine here. 

\textbf{Difference-of-convex objective.} 
We begin by rewriting the surrogate in \Eqn{eqn:emp-pauc-hinge-general} as a difference of  hinge surrogates in intervals $[0, \beta]$ and $[0, \alpha]$, thus allowing us to write the optimization objective  as a difference of two convex functions in $w$:
\begin{eqnarray*}
\widehat{R}^\hinge_{\pAUC(\alpha,\beta)}(w;\, S)
	&=&
		\frac{1}{j_\beta-j_\alpha} \bigg[
		\frac{1}{m}\sum_{i = 1}^m \sum_{j = 1}^{j_\beta} 
		\big( 1 \,-\, (w^\top x_i^+ \,-\, w^\top x_{(j)_w}^-)\big)_+\\
	&&  \hspace{3cm}
		~-~
		\frac{1}{m}\sum_{i = 1}^m \sum_{j = 1}^{j_\alpha} 
		\big( 1 \,-\, (w^\top x_i^+ \,-\, w^\top x_{(j)_w}^-)\big)_+
		\bigg]\\
	&=&
	\frac{1}{j_\beta-j_\alpha} \Big[
		j_\beta\widehat{R}^\tight_{\pAUC(0,\beta)}(w;\, S)		
				~-~
		j_\alpha\widehat{R}^\tight_{\pAUC(0,\alpha)}(w;\, S)		
		\Big],
\label{eqn:emp-pauc-struct-dc}
\end{eqnarray*}
where in the second step, we use Theorem \ref{thm:pauc-struct} to writte the hinge surrogate for partial AUC in $[0,\beta]$ in terms of the tight structural SVM formulation (see \Eqn{eqn:pauc-risk-struct-beta-2}).

\begin{figure}[t]
\begin{algorithm}[H]
\caption{Concave-Convex Procedure (CCCP) for $\text{SVM}^\dc_{\pAUC}$ in $[\alpha, \beta]$}
\label{algo:cccp}
\begin{algorithmic}[1]
\small{
\STATE \textbf{Inputs:} $S = (S_+, S_-), ~\alpha, ~\beta, ~C, ~\tau$ 
\STATE $f(w) = \frac{j_\beta}{j_\beta-j_\alpha} \widehat{R}^\tight_{\pAUC(0,\beta)}(w;\, S)$;\,\,\, $g(w) = \frac{j_\alpha}{j_\beta-j_\alpha} \widehat{R}^\tight_{\pAUC(0,\alpha)}(w;\, S)$
\STATE $H_w(Z, \pi) \,\equiv\, \Delta_{\AUC}(\pi^*,\pi) - w^\top \big( \phi((S_+, Z), \pi^*) - \phi((S_+, Z),\pi) \big)$
\STATE Initialize $w_0 \in \R^d$, \,\,$t = 0$
\STATE \textbf{Repeat} 
\STATE \hspace{0.5cm} $t = t+1$
\STATE \hspace{0.5cm} Compute supergradient for $-g(w)$:
\STATE \hspace{1cm} $(\bar{Z}, \bar{\pi}) \,=\, \underset{\substack{Z \subseteq S_-, \, |Z| = j_\alpha\\\pi  \,\in\,\Pi_{m,j_\alpha}}}
		{\operatorname{argmax}} H_w(Z, \pi);$
 \hspace{0.25cm} $v_t = -j_\alpha \phi((S_+, \bar{Z}), \bar{\pi})$
 		\vspace{0.1cm}
\STATE \hspace{0.5cm} Optimize convex upper bound:
\STATE \hspace{1cm} $w_t ~\in~ \underset{w \in \R^d}{\argmin} \,\, \frac{1}{2}||w||^2 \,+\, f(w) \,+\, w^\top v_t$
 		\vspace{0.1cm}
\STATE \textbf{Until} ~$(f(w_t) - g(w_t)) - (f(w_{t-1}) - g(w_{t-1})) \,\leq\, \tau $
\STATE \textbf{Output:} $w_{t}$ 
}
\end{algorithmic}
\end{algorithm}	
\end{figure}


\textbf{Concave-convex procedure.} The above difference-of-convex function can now be optimized directly using the well-known concave-convex procedure (CCCP) \cite{YuJoachims09,YuilleRang03}. This technique works by successively computing a gradient-based linear upper bound on the concave (or negative convex) part of the objective, and optimizing the resulting sum of convex and linear functions (see Algorithm \ref{algo:cccp} ). 

In our case, each iteration $t$ of this technique maintains a model vector $w_t$ and computes a supergradient of the concave term $-j_\alpha\widehat{R}^\tight_{\pAUC(0,\alpha)}(w;\, S)$ at $w_t$. Since this term is essentially the negative of the maximum of linear functions in $w$ (see \Eqn{eqn:pauc-risk-struct-beta-2}), one can obtain a supergradient of this term (w.r.t. $w$) by computing the gradient of the linear function at which the maximum is attained \cite{bertsekas99}; specifically, if $(\bar{Z}, \bar{\pi})$ is the subset-matrix pair at which the maximum is attained (which can be computed efficiently using Algorithm \ref{algo:mvc}), a supergradient of this term is $v_t = -j_\alpha \phi((S_+,\bar{Z}), \bar{\pi})$, with the corresponding linear upper bound given by $-j_\alpha \widehat{R}^\tight_{\pAUC(0,\alpha)}(w_t;\, S) \,+\, (w-w_t)^\top v_t$. This gives us a convex upper bound on our original difference-of-convex objective, which can be optimized efficiently using a straightforward variant of the structural SVM method discussed in Section \ref{sec:svm-beta}. The CCCP method then simply alternates between the above described linearization and convex upper bound optimization steps until the difference in objective value across two successive iterations falls below a tolerance $\tau > 0$.


The CCCP method is guaranteed to converge to only a locally optimal solution or to a saddle point \cite{YuilleRang03}, and moreover is computationally more expensive than the previous structural SVM approach (requiring to solve an entire structural SVM optimization in each iteration). However, as we shall see in our experiments in Section \ref{sec:expts}, in practice, this method yields higher partial AUC values than the structural SVM method in some cases. 


\section{Generalization Bound for Partial AUC}
\label{sec:genbound}
We have so far focused on developing efficient methods for optimizing the partial AUC. In this section, we look at the generalization properties of this evaluation measure. In particular, we derive a uniform convergence generalization bound for the partial AUC risk, thus establishing that good `training' performance in terms of partial AUC also implies good generalization performance. We first define the population or distribution version of the partial AUC risk for a general scoring function $f: X \> \R$:
\begin{eqnarray}
R_{\pAUC(\alpha,\beta)}[f; \D] \,=\, \frac{1}{\beta-\alpha}\E_{x^+ \sim \D_+, x^- \sim \D_-}\big[\1\big(f(x^+) \leq f(x^-)\big)\,\bT_{\alpha,\beta}(f, x^-)\big],
\label{eqn:gen-bound-pop-risk}
\end{eqnarray}
where $\bT_{\alpha,\beta}(f, x^-)$ is an indicator function which is 1 if $\P_{\tilde{x}^- \sim \D_-}(f(\tilde{x}^-) > f(x^-)) \in [\alpha, \beta]$ 
and is 0 otherwise. As before, we define the empirical partial AUC risk for a sample $S = (S_+, S_-)$:
\begin{eqnarray}
\widehat{R}_{\pAUC(\alpha,\beta)}[f; S] \,=\, \frac{1}{m(j_\beta - j_\alpha)} \sum_{i=1}^{m}\sum_{j=1}^{n}\1\big(f(x_i^+) \leq f(x_j^-)\big)\,\widehat{\bT}_{\alpha,\beta}(f, x^-_j),
\label{eqn:gen-bound-emp-risk}
\end{eqnarray}
where $\widehat{\bT}_{\alpha,\beta}(f, x^-_j)$ is an indicator function which is turned on only when $x^-_j$ lies in positions $j_\alpha+1$ to $j_\beta$ in the ranking of all negative instances by $f$.

We would like to show that the above empirical risk is not too far from the population risk for any scoring function chosen from a given real-valued function class $\F$ of reasonably bounded `capacity'. In our case, the capacity of such a function class $\F$ will
be measured using the VC dimension of the class of thresholded classifiers obtained from scoring functions in the class: $\T_\F = \big\{\sign \circ (f - t) \,|\, f \in \F, \, t \in \R\big\}$. We have the following uniform convergence bound for partial AUC:
\begin{theorem}
\label{thm:pauc-gen-bound}
Let $\F$ be a class of real-valued functions on $X$, and $\T_\F = \big\{\sign \circ (f - t) \,|\, f \in \F, \, t \in \R\big\}$. Let $\delta > 0$. Then with probability at least $1 - \delta$ (over draw of sample $S = (S_+, S_-)$ from $\D_+^m \times \D_-^n$), we have for all $f \in \F$,
\begin{eqnarray*}
\lefteqn{
{R}_{\pAUC(\alpha,\beta)}[f; \D]
 \,\, \leq \,\, \widehat{R}_{\pAUC(\alpha,\beta)}[f; S] 
}\\[2pt]
&& 
\hspace{3cm}
 \,+\, C\bigg(\sqrt{\frac{d\ln(m) + \ln(1/\delta)}{m}} + \frac{1}{\beta-\alpha}\sqrt{\frac{d\ln(n) + \ln(1/\delta)}{n}}\bigg),
\end{eqnarray*}
where $d$ is the VC dimension of $\T_\F$, and $C > 0$ is a distribution-independent constant.
\end{theorem}
The above result provides a bound on the generalization performance of a learned scoring function in terms of its empirical (training) risk. Also notice that the tightness of this bound depends on the size of the FPR range of interest. In particular, the smaller the FPR interval, looser is the bound. 

The proof of this result differs substantially from that for the full AUC \cite{Agarwal+05} as the complex structure of partial AUC forbids the direct application of standard concentration results like Hoeffding's inequality. Instead the difference between the empirical and population risks needs to be broken down into simpler additive terms that can in turn be bounded using standard arguments. We provide the details in Appendix \ref{app:pauc-gen-bound-proof}.
\footnote{
We would like to also note that in a follow-up work, Kar et al. \cite{Kar+14} provide a uniform convergence bound for the hinge loss based partial AUC surrogate in \Eqn{eqn:emp-pauc-hinge-general} that makes use of a covering number based analysis. Our VC-dimension based bound, on the other hand, applies directly to the discrete partial AUC risk (in \Eqn{eqn:emp-pauc-risk}), rather than to a surrogate on this quantity.} 


\section{Experiments}
\label{sec:expts}	

\begin{figure*}
\centering
\begin{table}[H]
\centering
\small
\begin{tabular}{|c|c|c|}
\hline
\textbf{Data set}		&\textbf{\#instances}	& \textbf{\#features}	
\\\hline
ppi			& 240,249		& 85		
\\
chemo		& 2,
142			& 1,021		
\\
kddcup08	& 102,294		& 117		
\\
kddcup06	& 4,429			& 116
\\\hline
\end{tabular}
\caption{Real data sets used.}
\label{tab:datasets-real}
\end{table}
\end{figure*}

In this section, we present experimental evaluations of the proposed SVM based methods for optimizing partial AUC on a number of real-world applications where the partial AUC is a performance measure of interest, and on benchmark UCI data sets. 
The structural SVM algorithms were implemented using a publicly available API from \cite{structsvm}\footnote{\url{http://svmlight.joachims.org/svm_struct.html}}, while the DC programming method was implemented using an API for latent structural SVM from \cite{YuJoachims09}\footnote{\url{http://www.cs.cornell.edu/~cnyu/latentssvm/}}. In each case, two-thirds of the data set was used for training and the remaining for testing, with the results averaged over five such random splits. The tunable parameters were chosen using a held-out portion of the training set treated as a validation set. The specific parameter choices, along with details of data preprocessing, can be found in Appendix \ref{app:expts}.  All experiments used a linear scoring function.\footnote{All the methods were implemented in C. Code for our methods will be made available at: \url{http://clweb.csa.iisc.ernet.in/harikrishna/Papers/SVMpAUC-tight/}.}


\subsection{Maximizing Partial AUC in $[0,\beta]$}

We begin with our results for FPR intervals of the form $[0,\beta]$. We considered two real-world applications where the partial AUC in $[0,\beta]$ is an evaluation measure of interest, namely, a protein-protein interaction prediction task and a drug discovery task (see Table \ref{tab:datasets-real}). We refer to the proposed structural SVM based method in Section \ref{sec:svm-beta} as $\SVM_\pAUC$. We included for comparison the structural SVM algorithm of Joachims (2005) for optimizing the full AUC which we shall call $\SVM_\AUC$ \cite{svmperf}, as well as three existing algorithms for optimizing partial AUC in $[0,\beta]$: asymmetric SVM (ASVM) \cite{asvm}, pAUCBoost \cite{paucboost}, and a greedy heuristic method due to \cite{paucmax}. 

%
\begin{figure}[t]
\centering
\subfigure[$\beta = 0.05$]{
	\includegraphics[scale=0.57,trim={0 0 6cm 0},clip]{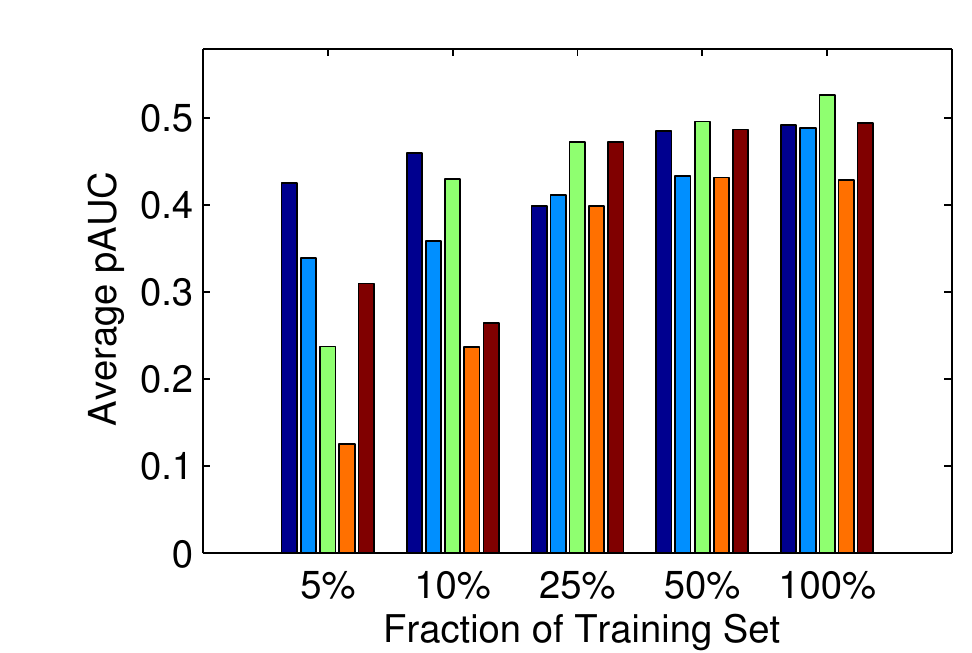}
}
\hspace{-0.35cm}
\subfigure[$\beta = 0.1$]{
	\includegraphics[scale=0.57,trim={0 0 6cm 0},clip]{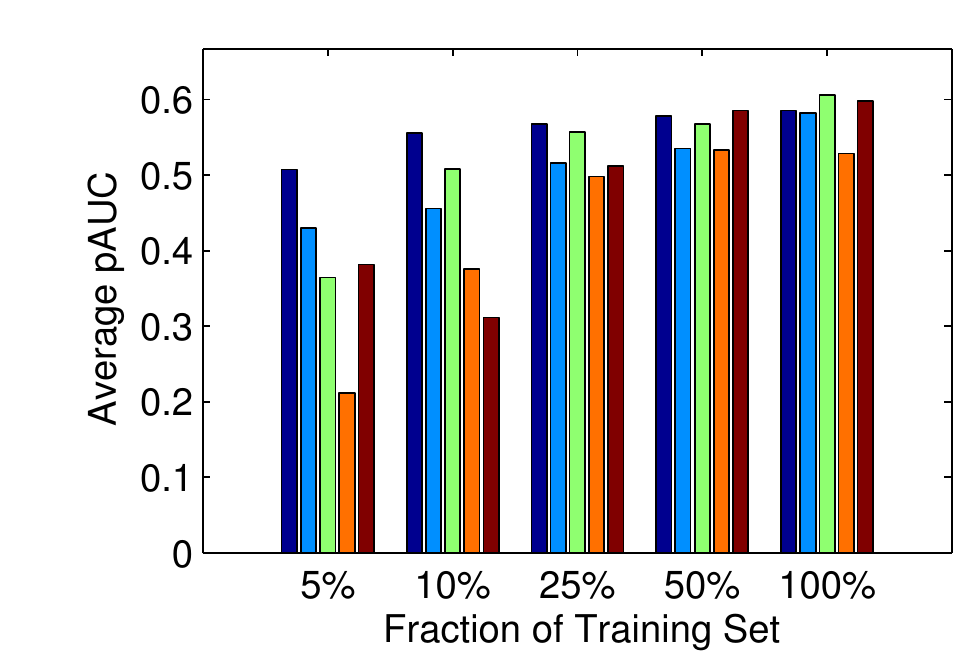}
}
\subfigure{
	\includegraphics[scale=0.5,trim={10cm 0 0 0},clip]{ppi-2-eps-converted-to.pdf}
}
\caption{Partial AUC maximization in $[0, \beta]$ on PPI data.}
\label{fig:ppi}
\end{figure}

~\\
\noindent \textbf{Protein-protein interaction prediction.} 
In protein-protein interaction (PPI) prediction, the task is to predict whether a given pair of proteins interact or not. Owing to the highly imbalanced nature of PPI data (e.g.\ only 1 in every 600 protein pairs in yeast are found to interact), the partial AUC in a small FPR range $[0,\beta]$ has been advocated as an evaluation measure for this application \cite{ppiQi06}. We used the PPI data for Yeast from \cite{ppiQi06}, which contains 2,865 protein pairs known to be interacting (positive) and a random set of 237,384 protein pairs assumed to be non-interacting (negative); each protein-pair is represented using 85 features.\footnote{In the original data set in \cite{ppiQi06}, each protein pair is represented by 162 features, but there were several missing features; we used a subset of 85 features that contained less than 25\% missing values (with missing feature values replaced by mean/mode values).} 
We evaluated the partial AUC performance of these methods on two FPR intervals $[0, 0.05]$ and $[0, 0.1]$. To compare the methods for different training sample sizes, we report results in Figure \ref{fig:ppi} for varying fractions of the training set. As seen, the proposed method almost always yields higher partial AUC in the specified FPR intervals compared to the $\SVM_\AUC$ method for optimizing the full AUC, thus confirming its focus on a select portion of the ROC curve. Interestingly, the difference in performance is more pronounced for smaller training sample sizes, implying that when one has limited training data, it is more beneficial to use the data to directly optimize the partial AUC rather than to optimize the full AUC. Also, in most cases, the proposed method performs comparable to or better than the other baselines; the pAUCBoost and Greedy-Heuristic methods perform particularly poorly on smaller training samples due to the use of heuristics.
~\\

\begin{figure}[t]
\centering
\subfigure[$\beta = 0.05$]{
	\includegraphics[scale=0.57,trim={0 0 6cm 0},clip]{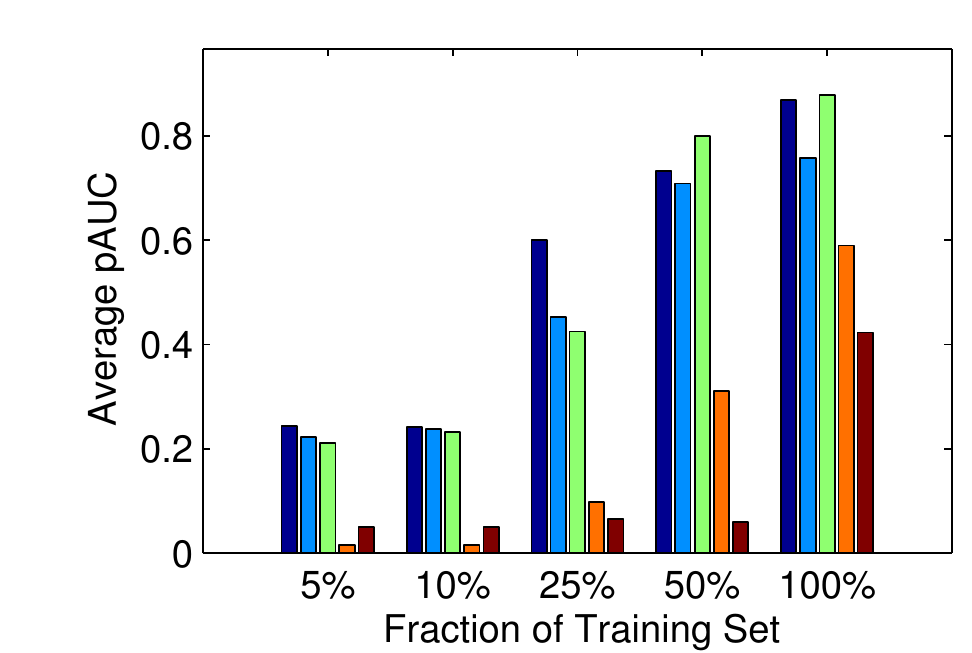}
}
\hspace{-0.35cm}
\subfigure[$\beta = 0.1$]{
	\includegraphics[scale=0.57,trim={0 0 6cm 0},clip]{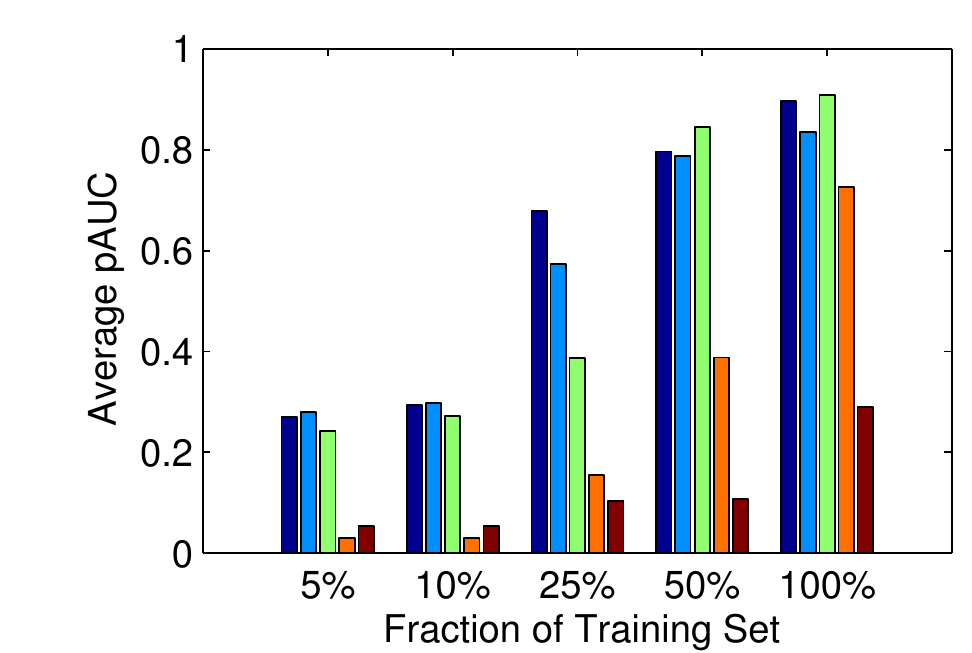}
}
\subfigure{
	\includegraphics[scale=0.5,trim={10cm 0 0 0},clip]{chemo-2-eps-converted-to.pdf}
}
\caption{Partial AUC maximization in $[0, \beta]$ on drug discovery data.}
\label{fig:chemo}
\end{figure}

\noindent \textbf{Drug discovery.} 
In the next task that we considered, one is given examples of chemical compounds that are active or inactive against a therapeutic target, and the goal is to rank new compounds such that the active ones are above the inactive ones. Here one is often interested in good ranking quality in the top portion of the ranked list, and hence good partial AUC in a small FPR interval $[0, \beta]$ in the initial portion of the ROC curve is a performance measure of interest. In our experiments, we used a virtual screening data set from \cite{JorissenGi05}; this contains 50 active/positive compounds (corresponding to the reversible antagonists of the
$\alpha_{\text{1A}}$ adrenoceptor) and 2092 ones that are inactive/negative, with each compound represented as a 1021-bit vector using the FP2 molecular fingerprint representation (as done in \cite{drug}). 
Figure \ref{fig:chemo} contains the partial AUC performance for varying fractions of the training set on two FPR intervals. Clearly, for the most part, $\text{SVM}_\text{pAUC}$ yields higher partial AUC values than $\text{SVM}_\text{AUC}$, and
performs comparable to or better than the other baseline algorithms.

\subsection{Maximizing Partial AUC in $[\alpha,\beta]$}
We next move to our experiments on partial AUC in a general $[\alpha,\beta]$ interval. We refer to the proposed structural SVM method for maximizing partial AUC in $[\alpha,\beta]$ again as $\SVM_\pAUC$ and our DC programming approach for optimizing the non-convex hinge surrogate as $\SVM^\dc_\pAUC$.  As baselines, we included $\SVM_\AUC$, pAUCBoost which can optimize partial AUC over FPR ranges $[\alpha,\beta]$, and an extension of the greedy heuristic method in \cite{paucmax} to handle arbitrary FPR ranges. We first present our results on a real-world application, where partial AUC in $[\alpha,\beta]$ is a useful evaluation measure.

\begin{table}[t]
\centering
\small
\begin{tabular}{|l|c|}
	\hline
								&	\multicolumn{1}{c|}{$\text{{pAUC}}(0.2s, 0.3s)$}
	\\\hline
	$\text{SVM}_{\text{pAUC}}[0.2s,0.3s]$	& 	\textbf{0.6376}		
	\\
	$\text{SVM}^\text{dc}_{\text{pAUC}}[0.2s,0.3s]$	& 	{0.6162}		
	\\
	$\text{SVM}_{\text{AUC}}$ 			&	0.6117	   
	\\
	pAUCBoost$[0.2s, 0.3s]$					& 	0.6033	 	
	\\
	Greedy-Heuristic$[0.2s, 0.3s]$				& 	0.5616	  	
	\\\hline
\end{tabular}	
\caption[Partial AUC maximization on KDD Cup 08 data]{Partial AUC maximization in $[\alpha,\beta]$ with KDD Cup 08 data. Here $s = 6848/101671$. }
\label{tab:kddcup08}
\end{table}

~\\
\textbf{Breast cancer detection.}
We consider the task stated in the KDD Cup 2008 challenge, where one is required to predict whether a given region of interest (ROI) from a breast X-ray image corresponds to a malignant (positive) or a benign (negative) tumor \cite{kddcup}. The data provided is collected from 118 malignant patients and 1,594 normal patients. Four X-ray images are available for each patient; overall, there are 102,294 candidate ROIs selected from these X-ray images, of which 623 are positive, with each ROI represented by 117 features. In the KDD Cup challenge, performance was evaluated in terms of the partial area under the \emph{free-response} operating characteristic (FROC) curve in a false positive range $[0.2,0.3]$ deemed clinically relevant based on radiologist surveys. The FROC curve \cite{Miller69} effectively uses a scaled version of the false positive rate in the usual ROC curve; for our purposes, the corresponding false positive rate is obtained by re-scaling by a factor of $s=6848/101671$ (this is the total number of images divided by the total number of negative ROIs). Thus, the goal in our experiments was to maximize the partial AUC in the clinically relevant FPR range $[0.2s,0.3s]$. \Tab{tab:kddcup08} presents results on algorithms $\SVM_\pAUC$ and $\SVM^\dc_\pAUC$ developed for FPR intervals of this form, as well as on the baseline methods; $\SVM_\pAUC$ performs the best in this case.

\begin{figure}
\centering
\begin{table}[H]
\small
\centering
\begin{tabular}{|c|c|c|}
\hline
\textbf{Data set}		&\textbf{\#instances}	& \textbf{\#features}	
\\\hline
a9a			& 48,842		& 123		
\\
cod-rna		& 488,565		& 8			
\\
covtype	& 581,012		& 54		
\\
ijcnn1		& 141,691		& 22		
\\
letter		& 20,000		& 16		
\\\hline
\end{tabular}
\caption{UCI data sets used.}
\label{tab:datasets-uci}
\end{table}
\end{figure}

\begin{table}[t]
\small
\centering
	\begin{tabular}{|l|c|c|c|c|c|}
	\hline
								&	\multicolumn{5}{c|}{$\text{{pAUC}}(0.02, 0.05)$}
	\\\hline
											& {a9a}	& {cod-rna}	& {covtype} & {ijcnn1} & {letter}
	\\\hline
	$\text{SVM}_{\text{pAUC}}$\,[0.02, 0.05] 	& 	0.2739	& {0.9187}		&	{0.2467}	& 	 {0.6131}	& \textbf{0.5208}
	\\
	$\text{SVM}^\text{dc}_{\text{pAUC}}$\,[0.02, 0.05] 	& 	0.3650	& \textbf{0.9196}		&	0.2410	& 	 \textbf{0.6798}	& {0.5182}
	\\
	$\text{SVM}_{\text{AUC}}$ 				& 	\textbf{0.4338}	& {0.9192}	& 	{0.2987}	& 	{0.4750}	& 0.4455  %
	\\
	pAUCBooost\,[0.02, 0.05] 					& {0.4012}	&   0.0330		&	\textbf{0.4485}	&	0.4913		& {0.4954}	
	\\
	GreedyHeuristic[0.02, 0.05] 				& 0.3417	&  0.0329 		&	0.2386	&	0.1201 		& 0.2888	
	\\\hline 
\end{tabular}	
\caption[{Partial AUC maximization in $[0.02, 0.05]$ on UCI data sets}]{Partial AUC maximization in [0.02, 0.05] on UCI data sets. }
\label{tab:0.02-0.05}
\end{table}

~\\
\noindent \textbf{UCI data sets.} To perform a more detailed comparison between the proposed structural SVM and DC programming methods for general $[\alpha,\beta]$ intervals, we also evaluated the methods on a number of benchmark data sets obtained from the UCI machine learning repository \cite{uci} (see Table \ref{tab:datasets-uci}). The results for the FPR interval $[0.02, 0.05]$ are shown in Table \ref{tab:0.02-0.05}; for completeness, we also report the performance of the baseline methods.  Despite having to solve a non-convex optimization problem (and hence running the risk of getting stuck at a locally optimum solution), $\SVM^\dc_\pAUC$ does perform better than $\SVM_\pAUC$ in some cases, though between the two, there is no clear winner. Also, on three of the five data sets, one of the two proposed methods yield the best overall performance.
The strikingly poor performance of pAUCBoost and Greedy-Heuristic on the cod-rna data set is because the features in this data set individually produce low partial AUC values; since these methods rely on local greedy heuristics, they fail to find a linear combination of the features that yields high partial AUC performance.

\subsection{Maximizing TPR at a Specific FPR Value}
We have so far seen that the proposed methods are good at learning scoring functions that yield high partial AUC in a specified FPR range. In our next experiment, we shall demonstrate that the proposed methods can also be useful in applications where one is required to learn a classifier with specific true/false positive requirements. In particular, we consider the task described in the KDD Cup 2006 challenge of detecting pulmonary emboli in medical images obtained from CT angiography \cite{kddcup06}.  Given a candidate region of interest (ROI) from the image, the goal is to predict whether it is a pulmonary emboli or not; a specific requirement here is that the classifier must have high TPR, with the FPR kept within a specified limit. 

Indeed if a classifier is constructed by thresholding a scoring function, the above evaluation measure can be seen as the partial AUC of the scoring function in an infinitesimally small FPR interval. However, given the small size of the FPR interval concerned, maximizing this evaluation measure directly may not produce a classifier with good generalization performance, particularly with smaller training samples (see generalization bound in Section \ref{sec:genbound}). Instead, we prescribe a more robust approach of using the methods developed in this paper to maximize partial AUC in an appropriate larger FPR interval, and constructing a classifier by suitably thresholding the scoring function thus obtained. 

\begin{table}[t]
\centering
\small
\begin{tabular}{|l|c|}
	\hline
								&	TPR at FPR $= 0.1$
	\\\hline
	$\text{SVM}_{\text{pAUC}}[0, 0.1]$ 	& \textbf{0.6131}
	\\
	$\text{SVM}_{\text{pAUC}}[0.05, 0.1]$ 	& {0.5839}
	\\	
	$\text{SVM}^\text{dc}_{\text{pAUC}}[0.05, 0.1]$ 	& {0.5912}
	\\
	$\text{SVM}_{\text{AUC}}$ 			& {0.5839}
	\\\hline
\end{tabular}	
\caption[Partial AUC maximization with KDD Cup 06 data]{Partial AUC maximization as a proxy for maximizing the TPR at a specified FPR on KDD Cup 06 data. }
\label{tab:kddcup06}
\end{table}

\begin{figure}[th!]
\centering
\vspace{-10pt}
\hspace{-12pt}
\subfigure[ppi]{
\centering
	\includegraphics[scale=0.6]{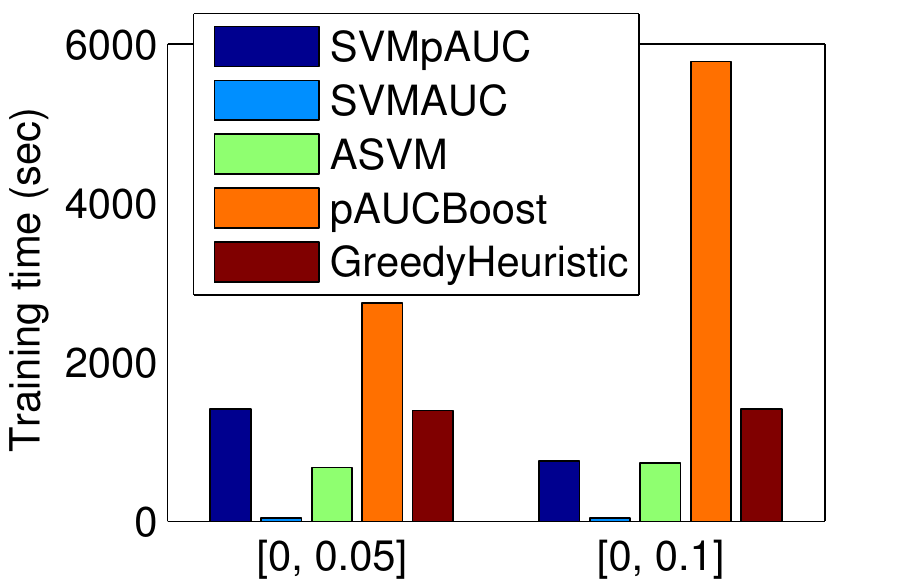}
	\hspace{-10pt}
	\includegraphics[scale=0.57]{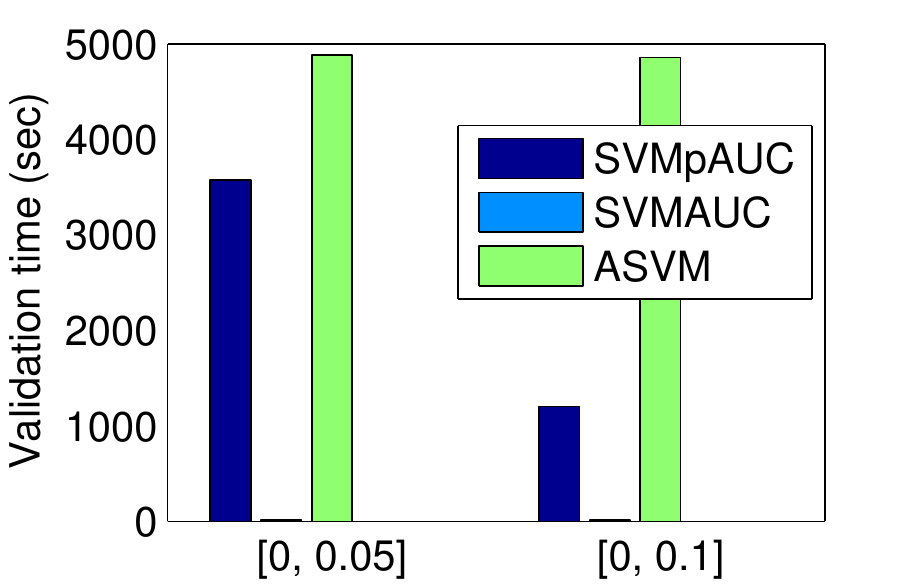}
}
\hspace{-2pt}
\subfigure[ijcnn1]{
\centering
	\includegraphics[scale=0.57]{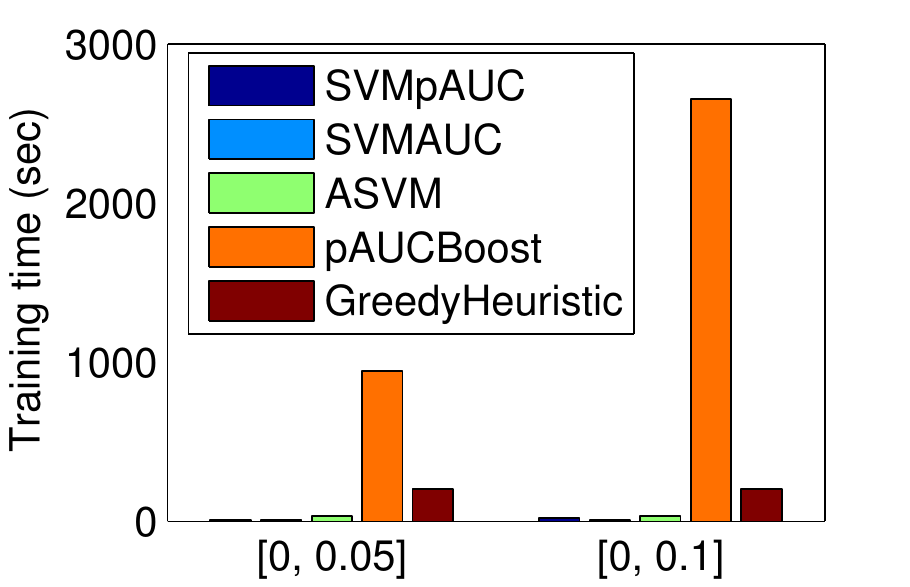}
	\hspace{-10pt}
	\includegraphics[scale=0.57]{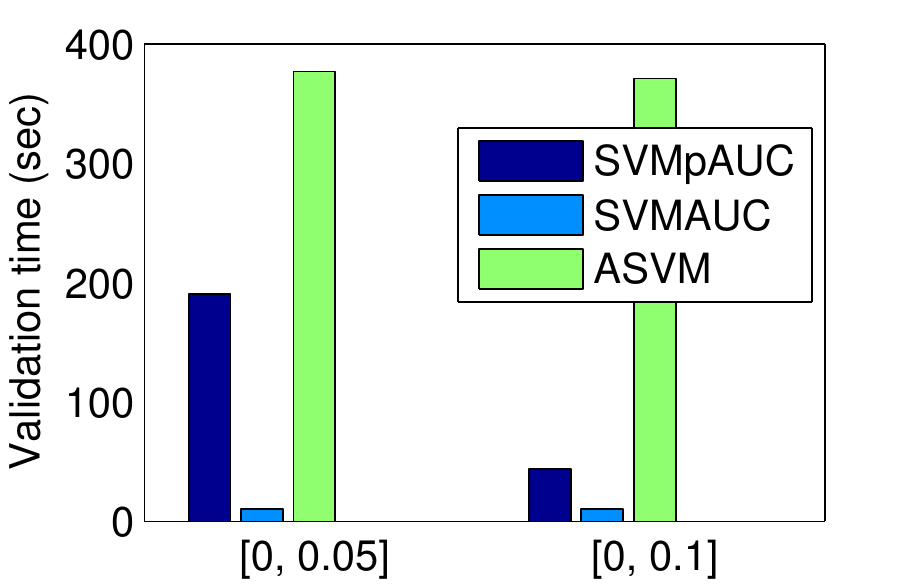}
}
\caption{Partial AUC maximization in $[0, \beta]$: Comparison of average training and validation times between $\SVM_\pAUC$ and baseline methods.}
\label{fig:time-beta}
\subfigure[kddcup08]{
	\includegraphics[scale=0.57]{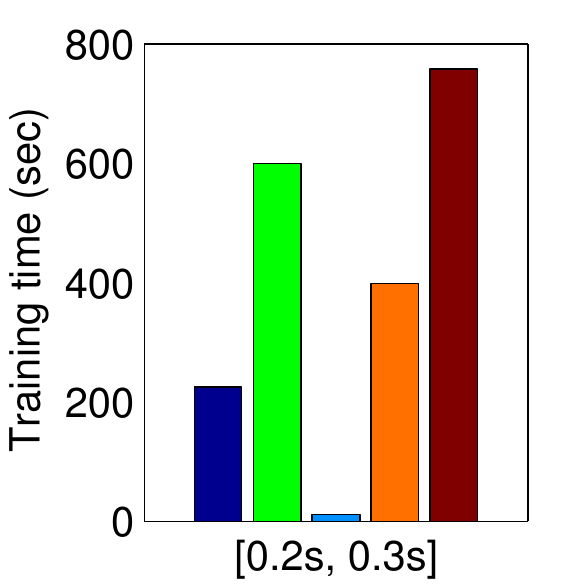}
	\hspace{-10pt}
	\includegraphics[scale=0.57]{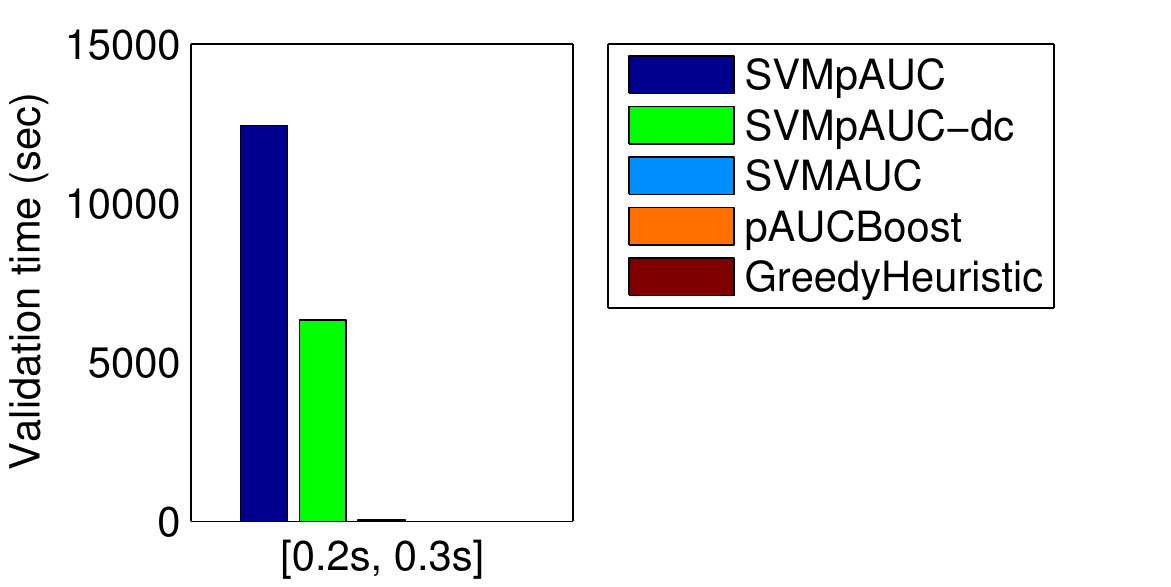}
}
\subfigure[ijcnn1]{
	\includegraphics[scale=0.57]{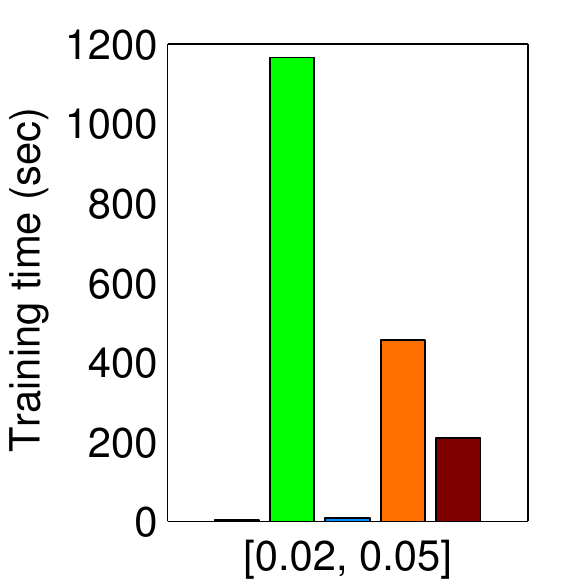}
	\hspace{-10pt}
	\includegraphics[scale=0.57]{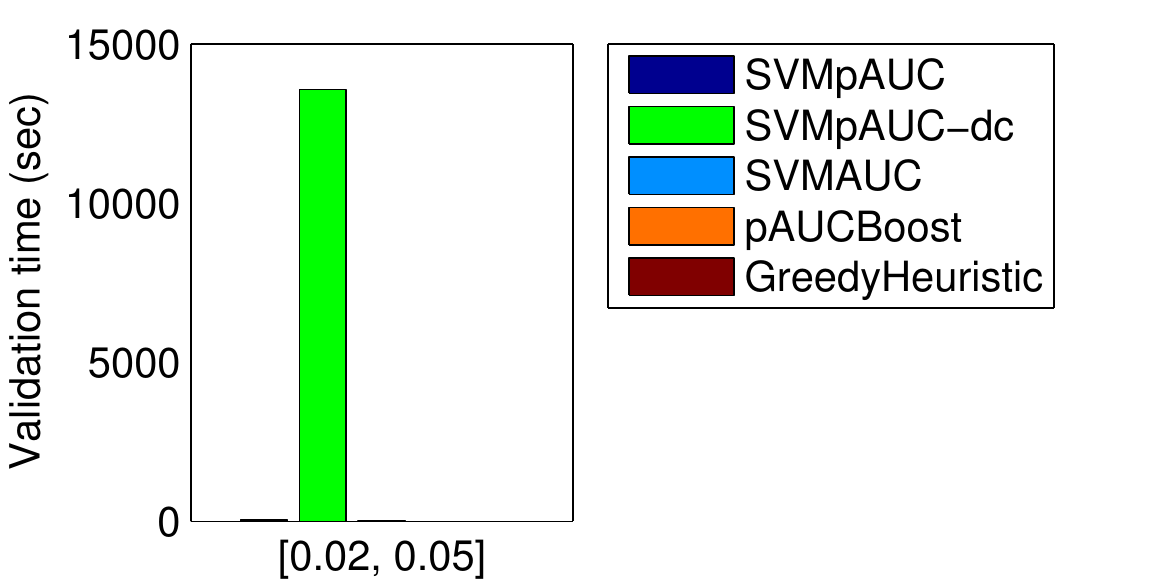}
}
\caption{Partial AUC maximization in $[\alpha, \beta]$: Comparison of average training and validation times between $\SVM_\pAUC$, $\SVM^\dc_\pAUC$ and baseline methods. Here $s = 6848/101671$.}
\label{fig:time-general}
\end{figure}

The data provided in the KDD Cup challenge consists of 4,429 ROIs represented with 116 features, of which 500 are positive. We considered a maximum allowable FPR limit of $0.1$ (which is one of the values prescribed in the challenge). The proposed partial AUC maximization methods were used to learn scoring functions for two FPR intervals, $[0, 0.1]$ and $[0.05, 0.1]$, that we expected will promote high TPR at the given FPR of 0.1; the performance of a learned model was then evaluated based on the TPR it yields when thresholded at a FPR of 0.1. Table \ref{tab:kddcup06} contains results for $\SVM_\pAUC$ on both intervals, and for $\SVM^\dc_\pAUC$ on the $[0.05, 0.1]$ interval. We also included $\SVM_\AUC$ for comparison. Interestingly, in this case, $\SVM_\pAUC$ in $[0, 0.1]$ performs the best; the DC programming method on the $[0.05, 0.1]$ interval comes a close second, performing better than the structural SVM approach for the same interval.\footnote{The experiments were carried out on a single 70\%-30\% train-test split provided in the challenge.}


We also note that in a follow-up work, the proposed approach was applied to a similar problem in personalized cancer treatment \cite{Majumder+15}, where the goal was to predict whether a given cancer patient will respond well to a drug treatment. In this application, one again requires high true positive rates (fraction of cases where the treatment is effective and the model predicts the same), subject to the false positive rate being within an allowable limit; as above, the problem was posed as a partial AUC maximization task, and the classifiers learned using our methods were found to yield higher TPR performance than standard approaches for this problem, while not exceeding the allowed FPR limit.

\subsection{Run-time Analysis}

\begin{figure}[t]
\vspace{-10pt}
\centering
\subfigure[ppi]{
\includegraphics[scale=0.64]{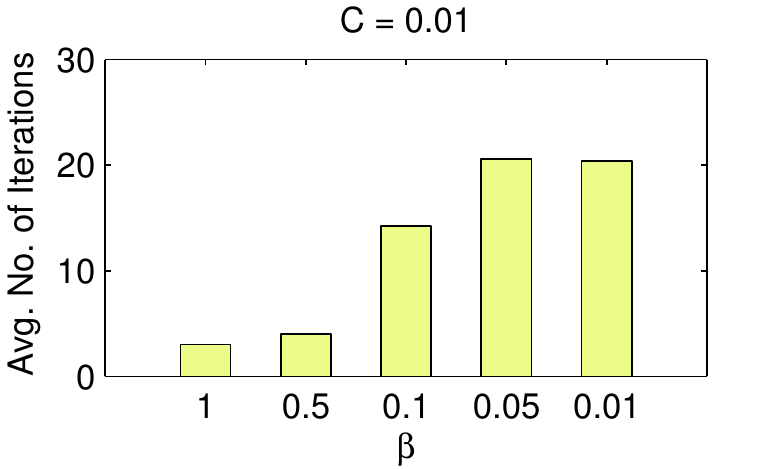}
\includegraphics[scale=0.64]{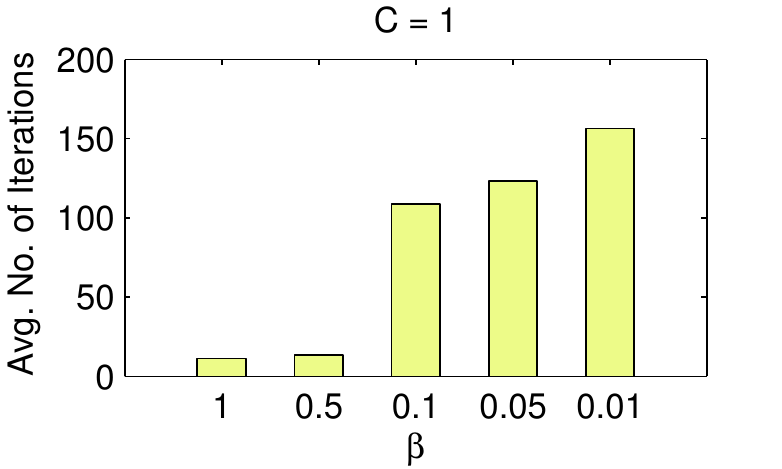}
\includegraphics[scale=0.64]{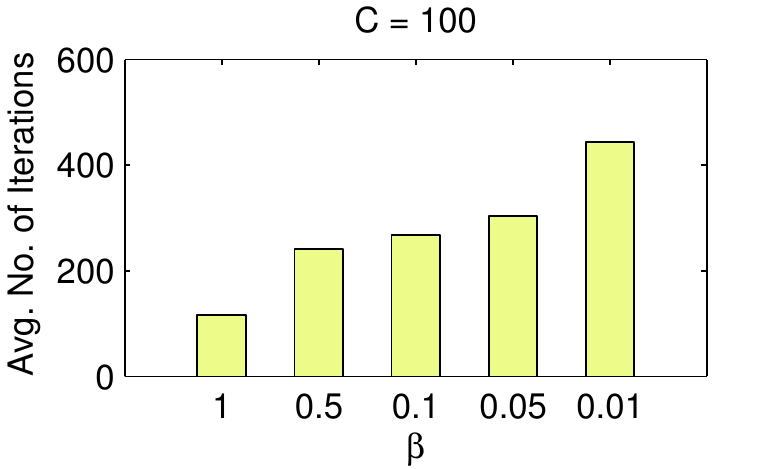}
}
\subfigure[covtype]{
\includegraphics[scale=0.64]{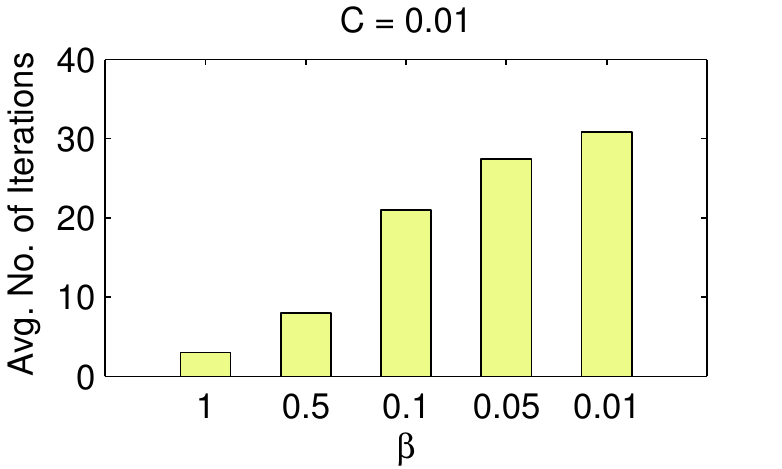}
\includegraphics[scale=0.64]{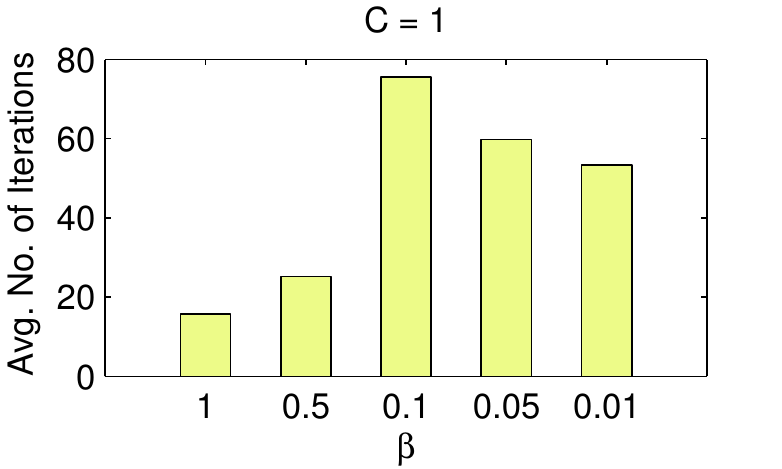}
\includegraphics[scale=0.64]{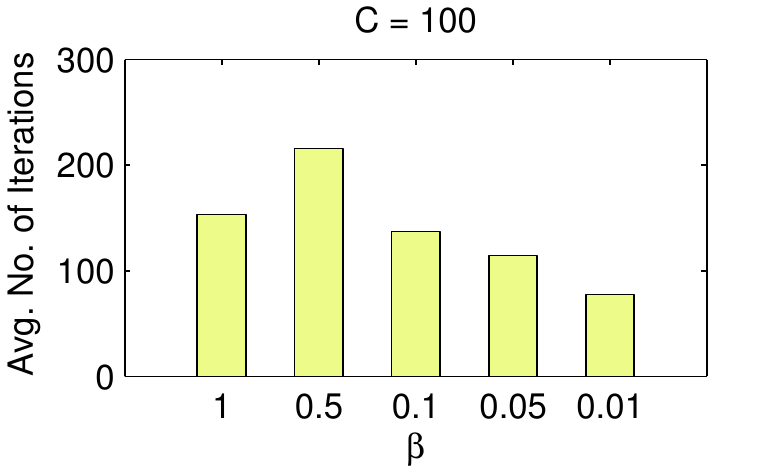}
}
\subfigure[ijcnn1]{
\includegraphics[scale=0.64]{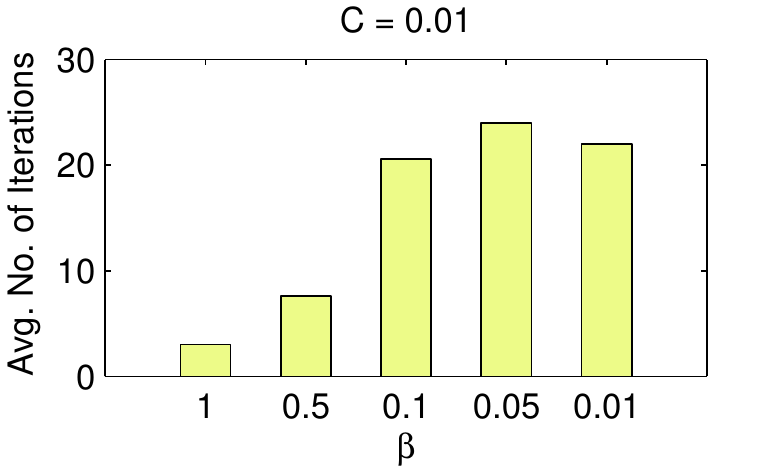}
\includegraphics[scale=0.64]{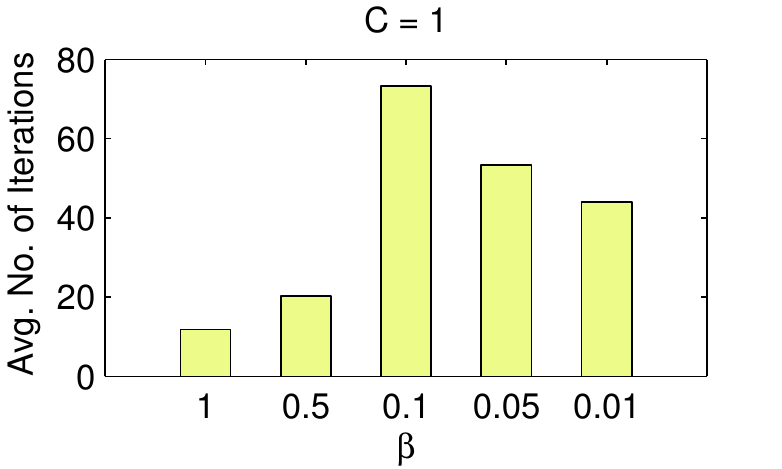}
\includegraphics[scale=0.64]{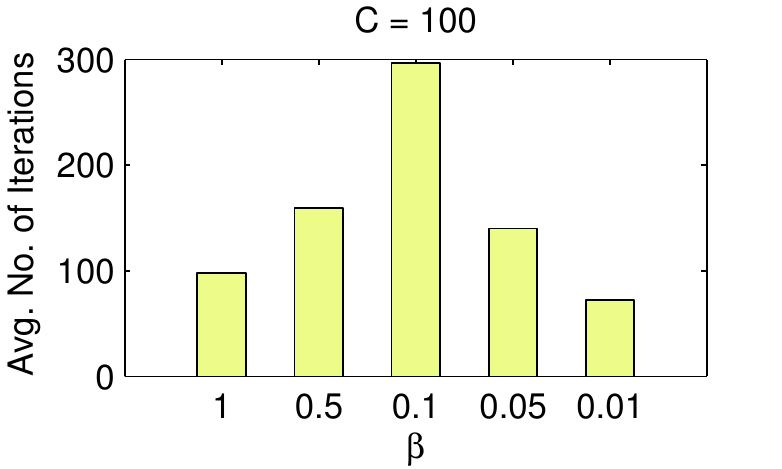}
}
\caption[{Partial AUC in $[0, \beta]$: Average number of cutting plane iterations in $\SVM_\pAUC$ vs. length of FPR interval for different values of $C$}]{Partial AUC in $[0, \beta]$: Average number of cutting plane iterations in $\SVM_\pAUC$ vs. length of FPR interval for different values of regularization parameter $C$. Here $\beta = 1$ corresponds to the full AUC optimizing method $\SVM_\AUC$.}
\label{fig:cutting-plane-beta}
\end{figure}
\begin{figure}[t]
\centering
\subfigure{
\includegraphics[scale=0.64]{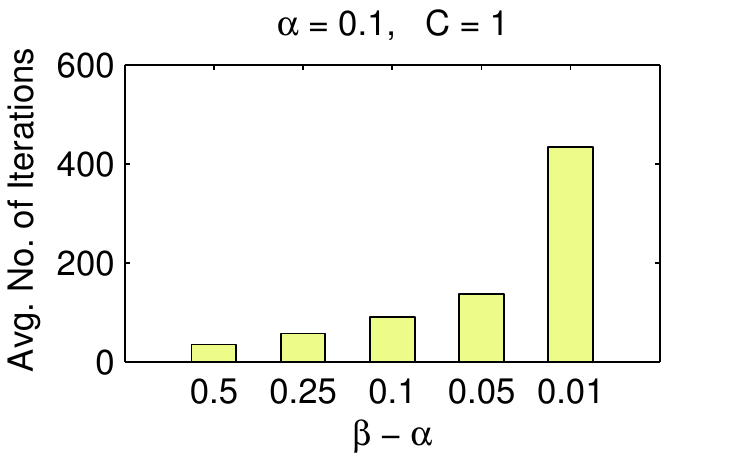}
\includegraphics[scale=0.64]{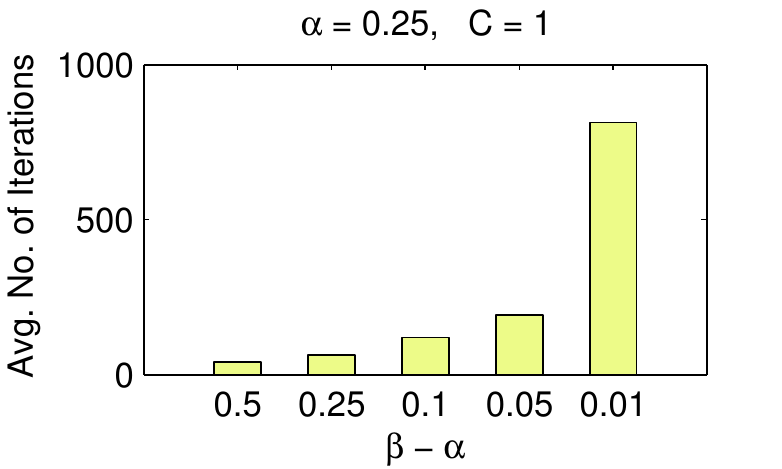}
\includegraphics[scale=0.64]{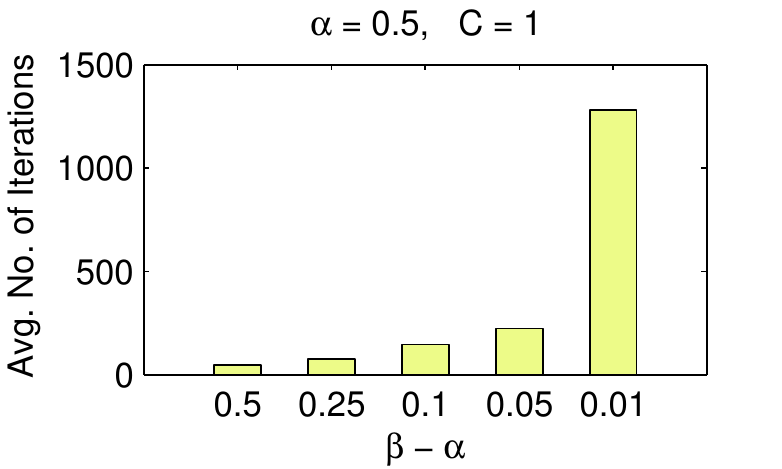}
}	
\caption[{Partial AUC in $[\alpha, \beta]$ on KDD Cup 08 data: Average number of cutting plane iterations  in $\SVM_\pAUC$ vs. length of FPR interval for different values of $\alpha$}]{Partial AUC in $[\alpha, \beta]$ on KDD Cup 08 data: Average number of cutting plane iterations  in $\SVM_\pAUC$ vs. length of FPR interval for different values of $\alpha$.}
\label{fig:cutting-plane-general}
\end{figure}

%
\noindent \textbf{Run-time comparison with baseline methods.} In our final set of experiments, we compared the running times of the various algorithms evaluated above. Figure \ref{fig:time-beta} contains the average training times (across five train-test splits) for partial AUC maximization tasks involving $[0, \beta]$ FPR intervals on two data sets. We also report the time take for validation/parameter tuning in $\SVM_\pAUC$, $\SVM_\AUC$ and ASVM; the remaining two methods do not have tunable parameters. The running times for partial AUC maximization in a general $[\alpha, \beta]$ range is shown in Figure \ref{fig:time-general} for two data sets. All experiments were run on an Intel Xeon (2.13 GHz) machine with 12 GB RAM. 

Notice that except for $\SVM_\AUC$, all the baselines require higher or similar training times compared to $\SVM_\pAUC$. Also, as expected, for the $[\alpha,\beta]$ intervals, the DC programming based $\SVM^\dc_\pAUC$ (which solves an entire structural SVM problem in each iteration) requires higher training time than $\SVM_\pAUC$. The reason for the full AUC maximizing $\SVM_\AUC$ method being the fastest in all cases, despite the subroutine for finding the most-violated constraint requiring higher computational time compared to that for partial AUC, is because the number of iterations required by the cutting plane solver is lower for AUC. This will become clear in our next experiments where we report the number of iterations taken by the cutting plane solver under different settings.
~\\

\noindent \textbf{Influence of $\alpha$, $\beta$ and $C$ on number of cutting plane iterations.}
We next analysed the number of iterations taken by the cutting plane method in $\SVM_\AUC$ and $\SVM_\pAUC$, i.e.\ the number of calls made to the routine for finding the most violated constraint (Algorithms \ref{algo:mvc} and \ref{algo:mvc-general}), for different FPR intervals and regularization parameter values. Our results for FPR ranges of the form $[0, \beta]$ are shown in Figure \ref{fig:cutting-plane-beta}, where we provide plots of the average number of cutting plane iterations (over five train-test splits) as a function of $\beta$ for different values of the regularization parameter $C$. It is often seen that the number of iterations of the cutting plane method increases as $\beta$ decreases (i.e.\ as the FPR interval becomes smaller), suggesting that optimizing partial AUC is harder for smaller intervals; this explains why the overall training time is lower for maximizing the full AUC (i.e.\ for $\SVM_\AUC$) compared to the time taken for maximizing partial AUC in a small interval. Similarly, the number of cutting plane iterations increases with $C$, as suggested by the method's convergence rate (see Section \ref{sec:svm-beta}). 

Our results for general FPR intervals $[\alpha, \beta]$ are shown in Figure \ref{fig:cutting-plane-general}, where we have plots of the average number of cutting plane iterations as a function of the interval length, for different values of $\alpha$. Again, it is seen that the number of iterations is higher for smaller FPR intervals; interestingly, the number of iterations also increases as the interval is farther away to the right.


\section{Conclusion and Open Questions}
\label{sec:conclusion}
The partial AUC is increasingly used as a performance measure in several machine learning and data mining applications. We have developed support vector algorithms for optimizing the partial AUC between two given false positive rates $\alpha$ and $\beta$. Unlike the full AUC, where it is straightforward to develop surrogate optimizing methods, even constructing a (tight) convex surrogate for the partial AUC turns out to be non-trivial. By exploiting the specific structure of the evaluation measure and extending the structural SVM framework of \cite{svmperf}, we have constructed convex surrogates for the partial AUC, and developed an efficient cutting plane method for solving the resulting optimization problem. In addition, we have also provided a DC programming method for optimizing a non-convex hinge surrogate that is tighter for general false positive ranges $[\alpha, \beta]$. Our empirical evaluations on several real-world and benchmark tasks indicate that the proposed methods do indeed optimize the partial AUC in the desired false positive range, often performing comparable to or better than existing baseline techniques.

Subsequent to the conference versions of this paper, there have been a number of follow-up works. As noted earlier, one such work applied our algorithm to an important problem in personalized cancer treatment \cite{Majumder+15}, where the task was to predict clinical responses of patients to chemotherapy. Other works include minibatch extensions of our structural SVM method for the $[0, \beta]$ range to online and large scale stochastic settings \cite{Kar+14}, as well as an ensemble style version of this method with application to a problem in computer vision \cite{Paisitkriangkrai+13,Paisitkriangkrai+14}

There are several questions that remain open. Firstly, we observed in our experiments that the number of iterations required by the proposed cutting plane solvers to converge depends on the length of the specified FPR range, but this is not evident from the current convergence rate for the solver (see Section \ref{sec:svm-beta}) obtained from a result in \cite{lintime}. It would be of interest to see if tighter convergence rates that match our empirical observation can be shown for the cutting plane solver. Secondly, it would be useful to understand the consistency properties of the proposed algorithms. Such studies exist for methods that optimize the full AUC and more recently, for other complex performance measures for classification such as the F-measure \cite{Koyejo+14, Narasimhan+14, Narasimhan+15}, but these results do not extend directly to the methods in this paper. Finally, one could look at extensions of the proposed algorithms to  multiclass classification and ordinal regression settings, where often there are different constraints on the error rates of a predictor on different classes. Again, there has been work on optimizing multiclass versions of the full AUC \cite{Waegeman+08, Clemencon+13, UematsuLee13}, and similar methods for multiclass variants of the partial AUC will be useful to have.
 


\subsection*{Acknowledgments}
HN thanks Prateek Jain and Purushottam Kar for helpful discussions. HN also thanks Purushottam Kar for providing preprocessed  versions of some of the data sets and a wrapper code for the experiments, and for an email correspondence that led us to identify an error in the version of \Thm{thm:surrogate-characterization} that appeared in a conference version of this paper \cite{NarasimhanAg13b}; the error has been corrected here. This work was done when HN was a PhD student at the Indian Institute of Science, Bangalore and was supported by a Google India PhD Fellowship. SA thanks DST for support under a Ramanujan Fellowship.

\bibliographystyle{plain}
\bibliography{partial-auc}


\appendix
\allowdisplaybreaks

\section{Proofs}
\subsection{Proof of Theorem \ref{thm:auc-struct}}
\label{app:auc-struct-proof}
\begin{proof}
We simplify the structural SVM surrogate into a pair-wise form:
\begin{eqnarray*}
\widehat{R}^\struct_{\AUC}(w; \, S) 
&=&
\underset{\pi \,\in\, \Pi_{m,n}}{\max}\big\{\Delta_{\AUC}(\pi^*,\pi) \,-\, (w^\top\phi(S,\pi^*) - w^\top \phi(S,\pi))\big\}\\ 
&=&
\underset{\pi  \,\in\,\Pi_{m,n}}
	{\operatorname{max}}
	\bigg\{\frac{1}{mn}\sum_{i=1}^m\sum_{j=1}^{n} \pi_{ij}\big(1 \,-\, w^\top( x^+_i -  x^-_j)\big)
		\bigg\}.
\end{eqnarray*}
Now consider solving a relaxed form of the above argmax over all matrices in $\{0,1\}^{m\times n}$.
Since the objective is linear in the individual entries of $\pi_{ij}$, each $\pi_{ij}$ can be optimized independently, giving us
\begin{eqnarray*}
\lefteqn{
\underset{\pi \,\in\, \{0,1\}^{m\times n}}{\max}\,
\frac{1}{mn}\sum_{i=1}^m\sum_{j=1}^{n} \pi_{ij}\big(1 \,-\, w^\top( x^+_i -  x^-_j)\big)
}\\
&=&
\frac{1}{mn}\sum_{i=1}^m\sum_{j=1}^{n} 
\underset{\pi_{ij} \,\in\, \{0,1\}}{\max}
\pi_{ij}\big(1 \,-\, w^\top( x^+_i -  x^-_j)\big)
\hspace{2cm}
\\
&=&
\frac{1}{mn}\sum_{i=1}^m\sum_{j=1}^{n} 
\bar{\pi}_{ij}\big(1 \,-\, w^\top( x^+_i -  x^-_j)\big),
\end{eqnarray*}
where $\bar{\pi}_{ij} = \1\big(w^\top( x^+_i -  x^-_j) \leq 1\big)$. It can be seen that this optimal matrix $\bar{\pi}$ is in fact a valid ordering matrix in $\Pi_{m,n}$, as it corresponds to ordering of instances where the positives are scored according to $w^\top x$ and the negatives are scored according to $w^\top x + 1$. Thus $\bar{\pi}$ is also a solution to the original unrelaxed problem and hence
\begin{eqnarray*}
\widehat{R}^\struct_{\AUC}(w; \, S) 
&=&
\frac{1}{mn}\sum_{i=1}^m\sum_{j=1}^{n} 
\bar{\pi}_{ij}\big(1 \,-\, w^\top( x^+_i -  x^-_j)\big)\\
&=&
\frac{1}{mn}\sum_{i=1}^m\sum_{j=1}^{n} 
\big(1 \,-\, w^\top( x^+_i -  x^-_j)\big)_+
\\
&=&
\widehat{R}^\hinge_{\AUC}(w; \, S),
~~\text{as desired.}
\\[-1.5cm]
\end{eqnarray*}
\end{proof}

\subsection{Proof of Theorem \ref{thm:hinge-nonconvex}}
\label{app:hinge-non-convex}
We begin with some intuition. Suppose the given FPR range is $\displaystyle \bigg[\frac{n-1}{n}, 1\bigg]$ where $j_\alpha = n-1 > 0$ and $j_\beta = n$, the hinge surrogate reduces to a `min' of convex functions in $w$ and it is easy to see that there are training samples $S$ where the surrogate is non-convex in $w$. The same argument can be extended to general $k^{\text{th}}$ order statistic of a set of convex functions in $w$ when $k > 1$, and as seen below to general FPR intervals. 

In particular, for any given FPR range $[\alpha,\beta]$ with $\alpha > 0$, we provide a two-dimensional training sample $S = (S_+, S_-)$ with $m = 1$ positive instance and $n$ negative instances, for which the hinge surrogate $\widehat{R}^\hinge_\pAUC(\alpha,\beta)(w; S)$ (in \Eqn{eqn:emp-pauc-hinge-general}) is non-convex in $w$.  The training sample is constructed as follows: $S_+ = (x_1^+ = (0,0)^\top) \in \R^2$ and $S_- = (x^-_1, \ldots, x^-_n) \in (\R^2)^n$, with
\[
x^-_j =
\begin{cases}
(0,-1)^\top & \text{if } j = 1\\
(0,0)^\top & \text{if } j \in \{2,\ldots, j_\alpha\}\\
(-1,0)^\top & \text{if } j = j_\alpha + 1\\
(-1,-1)^\top & \text{otherwise}
\end{cases}.
\]
%
%
The corresponding hinge surrogate for a linear model $w \in \R^2$ is given by:
\begin{eqnarray*}
\widehat{R}^\hinge_{\pAUC(\alpha,\beta)}(w; S) 
		&=& \frac{1}{m(j_\beta-j_\alpha)} \sum_{i = 1}^m \sum_{j = j_\alpha+1}^{j_\beta} 
		\big( 1 \,+\, w^\top(x_{(j)_w}^- \,-\, x_i^+)\big)_+\\
		&=& \frac{1}{j_\beta-j_\alpha} \sum_{j = j_\alpha+1}^{j_\beta} 
		\big( 1 \,+\, w^\top(x_{(j)_w}^-)\big)_+.
\end{eqnarray*}
For simplicity let us refer to the above surrogate function in the specified FPR interval using the shorthand $\widehat{R}(w)$. To show that the surrogate is non-convex in $w$, we will now identify linear models $w_1, w_2 \in \R^2$ such that: $\widehat{R}(0.5\,w_1 + 0.5\,w_2) >  0.5\,\widehat{R}(w_1) \,+\, 0.5\,\widehat{R}(w_2).$ Let $w_1 = (1,0)^\top \in \R^2$ and $w_2 = (0,1)^\top \in \R^2$, and denote $w_3 = 0.5\,w_1 + 0.5\,w_2 = (0.5,0.5)^\top$, a convex combination of these models. Note that $w_1$ ranks the instances using the first feature, with all negative instances ranked in positions $j_\alpha+1$ to $j_\beta$ given a score of $-1$; the model $w_2$ ranks the instances using the second feature, with all negative instances ranked in positions $j_\alpha+1$ to $j_\beta$ again given a score of $-1$; the model $w_3$ ranks instances using a convex combination of the features, and in this case, the negative instance ranked in position $j_\alpha+1$ gets a score of $-0.5$, while those in positions $j_\alpha+2$ to $j_\beta$ get a score of $-1$. The surrogate then evaluates to the following values for these models:
%
\begin{table}[H]
\center
\begin{tabular}{l}
$\displaystyle \widehat{R}(w_1) \,=\, \frac{1}{j_\beta-j_\alpha} \sum_{j = j_\alpha+1}^{j_\beta} 
		\big( 1 \,+\, w_1^\top x_{(j)_{w_1}}^-\big)_+ \,=\, \frac{1}{j_\beta-j_\alpha} \sum_{j = j_\alpha+1}^{j_\beta} ( 1 - 1 )_+ \,=\, 0$
\\
$\displaystyle \widehat{R}(w_2) \,=\, \frac{1}{j_\beta-j_\alpha} \sum_{j = j_\alpha+1}^{j_\beta} 
		\big( 1 \,+\, w_1^\top x_{(j)_{w_1}}^-\big)_+ \,=\, \frac{1}{j_\beta-j_\alpha} \sum_{j = j_\alpha+1}^{j_\beta}  ( 1 - 1)_+ \,=\, 0$\\
$\displaystyle \widehat{R}(w_3) \,=\, \frac{1}{j_\beta-j_\alpha} \sum_{j = j_\alpha+1}^{j_\beta} 
		\big( 1 \,+\, w_3^\top x_{(j)_{w_3}}^-\big)_+ \,=\, \frac{1}{j_\beta-j_\alpha} \Big[(1-0.5)_+ \,+\, \sum_{j = j_\alpha+2}^{j_\beta} ( 1 - 1)_+\Big]$\\[5pt]
\hspace{1.22cm}$\displaystyle\,=\,\frac{0.5}{j_\beta-j_\alpha}.$
\end{tabular}
\vspace{-0.3cm}
\end{table}
\noindent Clearly, $0.5\,\widehat{R}(w_1) \,+\, 0.5\,\widehat{R}(w_2) \,=\, 0 \,<\, \widehat{R}(w_3),$ confirming $\widehat{R}$ is non-convex in $w$.

\subsection{Proof of Theorem \ref{thm:surrogate-characterization}}
\label{app:surrogate-characterization-proof}
\begin{proof}
Let us denote $N = m(j_\beta - j_\alpha)$. From \Eqn{eqn:mvc-Q}, the structural SVM surrogate for partial AUC in a general interval $[\alpha, \beta]$  can be written in the following simplified form:
\begin{eqnarray}
\lefteqn{\widehat{R}^\tight_{\pAUC(\alpha,\beta)}(w; S)} \nonumber\\
&=&
\underset{\pi  \,\in\,\Pi^w_{m,j_\beta}}{\max}\,
		\frac{1}{N}\sum_{i=1}^m\bigg[
			-\sum_{j=1}^{j_\alpha}\pi_{i(j)_w} w^\top (x^+_i- \bar{z}_{j}) \,+\, \sum_{j=j_\alpha+1}^{j_\beta}\pi_{i(j)_w} \big(1-w^\top (x^+_i- \bar{z}_{j})\big)
		\bigg] \nonumber
		\\
&=&
\underset{\pi  \,\in\,\Pi^w_{m,j_\beta}}{\max}\,
		\frac{1}{N}\sum_{i=1}^m\bigg[
			-\sum_{j=1}^{j_\alpha}\pi_{i(j)_w} w^\top (x^+_i- x^-_{(j)_w}) \,+\, \sum_{j=j_\alpha+1}^{j_\beta}\pi_{i(j)_w} \big(1-w^\top (x^+_i- x^-_{(j)_w})\big)
		\bigg], \nonumber
		\\		
		\label{eqn:surrogate-characterization-1}
\end{eqnarray}
where recall that $\bar{z}_j$ denotes the $j$-th ranked instance negative instance (among all instances in $S_-$, in descending order of scores) by $w^\top x$. 

The upper bound is obtained by maximizing each $\pi_{ij}$ independently:
\begin{eqnarray*}
\lefteqn{\widehat{R}^\tight_{\pAUC(\alpha,\beta)}(w; S)}\\
&\leq&
		\frac{1}{N}\sum_{i=1}^m\bigg[
			-\sum_{j=1}^{j_\alpha}
			\max_{\pi_{i(j)_w} \in \{0,1\}}			
			\pi_{i(j)_w} w^\top (x^+_i- x^-_{(j)_w})\\
&&
		\hspace{4cm}
			 \,+\, 
			\sum_{j=j_\alpha+1}^{j_\beta}
			\max_{\pi_{i(j)_w} \in \{0,1\}}			
			\pi_{i(j)_w} \big(1-w^\top (x^+_i- x^-_{(j)_w})\big)
		\bigg]
		\hspace{1cm}		
		\\
&=&	\frac{1}{N}\sum_{i=1}^m\bigg[
			\sum_{j=1}^{j_\alpha}		
			\big(-w^\top (x^+_i- x^-_{(j)_w})\big)_+
			\,+\, 
			\sum_{j=j_\alpha+1}^{j_\beta}		
			\big(1-w^\top (x^+_i- x^-_{(j)_w})\big)_+
		\bigg].
\end{eqnarray*}

The lower bound is obtained by substituting in \Eqn{eqn:surrogate-characterization-1} a specific ordering matrix. Define $\hat{\pi} \in \{0,1\}^{m\times j_\beta}$
as follows: for each $x_i^+$ such that $w^\top x_i^+ < w^\top x_{(j_\alpha)_w}^-$,
\begin{equation*}
\label{eq:ybar1}
\hat{\pi}_{i(j)_w} = 
\begin{cases}
1
& \text{if } j \in \{1, \ldots, j_\alpha\}\\
\textbf{1}\big(w^\top x^+_i- w^\top x^-_{(j)_w} \leq 1\big) 
& \text{otherwise}
\end{cases}
\end{equation*}
and for each $x_i^+$ such that $w^\top x_i^+ \geq w^\top x_{(j_\alpha)_w}^-$,
\begin{equation*}
\label{eq:ybar2}
\hat{\pi}_{i(j)_w} = 
\begin{cases}
\textbf{1}\big(w^\top x^+_i- w^\top x^-_{(j)_w} \leq 0\big) 
& \text{if } j \in \{1, \ldots, j_\alpha - 1\}\\
0
& \text{otherwise}
\end{cases}.
\end{equation*}
The matrix $\hat{\pi}$ corresponds to an ordering of instances where all negative instances and the positive instances $i$ for which $w^\top x_i^+ \geq w^\top x_{(j_\alpha)_w}^-$ are scored according to $w^\top x$, while the remaining positive instances are scored by $w^\top x -1$. Clearly, $\hat{\pi}$ is a valid ordering matrix in $\Pi^w_{m,j_\beta}$. We then have
\begin{eqnarray}
\lefteqn{\widehat{R}^\tight_{\pAUC(\alpha,\beta)}(w; S)}
\nonumber
\\
&\geq&
		\frac{1}{N}\sum_{i=1}^m\bigg[
			\sum_{j=1}^{j_\alpha}
			\hat{\pi}_{i(j)_w} \big(-w^\top (x^+_i- x^-_{(j)_w})\big)
			 \,+\, 
			\sum_{j=j_\alpha+1}^{j_\beta}
			\hat{\pi}_{i(j)_w} \big(1-w^\top (x^+_i- x^-_{(j)_w})\big)
		\bigg] \nonumber\\
&=&
		\frac{1}{N}\sum_{i: w^\top x_i^+ < w^\top x_{(j_\alpha)_w}^-}\bigg[
			\sum_{j=1}^{j_\alpha}
			(1)\big(-w^\top (x^+_i- x^-_{(j)_w})\big)
			 \,+\, 
			\sum_{j=j_\alpha+1}^{j_\beta}
			\big(1-w^\top (x^+_i- x^-_{(j)_w})\big)_+
		\bigg] \nonumber\\
&&
		\hspace{3.8	cm}
		\,+\,
		\frac{1}{N}\sum_{i: w^\top x_i^+ \geq w^\top x_{(j_\alpha)_w}^-}\bigg[
			\sum_{j=1}^{j_\alpha}
			\big(-w^\top (x^+_i- x^-_{(j)_w})\big)_+
			 \,+\, 
			 0
		\bigg],\nonumber
\end{eqnarray}
which follows from the definition of $\hat{\pi}$. Now, for all indices $i$ and $j$ in the first term, $-w^\top (x^+_i- x^-_{(j)_w}) \,=\, -w^\top x^+_i + w^\top x^-_{(j)_w} \,\geq\, -w^\top x^+_i + w^\top x^-_{(j_\alpha)_w} \,\geq\, 0$, which implies  $-w^\top (x^+_i- x^-_{(j)_w}) = \big(-w^\top (x^+_i- x^-_{(j)_w})\big)_+$.  As a result,
\begin{eqnarray}
\lefteqn{\widehat{R}^\tight_{\pAUC(\alpha,\beta)}(w; S)}
\nonumber
\\
&\geq&
		\frac{1}{N}\sum_{i: w^\top x_i^+ < w^\top x_{(j_\alpha)_w}^-}\bigg[
			\sum_{j=1}^{j_\alpha}
			\big(-w^\top (x^+_i- x^-_{(j)_w})\big)_+
			 \,+\, 
			\sum_{j=j_\alpha+1}^{j_\beta}
			\big(1-w^\top (x^+_i- x^-_{(j)_w})\big)_+
		\bigg] \nonumber\\
&&
		\hspace{4.7cm}
		\,+\,
		\frac{1}{N}\sum_{i: w^\top x_i^+ \geq w^\top x_{(j_\alpha)_w}^-}
			\sum_{j=1}^{j_\alpha}
			\big(-w^\top (x^+_i- x^-_{(j)_w})\big)_+.
\label{eqn:pi-hat}
\end{eqnarray}
The desired bound then follows from the non-negativity of the last term.

Finally, we show that the upper bound on the surrogate holds with equality when $|w^\top x^+_i - w^\top x^-_j| \geq 1,\, \forall\, i \in \{1,\ldots,m\}, \,j \in \{1,\ldots,n\}$. 
Under this condition, if $w^\top (x^+_i- x^-_{(j)_w})$ is positive for some $i \in \{1,\ldots,m\}$ and $j \in \{1,\ldots,n\}$ , then it is also the case that $w^\top (x^+_i- x^-_{(j)_w}) \geq 1$. This therefore gives us:
\begin{eqnarray*}
\frac{1}{N}\sum_{i: w^\top x_i^+ \geq w^\top x_{(j_\alpha)_w}^-}
			\sum_{j=j_\alpha+1}^{j_\beta}
			\big(1-w^\top (x^+_i- x^-_{(j)_w})\big)_+ ~=~ 0
\end{eqnarray*}
Adding this term to the lower bound on the surrogate in \Eqn{eqn:pi-hat}, we have
\begin{eqnarray*}
\lefteqn{
\widehat{R}^\tight_{\pAUC(\alpha,\beta)}(w; S)
}
\nonumber
\\
&\geq&
		\frac{1}{N}\sum_{i = 1}^m\bigg[
			\sum_{j=1}^{j_\alpha}
			\big(-w^\top (x^+_i- x^-_{(j)_w})\big)_+
			 \,+\, 
			\sum_{j=j_\alpha+1}^{j_\beta}
			\big(1-w^\top (x^+_i- x^-_{(j)_w})\big)_+
		\bigg], \hspace{1cm}
\end{eqnarray*}
which is same as the upper bound on the surrogate in theorem statement; hence under the given condition on $w$, the surrogate is equal to the upper bound.
\end{proof}

\subsection{Proof of Theorem \ref{thm:pauc-gen-bound}}
\label{app:pauc-gen-bound-proof}
For simplicity, we provide the proof for FPR ranges of the form $[0, \beta]$. The proof can be easily extended to the general case by rewriting the partial AUC risk in $[\alpha, \beta]$ as a difference of (renormalized) partial AUC risks in $[0,\beta]$ and $[0,\alpha]$. Specifically, denoting henceforth the population pAUC risk in $[0,\beta]$ as $R_\beta$ and the corresponding empirical pAUC risk as $\widehat{R}_\beta$, we will show that with probability at least $1 -\delta$ (over draw of the training sample $S$), for all scoring functions $f \in \F$,
\[
{R}_{\beta}[f; \D] \,-\,\widehat{R}_{\beta}[f; S] ~\leq~ C\bigg(\sqrt{\frac{d\ln(m) + \ln(1/\delta)}{m}} ~\leq~ \frac{1}{\beta}\sqrt{\frac{d\ln(n) + \ln(1/\delta)}{n}}\bigg),
\]
where $d$ is the VC dimension of $\T_\F$, and $C > 0$ is a distribution-independent constant. We shall also assume throughout that $f$ has no ties and that $n\beta$ is an integer. Again it is straightforward to extend the proof to settings where these assumptions do not hold. 

We begin by introducing some notations.


\textbf{Top-$\beta$ Thresholds}. For a given $f: \X \> \R$, and false positive rate (FPR) $\beta \in (0, 1]$, we shall find it convenient to define the (population) top-$\beta$ threshold w.r.t.\ distribution $\D_-$ as the smallest threshold $t_{\D_-, f, \beta} \in \R$ for which the classifier $\sign \circ (f - t_{\D_-, f, \beta})$ has a FPR equal to $\beta$.\footnote{For simplicity, we consider only distributions $\D$ and false positive intervals $[0, \beta]$ for which the threshold $t_{\D_-, f, \beta}$ exists.} More specifically,
\[
t_{\D_-, f, \beta} = \arginf_{t \in \R} \Big\{ t \in \R ~\Big|~ \P_{x^- \sim \D_-}\big[f(x^-) > t\big] = \beta \Big\}.
\]
It follows from the above definition that $\E_{x^- \sim \D_-}\big[\1\big(f(x^-) > t_{\D_-, f, \beta}\big)\big] = \beta$. The empirical version of the top-$\beta$ threshold for a training sample of $n$ negatives $S_-$ is then:
\[
\wt_{S_-, f, \beta} = \argmin_{t \in \R} \bigg\{ t \in \R \,\bigg|\, \frac{1}{n}\sum_{j=1}^n \1\big(f(x^-_j) > t\big) \geq \beta \bigg\}.
\]
Given that $f$ has no ties, $\wt_{S_-, f, \beta}$ is the threshold on $f$ above which $n\beta$ of the negative instances in $S_-$ are ranked by $f$; in other words, $\sum_{j=1}^n \1\big(f(x_j^-) > \wt_{S_-, f, \beta}\big) = n\beta$. 

\textbf{Partial AUC risk.} We now rewrite the population and empirical partial AUC risks in \Eqn{eqn:gen-bound-pop-risk} and \ref{eqn:gen-bound-emp-risk} in terms of the above thresholds. For any scoring function $f: \X \> \R$, the population partial AUC risk for an FPR interval $[0, \beta]$ can be written as
\[
R_{\beta}[f; \D] \,=\, \frac{1}{\beta}\E_{x^+ \sim \D_+, x^- \sim \D_-}\Big[\1\big(f(x^+) \leq f(x^-), ~ f(x^-) > t_{\D_-, f, \beta}\big)\Big],
\]
and the empirical partial AUC risk for sample $S$ is given by
\[
\werr_{\beta}[f; S] \,=\, \frac{1}{mn\beta} \sum_{i=1}^{m}\sum_{j=1}^{n}\1\big(f(x_i^+) \leq f(x_j^-), ~ f(x_j^-) > \wt_{S_-, f, \beta}\big).
\]
%
%
 
%
%

Before providing the proof of the theorem, we will find it convenient to state the following uniform convergence result.
\begin{lemma}
Let $\F$ be a class of real-valued functions on $\X$, and $\T_\F = \big\{\sign \circ (f - t) \,|\, f \in \F, \, t \in \R\big\}$. Fix any $\epsilon > 0$. We then have
\[
\P_{S_+ \sim \D_+^m} \bigg(\bigcup_{f \in \F}
\bigcup_{t \in \R}\bigg\{
\bigg|
\frac{1}{m}\sum_{i=1}^{m}\1\big(f(x_i^+) \leq t\big)
\,-\,
 \E_{x^+ \sim \D_+}\Big[\1\big(f(x^+) \leq t\big)\Big]
 \bigg|
\,\geq \, \epsilon \bigg\}\bigg)
\,\leq\,
C_1 \, m^d e^{-2m\epsilon^2},
\]
where $d$ is the VC-dimension of $\T_\F$, and $C_1 > 0$ is a distribution-independent constant. Similarly,
\[
\P_{S_- \sim \D_-^n} \bigg(\bigcup_{f \in \F}
\bigcup_{t \in \R}\bigg\{
\bigg|
\frac{1}{n}\sum_{j=1}^{n}\1\big(f(x_j^-) > t\big)
\,-\,
 \E_{x^- \sim \D_-}\Big[\1\big(f(x^-) > t\big)\Big]\bigg|
\,\geq \, \epsilon \bigg\}\bigg)
\,\leq\,
C_2 \, n^d e^{-2n\epsilon^2},
\]
where $C_2 > 0$ is a distribution-independent constant.
\label{lem:uniform-covergence}
\end{lemma}
The above lemma can be proved by a standard VC-dimension based uniform convergence argument over the class of thresholded classifiers $\T_\F$. We are now ready to prove \Thm{thm:pauc-gen-bound}.

\begin{proof}[Proof of Theorem \ref{thm:pauc-gen-bound}]
We will find it convenient to define for a function $f$ and negative instance $x^-$, the expected loss on a randomly drawn positive instance, $\ell_+(f, x^-) = \E_{x^+ \sim \D_+}\big[\1\big(f(x^+) \leq f(x^-)\big)\big]$. We further introduce the following error terms defined in terms of the population and empirical top-$\beta$ thresholds respectively: 
\begin{eqnarray*}
\terr_{\beta}[f; \D, S_-] &=& \frac{1}{n\beta} \sum_{j=1}^n \ell_+(f, x_j^-)\,\1\big(f(x_j^-) > t_{\D_-, f, \beta}\big);\\
\berr_{\beta}[f; \D_+, S_-] &=& \frac{1}{n\beta} \sum_{j=1}^n \ell_+(f, x_j^-)\,\1\big(f(x_j^-) > \wt_{S_-, f, \beta}\big).
\end{eqnarray*}

We then have for any $f \in \F$,
\begin{eqnarray*}
\lefteqn{R_{\beta}[f; \D]\,-\,\werr_{\beta}[f; S]
~=~ 
\big(R_{\beta}[f; \D] \,-\, \terr_{\beta}[f; \D, S_-]\big)
}
&&\\
&& 
\hspace{2cm}
~+~
\big( \terr_{\beta}[f; \D, S_-] \,-\, \berr_{\beta}[f; \D_+, S_-]\big)
~+~
\big(\berr_{\beta}[f; \D_+, S_-]\,-\, \werr_{\beta}[f; S]\big)
.
\end{eqnarray*}
Thus for any $\epsilon > 0$,
\begin{eqnarray*}
\lefteqn{\P_{S \sim D_+^m \times D_-^n} \bigg(\bigcup_{f \in \F}\Big\{
R_{\beta}[f; \D] \,-\, \werr_{\beta}[f; S] \,\geq \, \epsilon \Big\}\bigg)}\\
&\leq& 
\underbrace{
 \P_{S_- \sim D_-^n} \bigg(\bigcup_{f \in \F}\Big\{R_{\beta}[f; \D]\,-\,\terr_{\beta}[f; \D, S_-] \,\geq \, \epsilon/3 \Big\}\bigg)
}_{\term_1}
\\
&& \hspace{2.5cm}
~+~ 
\underbrace{
\P_{S_- \sim D_-^n} \bigg(\bigcup_{f \in \F}\Big\{\terr_{\beta}[f; \D, S_-]\,-\,\berr_{\beta}[f; \D_+, S_-] \,\geq \, \epsilon/3\Big\} \bigg)}_{\term_2}
\\
&& \hspace{2.5cm}
 ~+~ 
\underbrace{ 
\P_{S \sim D_+^m \times D_-^n} 
\bigg(\bigcup_{f \in \F}\Big\{\berr_{\beta}[f; \D_+, S_-]\,-\,\werr_{\beta}[f; S]  \,\geq \, \epsilon/3\Big\} \bigg)
}_{\term_3} 
.
\end{eqnarray*}

We now bound each of the above probability terms separately. We start with the first term.
\begin{eqnarray*}
\lefteqn{\, R_{\beta}[f; \D]\,-\,\terr_{\beta}[f; \D, S_-]}\\
&=&
\frac{1}{\beta}\E_{x^- }\Big[ \ell_+(f, x^-)\,\1\big(f(x^-) > t_{\D_-, f, \beta}\big) \Big]\,-\,
\frac{1}{n\beta} \sum_{j=1}^n \ell_+(f, x_j^-)\,\1\big(f(x_j^-) > t_{\D_-, f, \beta}\big) \\
&=&
\frac{1}{\beta}
\E_{x^+}\bigg[
\E_{x^- }\Big[\1\big(f(x^+) \leq f(x^-)\big)\1\big(f(x^-) > t_{\D_-, f, \beta}\big)\Big]
\\
&&
\hspace{4.5cm}
~-~\frac{1}{n} \sum_{j=1}^n \1\big(f(x^+) \leq f(x_j^-)\big)\1\big(f(x_j^-) > t_{\D_-, f, \beta}\big) 
  \bigg]
 \\
&=&
\frac{1}{\beta}
\E_{x^+}\bigg[\E_{x^- }\Big[\1\Big(f(x^-) > \max\big\{f(x^+), t_{\D_-, f, \beta}\big\}\Big)
\\
&&
\hspace{4.5cm}
 ~-~
\frac{1}{n} \sum_{j=1}^n \1\Big(f(x_j^-) > \max\big\{f(x^+), t_{\D_-, f, \beta}\big\}\Big)\Big]\bigg]
\\
&\leq&
\frac{1}{\beta}
\sup_{x^+ \in \X}\bigg|
\E_{x^- }\Big[\1\Big(f(x^-) > \max\big\{f(x^+), t_{\D_-, f, \beta}\big\}\Big) \Big]
\\
&&
\hspace{4.5cm}
~-~
\frac{1}{n} \sum_{j=1}^n \1\Big(f(x_j^-) > \max\big\{f(x^+), t_{\D_-, f, \beta}\big\}\Big)
\bigg|\\
&\leq&
\sup_{t \in \R}\bigg|\frac{1}{\beta}
\E_{x^- }\big[\1\big(f(x^-) > t\big) \big]
~-~
\frac{1}{n} \sum_{j=1}^n \1\big(f(x_j^-) > t\big)\bigg|
,
\end{eqnarray*}
where in the third step, we use the fact that $f$ has no ties. Thus
\begin{eqnarray*}
\term_1 &=&
\P_{S_- \sim D_-^n} \bigg(\bigcup_{f \in \F}\Big\{R_{\beta}[f; \D] \,-\,
\terr_{\beta}[f; \D, S_-] 
 \,\geq \, \epsilon/3 \Big\}\bigg) \\
&\leq&
\P_{S_- \sim D_-^n} \bigg(\bigcup_{f \in \F}\bigg\{
\sup_{t \in \R}\bigg|\E_{x^- }\big[\1\big(f(x^-) > t\big) \big]
~-~
\frac{1}{n} \sum_{j=1}^n \1\big(f(x_j^-) > t\big)
\bigg|
~\geq~ \beta\epsilon/3 \bigg\}\bigg)\\
&=&
\P_{S_- \sim D_-^n} \bigg(\bigcup_{f \in \F}\bigcup_{t \in \R}\bigg\{
\bigg|\E_{x^- }\big[\1\big(f(x^-) > t\big) \big] 
~-~
\frac{1}{n} \sum_{j=1}^n \1\big(f(x_j^-) > t\big)\bigg|	
~\geq~ \beta\epsilon/3 \bigg\}\bigg)\\
&\leq&
C_2 \, n^d e^{-2n\beta^2\epsilon^2/9},
\end{eqnarray*}
which follows by applying the result in \Lem{lem:uniform-covergence}.

For the second term, we have
\begin{eqnarray*}
\lefteqn{\term_2}\\
 &=& \P_{S_- \sim D_-^n} \bigg(\bigcup_{f \in \F}\Big\{\terr_{\beta}[f; \D, S_-] \,-\, \berr_{\beta}[f; \D_+, S_-]
 \,\geq \, \epsilon/3\Big\} \bigg)
\\
&=& 
\P_{S_- \sim D_-^n} \bigg(\bigcup_{f \in \F}\bigg\{
\frac{1}{n\beta} \sum_{j=1}^n \ell_+(f, x_j^-)\,\1\big(f(x_j^-) > t_{\D_-, f, \beta}\big)
\\
&&
\hspace{4.5cm}
\,-\,
\frac{1}{n\beta} \sum_{j=1}^n \ell_+(f, x_j^-)\,\1\big(f(x_j^-) > \wt_{S_-, f, \beta}\big)
~\geq~ \epsilon/3 
 \bigg\}
 \bigg)\\
&\leq&
\P_{S_- \sim D_-^n} \bigg(\bigcup_{f \in \F}\bigg\{
\bigg| \frac{1}{n\beta} \sum_{j=1}^n \ell_+(f, x_j^-) \Big[\1\big(f(x_j^-) > t_{\D_-, f, \beta}\big)\\
&&
\hspace{4.5cm}
\,-\,
\1\big(f(x_j^-) > \wt_{S_-, f, \beta}\big)\Big] \bigg|
~\geq~ \epsilon/3 
 \bigg\}
 \bigg).
\end{eqnarray*}
Note that if $t_{\D_-, f, \beta} \leq  \wt_{S_-, f, \beta}$, then $\1\big(f(x^-) > t_{\D_-, f, \beta}\big) - \1\big(f(x^-) > \wt_{S_-, f, \beta}\big) \geq 0, \,\forall x^- \in \X$, and if $t_{\D_-, f, \beta} >  \wt_{S_-, f, \beta}$, then $\1\big(f(x^-) > t_{\D_-, f, \beta}\big) - \1\big(f(x^-) > \wt_{S_-, f, \beta}\big) \leq 0, \,\forall x^- \in \X$; since one of these two cases will always hold, and because $\ell$ is bounded ($0 \leq \ell_+(f, x^-) \leq 1, \, \forall x^- \in \X$), we have
\begin{eqnarray*}
\term_2
&\leq&
\P_{S_- \sim D_-^n} \bigg(\bigcup_{f \in \F}\bigg\{
\bigg|\frac{1}{n\beta} \sum_{j=1}^n \Big[\1\big(f(x_j^-) > t_{\D_-, f, \beta}\big)
\\
&&
\hspace{4.5cm}
\,-\, 
\1\big(f(x_j^-) > \wt_{S_-, f, \beta}\big)
\Big]
\bigg|
~\geq~ \epsilon/3 
 \bigg\}
 \bigg)\\
&=&
\P_{S_- \sim D_-^n} \bigg(\bigcup_{f \in \F}\bigg\{
\bigg|
\frac{1}{n} \sum_{j=1}^n\1\big(f(x_j^-) > t_{\D_-, f, \beta}\big)\\
&&
\hspace{4.5cm}
\,-\, 
\frac{1}{n} \sum_{j=1}^n \1\big(f(x_j^-) > \wt_{S_-, f, \beta}\big)
\bigg|
~\geq~ \beta\epsilon/3 
 \bigg\}
 \bigg)\\
&=&
\P_{S_- \sim D_-^n} \bigg(\bigcup_{f \in \F}\bigg\{
\bigg|
\frac{1}{n} \sum_{j=1}^n \1\big(f(x_j^-) > t_{\D_-, f, \beta}\big)
\,-\,
\beta
\bigg|
~\geq~ \beta\epsilon/3 
 \bigg\}
 \bigg)\\
&=&
\P_{S_- \sim D_-^n} \bigg(\bigcup_{f \in \F}\bigg\{
\bigg|
\frac{1}{n} \sum_{j=1}^n \1\big(f(x_j^-) > t_{\D_-, f, \beta}\big)\\
&&
\hspace{4.5cm}
\,-\, 
\E_{x^-}\big[\1\big(f(x^-) > t_{\D_-, f, \beta}\big)\big]
\bigg|
~\geq~ \beta\epsilon/3 
 \bigg\}
 \bigg)\\
&\leq&
\P_{S_- \sim D_-^n} \bigg(\bigcup_{f \in \F}\bigcup_{t \in \R}\bigg\{
\bigg|
\frac{1}{n} \sum_{j=1}^n \1\big(f(x_j^-) > t\big)
\,-\, 
\E_{x^-}\big[\1\big(f(x^-) > t\big)\big]
\bigg|
~\geq~ \beta\epsilon/3 
 \bigg\}
 \bigg)\\
&\leq&
C_2 \, n^d e^{-2n\beta^2\epsilon^2/9},
\end{eqnarray*}
where the third and fourth steps follow respectively from the definitions of the thresholds $\wt_{S_-, f, \beta}$ and $t_{\D_-, f, \beta}$ respectively; the last step uses the uniform convergence result in \Lem{lem:uniform-covergence}.

We next focus on bounding $\term_3$.
\begin{eqnarray}
\lefteqn{\berr_{\beta}[f; \D_+, S_-] \,-\, \werr_{\beta}[f; S]}
\nonumber
\\
&=&
\frac{1}{n\beta} \sum_{j=1}^{n} \E_{x^+}\Big[\1\big(f(x^+) \leq f(x_j^-), ~ f(x_j^-) > \wt_{S_-, f, \beta}\big)\Big]
\nonumber\\
&&
\hspace{4cm}
~-~ 
\frac{1}{mn\beta}\sum_{j=1}^{n}\sum_{i=1}^{m}\1\big(f(x_i^+) \leq f(x_j^-), ~ f(x_j^-) > \wt_{S_-, f, \beta}\big)
\nonumber
\\
&=&
\frac{1}{n\beta} \sum_{j=1}^{n}\1\big(f(x_j^-) > \wt_{S_-, f, \beta}\big)\bigg[
 \E_{x^+}\Big[\1\big(f(x^+) \leq f(x_j^-)\big)\Big]
\,-\,
\frac{1}{m}\sum_{i=1}^{m}\1\big(f(x_i^+) \leq f(x_j^-)\big) 
 \bigg]
\nonumber 
 \\
&\leq&
\frac{1}{n\beta} \sum_{j=1}^{n}\1\big(f(x_j^-) > \wt_{S_-, f, \beta}\big) \, \sup_{t \in \R}\bigg| \E_{x^+}\Big[\1\big(f(x^+) \leq t\big)\Big]\,-\,\frac{1}{m}\sum_{i=1}^{m}\1\big(f(x_i^+) \leq t\big)
\bigg|
\nonumber 
 \\
&=& \frac{1}{n\beta}(n\beta)\sup_{t \in \R}\bigg|
 \E_{x^+}\Big[\1\big(f(x^+) \leq t\big)\Big]
\,-\,
\frac{1}{m}\sum_{i=1}^{m}\1\big(f(x_i^+) \leq t\big)\bigg|
\nonumber 
 \\
&=& \sup_{t \in \R}\bigg|
 \E_{x^+}\Big[\1\big(f(x^+) \leq t\big)\Big]
\,-\,
\frac{1}{m}\sum_{i=1}^{m}\1\big(f(x_i^+) \leq t\big) 
\bigg|
 ,
 \label{eqn:term1-subs}
\end{eqnarray}
where the fourth step uses the definition of the empirical top-$\beta$ threshold $\wt_{S_-, f, \beta}$.
Then
\begin{eqnarray*}
{\term_3} &=&
\P_{S \sim D_+^m \times D_-^n} \bigg(\bigcup_{f \in \F}\Big\{\berr_{\beta}[f; \D_+, S_-]\,-\,\werr_{\beta}[f; S]  \,\geq \, \epsilon/3 \Big\}\bigg)\\
&=&
\E_{S_-}\Bigg[\P_{S_+|S_-} \bigg(\bigcup_{f \in \F}\bigg\{
\berr_{\beta}[f; \D_+, S_-]\,-\,\werr_{\beta}[f; S]
~\geq~ \epsilon/3 \bigg\}\bigg)\Bigg]\\
&\leq&
\E_{S_-}\Bigg[\P_{S_+|S_-} \bigg(\bigcup_{f \in \F}\bigg\{
\sup_{t \in \R}\bigg|
\E_{x^+}\Big[\1\big(f(x^+) \leq t\big)\Big]
\\
&&\hspace{3.5cm}
\,-\,
 \frac{1}{m}\sum_{i=1}^{m}\1\big(f(x_i^+) \leq t\big)\bigg|
\,\geq \, \epsilon/3 \bigg\}\bigg)
\Bigg]\\
&=&
\E_{S_-}\Bigg[\P_{S_+|S_-} \bigg(\bigcup_{f \in \F}
\bigcup_{t \in \R}\bigg\{
\bigg|
 \E_{x^+}\Big[\1\big(f(x^+) \leq t\big)\Big]\\
&&
\hspace{3.5cm}
\,-\,
\frac{1}{m}\sum_{i=1}^{m}\1\big(f(x_i^+) \leq t\big)
\bigg|
\,\geq \, \epsilon/3 \bigg\}\bigg)
\Bigg]\\
&\leq&
C_1 \, m^d e^{-2m\epsilon^2/9},
\end{eqnarray*}
where the third step follows from \Eqn{eqn:term1-subs}, and the last step follows from the uniform convergence result in Lemma \ref{lem:uniform-covergence}.

Combining the bounds on each of the three terms gives us our desired result.
\end{proof}

\section{Other Supplementary Material}
\subsection{Efficient Implementation of Algorithms \ref{algo:mvc} and \ref{algo:mvc-general}}
\label{app:mvc-efficient}
In their current form, Algorithms \ref{algo:mvc} and \ref{algo:mvc-general} work with ordering matrices of size $mj_\beta$, and hence require a run time of $O(mj_\beta)$ to at least visit each entry of the matrix once. However, as in the case of the full AUC \cite{svmperf}, a more compact representation of size $m+j_\beta$ is sufficient to implement these algorithm. The specific representation that we use maintains counts of number of negative instances ranked below each positive instance and the number of positives ranked above each negative instance. With this new representation, the operations on the ordering matrices can be implemented using a simple sort operation, with a reduced time complexity of $O((m+j_\beta)\log(m+j_\beta))$.

More formally, given a sample with positive instances $S_+ = \{x^+_1, \ldots, x^+_m\}$ and negative instances $Z = \{z_1, \ldots, z_{j_\beta}\}$, we use a $(m+n)$-vector $\a = [\a^+, \a^-] \in \N^{m+j_\beta}$ to represent a relative ordering of the positives and negatives, where $a^+_i$ shall denote the number of negatives ranked below positive instance $x^+_i$ and $a^-_j$ is the number of positives ranked above negative instance $z_j$. The joint feature map in \Eqn{eqn:psi} for sample $(S_+, Z)$ can now be rewritten in terms of this representation:
\begin{eqnarray*}
\phi((S_+, Z),\pi) &=& \frac{1}{mj_\beta} \sum_{i=1}^m \sum_{j=1}^{j_\beta} (1-\pi_{ij}) (x^+_i - z_j)
	\,\,=\,\, \frac{1}{mj_\beta} \bigg[\sum_{i=1}^m a^+_i x^+_i - \sum_{j=1}^{j_\beta} a^-_j z_j\bigg].
\end{eqnarray*}
Similarly, the AUC loss in \Eqn{eqn:auc-loss} for $m$ positive instances and $j_\beta$ negative instances can be written as:
\begin{eqnarray*}
\Delta_{\AUC}(\pi^*,\pi) =  \frac{1}{mj_\beta}\sum_{i = 1}^m \sum_{j=1}^n \pi_{i,j}
	= \frac{1}{mj_\beta} \sum_{j=1}^{j_\beta} (m- a^-_j),
\end{eqnarray*}
and the pAUC loss in \Eqn{eqn:pauc-loss} for $m$ positives and $j_\beta$ negatives becomes:
\[
\Delta^\text{tr}_\pAUC(\pi^*,\pi) \,=\,  \frac{1}{m(j_\beta-j_\alpha)}\sum_{i = 1}^m \sum_{j=j_\alpha+1}^{j_\beta} \pi_{i,(j)_\pi}
\,=\, \frac{1}{m(j_\beta-j_\alpha)} \sum_{j=j_\alpha+1}^{j_\beta} \big(m - a^-_{(j)_\a}\big),
\]
where $(j)_\a$ denotes the index of the $j$-th ranked negative instance by any fixed ordering of instances consistent with $\a$ (note that all such orderings yield the same loss).

In the case of Algorithm \ref{algo:mvc}, line 3 can now be implemented in $O((m+j_\beta)\log(m+j_\beta))$ time by sorting the instances in $(S_+, \bar{Z})$ according to scores $w^\top x$ on positive instances and $w^\top x + 1$ on negative instances, and by constructing the desired output ordering $\bar{\a}$ from the sorted list. Since line 2 requires a sorting of negatives, the overall run time becomes $O((m+j_\beta)\log(m+j_\beta) \,+\, n\log(n))$. Similarly for Algorithm \ref{algo:mvc}, one can rewrite $H^i_w$ in terms of the new representation; the loop in lines 3--7 can then be implemented using two sorted lists (corresponding to $\pi$ and $\pi'$) and by maintaining appropriate counts while making a single pass over these lists. One again obtains a run time of $O((m+j_\beta)\log(m+j_\beta) \,+\, n\log(n))$ in this case.

%
%

%
%
\subsection{Additional Experimental Details}
\label{app:expts}
Here we provide additional details about our experiments.

\textbf{Parameter setting.} In most experiments, the parameters for the proposed SVM methods were chosen as follows: the regularization parameter $C$ was chosen from the range $\{10^{-5}, \ldots, 10^4\}$ for $\text{SVM}_{\text{pAUC}}$ and $\text{SVM}_{\text{AUC}}$, and from $\{10^{-2}, \ldots, 10^4\}$ for $\text{SVM}^\text{dc}_{\text{pAUC}}$ using a held-out portion of training set; the tolerance parameter $\epsilon$ was set to $10^{-4}$ for $\text{SVM}_{\text{pAUC}}$ and $\text{SVM}_{\text{AUC}}$, and the parameter $\tau$ in $\text{SVM}^\text{dc}_{\text{pAUC}}$ was set to $10^{-3}$. 

For the experiments on UCI data sets in Tables  \ref{tab:0.02-0.05}, a slightly smaller range of values $\{10^{-5}, \ldots, 10^{3}\}$ was used while tuning the regularization parameter $C$ for $\text{SVM}_{\text{pAUC}}$ on the a9a and covtype data sets. For the experiments on breast cancer detection in Table \ref{tab:kddcup08}, the tolerance parameter $\epsilon$ was set to $10^{-3}$ for $\text{SVM}_{\text{pAUC}}$ and $\text{SVM}^{\text{dc}}_{\text{pAUC}}$, and to $10^{-4}$ for $\text{SVM}_{\text{AUC}}$. For the results on pulmonary emboli detection in Table \ref{tab:kddcup06}, we used a smaller range of $\{10^{-5}, \ldots, 10^3\}$ for choosing the parameter $C$ in $\text{SVM}_{\text{pAUC}}$ and $\text{SVM}^\text{dc}_{\text{pAUC}}$.

We now move to the baseline methods. The two tunable parameters $\mu$ and $\tau$ in the baseline ASVM method \cite{asvm} were chosen from the ranges $\{0.0005, 0.001, \ldots, 0.1\}$ and $\{5 \times 10^{-5}, 10^{-4}, \ldots, 0.01\}$ respectively; in the case of the baseline pAUCBoost method \cite{paucboost}, the number of iterations was set to 100 for the smaller letter, chemo, leukemia and kddcup06 data sets and to 25 for the rest; in the Greedy-Heuristic method \cite{paucmax}, the coarseness parameter (number of equally spaced intervals into which each feature needs to be divided) was set to 100.


\textbf{Data preprocessing.} For data sets with more than two classes, one of the classes was taken as positive and the remaining were combined into the positive class. All continuous valued features in the training sets were normalized to zero mean and unit variance, and the same transformation was replicated on the test set. In case of the ppi data set alone, we assumed a transductive setting (this is a reasonable assumption as the set of all protein pairs is known apriori), where the statistics for normalization were computed on the entire data set.


\end{document}